\documentclass[11pt]{article}
\usepackage[round]{natbib}
\RequirePackage[OT1]{fontenc}
\RequirePackage{amsthm,amsmath}
\allowdisplaybreaks[4]
\usepackage{amssymb}
\usepackage{fullpage}
\usepackage{amsthm}
\usepackage{bbm}
\usepackage{bm}
\usepackage{float}
\usepackage{graphicx}
\usepackage{placeins}
\usepackage{subcaption}
\usepackage{xcolor}
\usepackage{tikz}
\usetikzlibrary{arrows.meta,decorations.pathreplacing}
\usepackage[hmargin={1.0in, 1.0in}, vmargin={1.0in, 1.0in}]{geometry}
\usepackage{booktabs}
\usepackage{array}
\newcolumntype{Z}{>{\setbox0=\hbox\bgroup}c<{\egroup}@{\hspace*{-\tabcolsep}}}
\usepackage{setspace}
\usepackage{diagbox}
\usepackage{algorithm,algpseudocode}
\usepackage{xurl}
\usepackage{hyperref}
\hypersetup{
	colorlinks=true,
	linkcolor=blue,
	filecolor=magenta,
	urlcolor=cyan,
	citecolor=blue,
}
\usepackage{cleveref}

\makeatletter
\providecommand*{\theHALG@line}{}
\renewcommand*{\theHALG@line}{\thealgorithm.\arabic{ALG@line}}
\makeatother

\algdef{SE}[SUBALG]{Indent}{EndIndent}{}{\algorithmicend\ }%
\algtext*{Indent}
\algtext*{EndIndent}

\newcommand{\E}{\mathbb{E}}

\def\cI{\mathcal I}

\def\cJ{\mathcal J}

\theoremstyle{definition}
\newtheorem{thm}{Theorem}[section]
\newtheorem{lemma}{Lemma}[section]
\newtheorem{corollary}{Corollary}[section]
\newtheorem{prop}{Proposition}[section]
\newtheorem{assum}{Assumption}[section]
\newenvironment{thmbis}[1]
{\renewcommand{\theprop}{\ref{#1}$^*$}%
	\addtocounter{prop}{-1}%
	\begin{prop}}
	{\end{prop}}

\newtheorem{rem}{Remark}[section]

\newtheorem{defn}{Definition}[section]

\DeclareMathOperator*{\argmin}{arg\,min}
\DeclareMathOperator*{\argmax}{arg\,max}

\usepackage{titlesec}
\titlespacing*\section{0pt}{2pt plus 2pt minus 2pt}{0pt plus 2pt minus 2pt}
\titlespacing*\subsection{0pt}{2pt plus 2pt minus 2pt}{0pt plus 2pt minus 2pt}
\titlespacing*\subsubsection{0pt}{0pt plus 2pt minus 2pt}{0pt plus 2pt minus 2pt}

\date{\vspace{-5ex}}
\newcommand{\blind}{1}

\begin{document}
\setlength{\abovedisplayskip}{4pt}
\setlength{\belowdisplayskip}{4pt}
\setlength{\abovedisplayshortskip}{2pt}
\setlength{\belowdisplayshortskip}{2pt}

\if1\blind
{
	\title{On Non-Stationary Dynamic Pricing: Adaptivity and Optimality}
	\author{Feiyu Jiang\footnote{School of Management, Fudan University} ~and Zifeng Zhao\footnote{Mendoza College of Business, University of Notre Dame}}
	\date{\today}
	\maketitle
} \fi

\if0\blind
{
	\title{}
	\author{}
	\date{}
	\maketitle
} \fi

\begin{abstract}
We study the contextual dynamic pricing problem under non-stationarity, where a firm sells products to $T$ sequentially arriving consumers that behave according to an unknown demand model that can change over time. The demand model is assumed to be a generalized linear model (GLM), allowing for a feature vector in $\mathbb{R}^d$ that encodes product and consumer information. To achieve optimal revenue (i.e., least regret), the firm needs to learn and exploit the unknown GLMs while monitoring for potential changes. We propose a multiscale change-point detection based algorithm that achieves a regret of order $\widetilde{O}(\sqrt{s_TdT}\wedge\{V_T^{1/3}d^{1/3}T^{2/3}+\sqrt{dT}\})$, where $s_T$ is the number of piecewise stationary segments and $V_T$ is a newly defined notion of design-adjusted variation budget of the model parameters. Our algorithm is adaptive and does not require knowing $s_T$ or $V_T$. Moreover, to our knowledge, this is the first dynamic pricing algorithm that is adaptive to the nature of changes and achieves the best-of-both-worlds rate, thus closing a long-standing gap in the literature. We remark that, due to the varying contexts, existing works in the adaptive non-stationary bandit literature cannot be applied to achieve optimality for contextual dynamic pricing. The regret is further accompanied with a newly constructed minimax lower bound, confirming the optimality of our algorithm (up to logarithmic factors). Extensive numerical experiments are conducted to illustrate the efficiency and robustness of the proposed algorithm in non-stationary dynamic pricing.
\end{abstract}

\noindent\textit{Keywords}: dynamic pricing, multiscale change-point detection, minimax optimality, generalized linear model, total variation, online learning

\section{Introduction}\label{sec:intro}
With technological advances and prevalence of online marketplaces, many firms can now dynamically make pricing decisions while having access to an abundance of contextual information such as consumer characteristics, product features and economic environment. On the other hand, in practice, demand models are usually unknown and firms need to dynamically learn how the contextual information impacts consumer behavior. Thus, to maximize its revenue, the firm needs to implement dynamic pricing, which aims to optimally balance the trade-off between learning the unknown demand function and earning revenues by exploiting the estimated demand model. 

Due to its importance in revenue management, dynamic pricing has been extensively studied in the literature under various settings. The majority of studies on dynamic pricing, however, focus on the case where the demand model is unknown but \textit{stationary}. In other words, it assumes that the way consumers react to prices and product features remains unchanged over time. While this can be a reasonable assumption over a short horizon, empirical evidence suggests that consumer behavior experiences various forms of \textit{changes} over time. First, consumer demand may experience abrupt shifts triggered by exogenous shocks: the onset of Covid-19 pandemic led to sudden changes in household spending patterns~\citep{baker2020does}; similarly, the September 11 attack produced a sharp negative shock to U.S.\ airline demand. At the same time, consumer behavior may also evolve gradually: \cite{parker1997price} show that price sensitivity can vary over product life cycles; in addition, broader market forces such as persistent inflation or growth in employee compensation can gradually alter consumers’ willingness to  pay~\citep{mcouat2024privatebrands}. 

Motivated by these observations, in this paper, we study the dynamic pricing problem under non-stationarity of \textit{unknown} nature. To be specific, we consider a firm selling products to $T$ sequentially arriving consumers, where the $t$th consumer arrives in time period $t$, $t \in \{1, \ldots, T\}$. For each round $t$, the contextual information is featurized by a vector $z_t\in \mathbb R^{d}$, $d \geq 1$. Conditional on the covariate $z_t$ and the price $p_t$, the consumer demand $y_t$ follows a generalized linear model~(GLM) with parameter $\theta_t.$ Due to non-stationarity, the demand model, i.e.,\ the parameter sequence $\{\theta_t\}_{t=1}^T$, may change over time. Importantly, we do not restrict the nature of the non-stationarity and allow it to be either (i).\ \textit{structured} changes where $\{\theta_t\}_{t=1}^T$ is piecewise stationary with $s_T\geq 1$ stationary segments (i.e.\ $s_T-1\geq 0$ abrupt change-points), or (ii).\ \textit{unstructured} changes where the amount of variation in $\{\theta_t\}_{t=1}^T$ is bounded by a budget $V_T>0$ (see Definition \ref{dfn_variation} later for more details). The firm initially has no information about the parameter $\{\theta_t\}_{t=1}^T$. We aim to design a pricing policy that is \textit{adaptive} to the nature of the non-stationarity and achieves near-optimal revenue performance in both worlds, measured by the firm’s $T$-period expected regret. That is, the revenue loss compared to a clairvoyant who has perfect knowledge of the underlying demand model.

\subsection{List of Contributions}
This paper makes the following main contributions:
\begin{itemize}
    \item We design a novel \textit{multiscale change-point detection based dynamic pricing}~(MCP-DP) algorithm, detailed in Algorithm \ref{alg:mcpdp}, which achieves a near-optimal regret of order $\widetilde{O}(\sqrt{s_TdT}\wedge\{\sqrt{dT}+V_T^{1/3}d^{1/3}T^{2/3}\})$ as shown in Theorem \ref{thm:mcpdp-upper}. To the best of our knowledge, MCP-DP is the first algorithm in the dynamic pricing literature that is adaptive and provably optimal under both structured and unstructured non-stationarity, therefore closing a long-standing gap in the literature. Moreover, unlike existing algorithms for dynamic pricing under structured non-stationarity (i.e., abrupt changes), MCP-DP achieves optimality without imposing any condition on the minimum change size and segment length of each stationary period, which is seen for the first time in the literature. In addition, under mild conditions, we derive a new high-probability upper bound on the \textit{prediction} error of MLE for a mixture of GLMs, which provides the foundation for the optimality of MCP-DP and can be of independent interest.

    \item As a key ingredient of MCP-DP, we propose a novel likelihood ratio based multiscale online change-point detection algorithm for GLMs and carefully design the multiscale sampling mechanism and detection threshold to achieve optimal regret. Moreover, we introduce a new notion of \textit{design-adjusted} variation budget, defined in Definition \ref{dfn_variation}, which accounts for the variation in the model parameters $\{\theta_t\}_{t=1}^T$ and its interaction with the design matrix $\Sigma_z:=\mathbb E(z_tz_t^\top)$ of the context. This newly defined variation budget differs from and generalizes existing ones in the literature, which only considers the variation in $\{\theta_t\}_{t=1}^T$. Importantly, it provides the tightest characterization of the amount of changes, thus the sharpest characterization of the optimal regret, for dynamic pricing under unstructured non-stationarity.
    
    \item We show that, for dynamic pricing under non-stationarity, the regret lower bound is $\Omega(\sqrt{s_T dT}\wedge V_T^{1/3}d^{1/3}T^{2/3})$, which simultaneously covers the cases of structured and unstructured changes. Moreover, the dependence on $d$ is seen for the first time in the dynamic pricing literature for both cases, as existing lower bounds only cover the covariate-free settings. To reveal the problem difficulty in terms of $d$, new technical arguments based on Assouad's lemma are developed in the proof, which are of independent interest, as we need to handle the case of diverging $d$ as $T$ grows unbounded. This result is collected in \Cref{sec:lb}.

    \item The theoretical findings are further supported with extensive numerical studies in \Cref{sec:experiments}, illustrating the adaptivity, statistical and computational efficiency of MCP-DP in a wide range of settings for non-stationary dynamic pricing.
\end{itemize}

The rest of this paper is organized as follows. \Cref{subsec:lit} provides a detailed literature review and further highlights our contributions and differences from existing works. We rigorously formulate the problem and further characterize its fundamental difficulty with a minimax lower bound in \Cref{sec:setup_lb}. We introduce the MCP-DP algorithm in \Cref{sec:upper_bound} along with its upper bound and a sketch of proofs. Numerical studies are given in \Cref{sec:experiments} to illustrate the efficiency of MCP-DP. We conclude with discussions in \Cref{sec:conclusion}.  All technical arguments can be found in the supplement.

\subsection{Related Literature}\label{subsec:lit}

Our work is closely related to two streams of literature: dynamic pricing with demand learning and the non-stationary bandit and reinforcement learning problem.

\noindent\textbf{Dynamic pricing with demand learning.} There is a growing body of literature on dynamic pricing under various settings, where a firm aims to maximize revenue by balancing the trade-off between learning and earning~(i.e.,\ exploration and exploitation) when faced with an unknown (but stationary) demand function. See, for example, \cite{besbes2009dynamic}, \cite{farias2010dynamic}, \cite{broder2012dynamic}, \cite{keskin2014dynamic}, \cite{den2015dynamic},  \cite{chen2019welfare,chen2021nonparametric,chen2021joint}, \cite{nambiar2019dynamic}, \cite{wang2021multimodal}, \cite{chen2022differential}, \cite{bastani2022meta},  \cite{luo2024distribution}, \cite{fan2024policy}, \cite{jia2024online}, \cite{chai2024localized}, \cite{tang2025offline}, and \cite{liu2025contextual}. More recently,  \cite{xu2024pricing} and \cite{zhao2026contextual} revisit the contextual dynamic pricing problem under stationary demand models and derive near-optimal regret upper bounds of order $\widetilde{O}(\sqrt{dT})$.

While the majority of studies on dynamic pricing focus on the case where the demand model stays unchanged over time, there are also works tackling non-stationarity. Under structured non-stationarity where the model parameter is piecewise constant,
\cite{besbes2011minimax} study a setting where the demand function may change once in an unknown time period, while assuming that the firm has the knowledge of the pre- and post-change demand functions; \cite{keskin2022data} consider a setting where the seller needs to make joint pricing and inventory decisions under possible abrupt changes in a demand function with a given non-linear form; \cite{chen2019dynamic} investigate the setting where the firm needs to make pricing decisions
for a single product on multiple local markets with evolving market sizes and abrupt changes in the demand function. In contrast, \cite{keskin2017chasing} study non-stationary dynamic pricing for a covariate-free linear demand model under both structured and unstructured changes, and show that the problem complexity is fundamentally different for the two cases. \cite{zhao2026high} further study both cases for the setting where the demand model is a high-dimensional contextual GLM. 

However, all the aforementioned works for non-stationary dynamic pricing are \textit{non-adaptive} in that they require knowing the nature of the change. Indeed, in \cite{keskin2017chasing} and \cite{zhao2026high}, \textit{different} algorithms are designed separately to achieve optimal regret under the structured and unstructured non-stationarity. Moreover, to achieve optimal regret, the aforementioned works impose conditions on the minimum change size and segment length in the case of structured changes, and require the knowledge of the variation budget $V_T$ in the case of unstructured changes. In contrast, thanks to the novel multiscale change-point detection mechanism, our MCP-DP algorithm is truly adaptive and achieves optimal regret for dynamic pricing under both structured and unstructured non-stationarity, without imposing conditions on the change size and segment length or knowing the variation budget. Moreover, as mentioned above, we introduce a new notion of \textit{design-adjusted} variation budget, which generalizes and sharpens the ones defined in \cite{keskin2017chasing} and \cite{zhao2026high}.

\noindent\textbf{Non-stationary bandit and RL.} There is also a rich literature in non-stationary bandit and reinforcement learning~(RL), especially under the multi-armed bandit~(MAB) setting under the name ``switching bandit"~\citep{garivier2011upper}. Same as the non-stationary dynamic pricing literature, most existing works treat unstructured and structured non-stationarity \textit{separately}: for unstructured changes, \cite{besbes2014stochastic} study the non-stationary covariate-free MAB problem and \cite{cheung2022hedging} further extend it to non-stationary linear and generalized linear bandits; for the structured changes, various change-point detection based algorithms are proposed for the non-stationary covariate-free MAB problem, see e.g., \cite{liu2018change, cao2019nearly, besson2022efficient}. 


An important work is \cite{chen2019new}, which proposes one of the first adaptive algorithms in non-stationary contextual bandits that can handle both structured and unstructured changes.  More recently, \cite{suk2022tracking} refines this line of work by introducing a notion of ``significant shifts", showing that in finite-armed bandits, one may only need to restart after changes that are severe enough to affect regret.  However, these results are developed for \textit{finite}  action spaces and rely on arm-level reward comparisons, which incur a regret that scales polynomially with the number of arms $K$, and therefore cannot be directly applied to contextual dynamic pricing, which has a continuous action space (i.e.,\ price). We note that similar ideas of adapting only to statistically relevant changes have also appeared in other settings, such as \cite{huang2025stability}, which consider a general non-stationary stochastic optimization framework under full feedback, and \cite{nguyen2026non}, which study non-stationary Lipschitz bandits. 

\cite{wei2021non} further proposes a black-box reduction that turns a generic UCB-type algorithm for stationary bandit/RL problems into an adaptive algorithm for the non-stationary version. \cite{wang2023adaptivity} further extends \cite{wei2021non} to online stochastic optimization with bandit feedback. However, \cite{wei2021non} hinges on a crucial condition (see Assumption 1 therein) that in general requires a \textit{fixed} action set, as is noted in Appendix I.2 therein. Therefore, it does not work for contextual dynamic pricing as the contexts $\{z_t\}$ (and thus the action set) can vary over time. Thus, our work for adaptive non-stationary dynamic pricing requires significantly different methodology and technical arguments than the adaptive non-stationary bandit/RL literature. We refer to \Cref{rem:comparison_RL} later for a more technical discussion.

\subsection{Notation}
For a vector $a\in\mathbb{R}^d$, denote its Euclidean norm as $\|a\|$.  Denote~$\otimes$ as the Kronecker product. Let $\|\cdot\|_{\mathrm{op}}$, $\lambda_{\min}(\cdot)$ and $\lambda_{\max}(\cdot)$ be the $L_2\to L_2$-operator norm, the smallest and largest eigenvalues of a matrix. For a positive definite matrix $M \succ 0$ and any vector $v$ of compatible dimension, let $\|v\|_{M} = \sqrt{v^{\top} M v}$.  For any positive integer $K$, denote $[K]=\{1, \ldots, K\}$.  
For any vector $v$ and set~$S$, let $\Pi_S(v)$ be the $L_2$-projection of $v$ onto $S$.   We let $O(\cdot)$, $o(\cdot)$, and $\Omega(\cdot)$ denote Landau's Big-O, little-o, and Big-Omega, respectively.  The notation $\widetilde{O}(\cdot)$ disregards poly-logarithmic factors for $O(\cdot)$.  For two nonnegative quantities $a$ and $b$, we write $a\lesssim b$ if there exists a universal constant $C>0$,  such
that $a\le Cb$.  We write $a\gtrsim b$ if $b\lesssim a$, and write
$a\asymp b$ if both $a\lesssim b$ and $a\gtrsim b$ hold.  

\section{Problem Formulation and Lower Bound}\label{sec:setup_lb}
We first present the detailed problem setup in \Cref{sec:setup} and introduce measures of structured and unstructured non-stationarity in \Cref{sec:non}. We then describe the class of pricing policies and the performance metric in \Cref{sec:policy_metric}. Finally, we illustrate the fundamental difficulty of the problem by establishing a minimax lower bound in \Cref{sec:lb}.

\subsection{Model Setup}\label{sec:setup}
We consider a seller who interacts with $T\in\mathbb{N}_{+}$ sequentially arriving consumers. At the beginning of round $t\in[T]$, a context vector
$z_t\in\mathbb R^d$, which may contain information about both the customer and the product, is observed.  After observing $z_t$, the seller posts a
price $p_t$ from a compact interval $[l,u]\subset\mathbb R_+$ and then observes a demand response $y_t\in\mathbb{R}$ from the customer.

Define $x_t=(z_t^{\top},-p_tz_t^{\top})^{\top}$ as the covariate, and $\theta_t=(\alpha_t^{\top},\beta_t^{\top})^{\top}\in\Theta$ as the {true} model parameter at time $t\in[T]$, where $\alpha_t,\beta_t\in\mathbb{R}^d$, and $\Theta\subset\mathbb{R}^{2d}$ is the parameter space. We assume that the demand $y_t$ at time $t$, conditional on $x_t$ and the current parameter
$\theta_t$, follows a canonical GLM
\begin{equation}\label{eq:glm-density}
f_t(y\mid x_t,\theta_t)
=
\exp\left\{
\frac{y\,x_t^\top\theta_t-\psi(x_t^\top\theta_t)}{a(\phi)}+h(y)
\right\},
\end{equation}
where $\psi:\mathbb R\mapsto\mathbb R$ is a known function,  $a(\phi)$ is a fixed and known scale parameter and  $h(y)$ is a   known normalizing
function of the distribution family. The GLM demand model in \eqref{eq:glm-density} has been widely used in the dynamic pricing literature \citep[see e.g.][]{ban2021personalized,xu2024pricing,wang2025dynamic,zhao2026contextual}, which covers the Gaussian linear regression, logistic regression, and Poisson regression, among other important GLMs.

From \eqref{eq:glm-density}, we have that conditional on $x_t$, the expected demand takes the form
\begin{equation}\label{eq:glm-mean}
\mathbb E[y_t\mid x_t,\theta_t]
=
\psi'(x_t^\top\theta_t)
=
\psi'\{z_t^\top\alpha_t-(z_t^\top\beta_t)p_t\},
\end{equation}
where $\psi'(\cdot)$ is the first order derivative of $\psi(\cdot)$ and its inverse is known as the link function of a GLM.  For example, for the linear model, $\psi'(x)\equiv x$ can be used for modeling continuous demand; and for the logistic model, $\psi'(x)=1/[1+\exp(-x)]$ is suitable for modeling probabilities of binary purchase decisions. Here in \eqref{eq:glm-mean}, $z_t^\top\alpha_t$ represents the intrinsic utility of the product and $z_t^\top\beta_t$ represents the price elasticity. We further denote $\epsilon_t:=y_t-\psi'(x_t^\top\theta_t)$ as the error term.

Based on \eqref{eq:glm-mean}, given parameter $\theta$, context $z$, and price $p$, define the expected revenue as
\begin{align}\label{eq:rev_func}
    r(p;z,\theta)=p\,\psi'\!\left(z^\top\alpha-(z^\top\beta)p\right)
\end{align}
and the optimal price as $p^*(z,\theta) \in \argmax_{p\in[l,u]} r(p;z,\theta)$. Furthermore, denote $p_t^*=p^*(z_t,\theta_t)$ as the true optimal price for round $t.$  We now introduce some regularity conditions on the covariate $z_t$, the model parameter $\theta_t$ and the GLM model in \eqref{eq:glm-density}.

\begin{assum}\label{ass_context}
(a).\ The context $\{z_t\}_{t=1}^T$ is a sequence of i.i.d.  random vectors from an unknown distribution supported on a compact set $\mathcal{Z}\subset \{z\in\mathbb{R}^d:\|z\|\leq c_z\}$, where $c_z\geq 1$ is an absolute constant. (b).\ Let $\Sigma_z=\mathbb{E}(z_tz_t^{\top})$, we assume that $\lambda_z:=\lambda_{\min}(\Sigma_z)>0$.
\end{assum}

\begin{assum}\label{ass_glm}
  (a).\ Define $\mathcal{X}=\left\{(z^{\top},-pz^{\top}):z\in\mathcal{Z}, p\in[l,u]\right\}$.  There exists an absolute constant $\sigma>0$ such that for any $x\in\mathcal{X}$, we have $\mathbb{E}[\exp(\lambda \epsilon_t\mid x_t=x)]\leq \exp(\lambda^2\sigma^2/2)$ for any $\lambda>0.$  (b).\ $\psi(\cdot)$ is three times continuously differentiable with a second derivative $\psi''(a)>0$ for all $a\in\mathbb{R}$.
\end{assum}

\begin{assum}\label{ass_boundedness}
(a).\ The parameter space $\Theta$ is a closed and convex set such that $  z^{\top}\alpha, z^{\top}\beta\in[c_l,c_u]$ for all $z\in\mathcal{Z}$ and $\theta \in\Theta$, where $c_l, c_u$ are some absolute constants. (b).\ The optimal price $p_t^*$ is unique and falls into the interior of the compact price range $[l,u]$ for all $t\in [T].$
\end{assum}

We remark that all the above assumptions are mild and standard in the dynamic pricing literature, see e.g.\ \cite{keskin2017chasing}, \cite{ban2021personalized}, \cite{chen2022statistical}, and \cite{zhao2026contextual}. Assumption \ref{ass_context} requires the contexts $\{z_t\}$ to be stochastic, which ensures sufficient variation of the sample design matrix. This is crucial for controlling the concentration of the estimation and prediction error of (mixture of) GLMs, and thus for the validity of the proposed multiscale likelihood-ratio test based algorithm. Note that we only require the covariance matrix $\Sigma_z$ to be full-rank (i.e.\ $\lambda_z>0$) without imposing an absolute lower bound on $\lambda_z$. Indeed, the boundedness of $\|z_t\|$ in \Cref{ass_context}(a) implies that $\lambda_{\max}(\Sigma_z)\leq c_z$ and $\lambda_z\leq c_z/d$. Assumption \ref{ass_glm}(a) imposes sub-Gaussianity on the error term, and Assumption \ref{ass_glm}(b) requires the log-likelihood function to be strictly convex.

\Cref{ass_boundedness}(a) assumes bounded intrinsic utility and price sensitivity, and \Cref{ass_boundedness}(b) requires the optimal price $p_t^*$ to be uniquely achieved in $(l,u)$. In particular, by Berge's maximum theorem \citep{berge1957two} (see e.g.\ Proposition S.5.2 in \cite{zhao2026contextual}), \Cref{ass_boundedness}(b) implies that the optimal price $p^*(z;\theta)$ is a Lipschitz function of $\theta$, which allows to connect the regret with the squared prediction error of the GLM demand model. For concrete examples, suppose the price sensitivity $z_t^\top \beta_t$ is within a compact positive interval, it is easy to see that the optimal price is unique and takes the closed-form $p_t^*={z_t^\top \alpha_t}/{(2z_t^\top \beta_t)}$ for the Gaussian linear regression and $p_t^*$ is the unique solution to $1+\exp\left(z_t^\top\alpha_t - (z_t^\top\beta_t) p\right)-(z_t^\top\beta_t) p=0$ for the logistic regression. Furthermore, we can always find a bounded price range $[l,u]$ that covers the optimal price $p_t^*$ for all $t \in [T]$ given that the intrinsic utility $z_t^\top \alpha_t$ is bounded.


\subsection{Measures of Non-stationarity}\label{sec:non}
We allow the demand model to be non-stationary, in the sense that the true model parameter $\{\theta_t\}$ may vary over time. In the following, we distinguish between two types of non-stationarity commonly used in the non-stationary bandit and dynamic pricing literature.

\Cref{dfn_st} defines \textit{structured} non-stationarity~\citep[e.g.][]{keskin2017chasing, chen2019dynamic, zhao2026high}, where $\{\theta_t\}$ assumes a piecewise stationary structure that may encounter abrupt shifts.
\begin{defn}[Piecewise stationarity]\label{dfn_st}
The parameter sequence $\{\theta_t\}_{t\in\mathcal{I}}$ is piecewise constant on  $\mathcal{I}=[\mathcal{I}_a,\mathcal{I}_b]\subseteq [T]$ if there exist $s_{\mathcal{I}}-1 \geq 0$ unknown change-points $\mathcal{I}_a=\tau_1<\tau_2<\cdots<\tau_{s_{\mathcal{I}}}<\tau_{s_{\mathcal{I}}+1}=\mathcal{I}_b+1$, such that, for each segment $j\in[s_{\mathcal{I}}]$,
\[
    \theta_t=\theta^{(j)}, \qquad t=\tau_j,\ldots,\tau_{j+1}-1 .
\]
Here, $s_{\mathcal{I}}:=1+\sum_{t=\mathcal{I}_a+1}^{\mathcal{I}_b} \mathbb I\{\theta_t\neq\theta_{t-1}\}$ is the number of stationary segments and we write $s_T:=s_{[1:T]}$.
\end{defn}

\Cref{dfn_variation} defines \textit{unstructured} non-stationarity~\citep[e.g.][]{keskin2017chasing,cheung2019learning}, where $\{\theta_t\}$ is allowed to vary freely subject to a variation budget. This formulation can capture smoother forms of temporal evolution, such as gradual shift in consumer demand.

\begin{defn}[Design-adjusted variation budget]\label{dfn_variation}
Denote $\widetilde{\Sigma}_z=I_2\otimes\Sigma_z \in \mathbb R^{2d\times 2d}$ as the weight matrix induced by the design matrix $\Sigma_z$ of the context $\{z_t\}$. For $\mathcal I=[\mathcal{I}_a,\mathcal{I}_b]\subseteq[T]$, define the design-adjusted variation of the parameter sequence $\{\theta_t\}_{t\in\mathcal{I}}$ as
\begin{align}\label{eq:variation_budget}
    V_{\mathcal I}
    =
    \sup_{\{t_0,\ldots,t_m\}\in\mathcal P_{\mathcal I}}
    \sum_{q=1}^m
    \|\theta_{t_q}-\theta_{t_{q-1}}\|_{\widetilde\Sigma_z}^2,
\end{align}
where $\mathcal{P}_{\mathcal I}$ is the set of all ordered partitions of $\mathcal I=[\mathcal{I}_a,\mathcal{I}_b]=\bigcup_{q=1}^{m}[t_{q-1},t_{q}]$ such that $\{t_0,t_1,\ldots,t_m\}$ satisfies $\mathcal{I}_a= t_0<\ldots<t_m= \mathcal{I}_b$ for some $m=1,2,\ldots,\mathcal{I}_b-\mathcal{I}_a$. We write $V_T:=V_{[1,T]}.$
\end{defn}

We remark that the newly-defined notion of \textit{design-adjusted} variation budget \eqref{eq:variation_budget} in \Cref{dfn_variation} is a generalization and sharper version of the existing \textit{unweighted} variation budget
$$V_{\cI}^{(2)}=\sup_{\{t_0,\ldots,t_m\}\in\mathcal P_{\mathcal I}}\sum_{q=1}^m\|\theta_{t_q}-\theta_{t_{q-1}}\|^2$$ in the non-stationary dynamic pricing literature~\citep[e.g.][]{keskin2017chasing}. Note that we have
\[
    \lambda_{\min}(\Sigma_z) V_{\mathcal I}^{(2)}
    \le
    V_{\mathcal I}
    \le
    \lambda_{\max}(\Sigma_z) V_{\mathcal I}^{(2)} .
\]
By the boundedness of $\{z_t\}$ in \Cref{ass_context}, we have $\lambda_{\min}(\Sigma_z)\leq c_z/d$ and $\lambda_{\max}(\Sigma_z)\leq c_z$ where $c_z>0$ is an absolute constant. Therefore, $V_{\mathcal{I}}$ is strictly sharper than $V_{\mathcal{I}}^{(2)}$ in terms of order and, in particular, we have $V_{\mathcal{I}}\asymp V_{\mathcal{I}}^{(2)}/d$ given a balanced design $\lambda_{\min}(\Sigma_z)\asymp \lambda_{\max}(\Sigma_z) \asymp 1/d.$

The unweighted $V_{\cI}^{(2)}$ only measures the variation in the parameter $\{\theta_t\}_{\mathcal{I}}$, which is natural for context-free dynamic pricing as in \cite{keskin2017chasing}. However, for contextual dynamic pricing, the impact of $\theta_t$ to the demand $y_t$~(see \eqref{eq:glm-mean}) is manifested via $z_t^\top\alpha_t$ and $z_t^\top\beta_t$, and hence should intrinsically depend on the distribution of the context $\{z_t\}$. The design-adjusted variation budget $V_{\mathcal{I}}$ captures this distinction via the weight matrix $\widetilde\Sigma_z$. In particular, for two parameters
$\theta=(\alpha^\top,\beta^\top)^\top$ and
$\theta'=(\alpha'^\top,\beta'^\top)^\top$, we have that 
\[
    \|\theta-\theta'\|_{\widetilde\Sigma_z}^2
    =
    \mathbb E\{z_t^\top(\alpha-\alpha')\}^2
    +
    \mathbb E\{z_t^\top(\beta-\beta')\}^2.
\]
Thus, the design-adjusted variation in \eqref{eq:variation_budget} captures the overall change in the intrinsic utility and price elasticity averaged over the context distribution. Specifically, $V_{\mathcal{I}}$ reflects the intuition that if a change in $\theta_t$ occurs mostly along directions that are rarely represented by the contexts, then it is less consequential for the regret, since it has little effect on the demand, and vice versa. Therefore, $V_{\mathcal{I}}$ provides a sharper characterization of the effective amount of non-stationarity in $\{\theta_t\}_{t\in \mathcal{I}}$. For a concrete example where $V_{\mathcal{I}}$ generalizes $V_{\mathcal{I}}^{(2)}$, we refer to \Cref{subsec:baseline_exp} of the numerical experiments.

In practice, the firm may not know which world it lives in when it comes to the two types of non-stationarity. Therefore, it is important to design an efficient dynamic pricing algorithm that is adaptive and achieves the best-of-both-worlds rate, without requiring the knowledge of the nature of the non-stationarity, nor the knowledge of $s_T$ or $V_T.$


\subsection{Pricing Policies and Performance Metric}\label{sec:policy_metric}

In a dynamic pricing problem, the key objective is to design a pricing policy that maximizes the revenue. As discussed above, we assume that the firm has no knowledge of the non-stationary environment.  To be specific, the firm does not know (1) the type of non-stationarity, (2) the number of stationary segments $s_T$ or variation budget $V_T$, or (3) the model parameter $\{\theta_t\}_{t=1}^T$.

Denote $\mathcal F_{t-1}=\sigma\left(\{(z_s,p_s,y_s):1\le s\le t-1\}\cup \{z_t\}\right)$ as the natural filtration generated by the demand, price and covariate history up to time $t$, together with the current covariate $z_t.$ Denote by $\Pi$ the family of all price processes $\{p_1, \ldots, p_T\}$ satisfying the condition that $p_t$ is $\mathcal{F}_{t-1}$-measurable for all $t\in [T].$ In other words, we require the pricing policy to be non-anticipating. 

Given a pricing policy $\pi\in\Pi$, we evaluate its performance using the common notion of regret: the revenue loss compared with a clairvoyant that has the perfect knowledge of the model parameter $\{\theta_t\}_{t=1}^T$. In particular, with the revenue function $r(\cdot; \cdot, \cdot)$ defined in~\eqref{eq:rev_func}, we define the regret as
\begin{align}\label{dfn_regret}
    R_T^\pi = \sum_{t=1}^Tr(p_t^*;z_t,\theta_t)-r(p_t;z_t,\theta_t),
\end{align}
where recall $p_t^*=p^*(z_t,\theta_t)=\argmax_{p\in[l,u]}r(p;z_t,\theta_t)$ is the revenue-maximizing price at round $t.$

\subsection{Minimax Lower Bound}\label{sec:lb}

\Cref{thm:lb} establishes the fundamental difficulty of the non-stationary contextual dynamic pricing problem, which, for the first time in the literature, fully characterizes the impact of the horizon $T$, the context dimension $d$, and the non-stationarity measure $s_T$ and $V_T$ on the regret $R_T^\pi$.
\begin{thm}[Minimax lower bound]\label{thm:lb}
Let $\mathcal M(s,V)$ denote the class of non-stationary contextual dynamic pricing models satisfying Assumptions \ref{ass_context}-\ref{ass_boundedness}, with at most $s$ stationary segments and at most $V$ design-adjusted variation budget. There exists an absolute constant $c_{lb}>0$ such that, for any non-anticipating
pricing policy $\pi$, and $T \gtrsim s d\log(dT) \wedge  dV\log^{3/2}(dT),$ it holds that
\[
    \inf_{\pi\in\Pi}\sup_{\mathbb P\in\mathcal M(s,V)}
    \mathbb{E}_{\mathbb P}R_T^\pi
    \ge
    c_{lb}\left\{
        \sqrt{sd T}
        \wedge
        (\sqrt{dT}+d^{1/3}V^{1/3}T^{2/3})
    \right\}.
\]
\end{thm}

\Cref{thm:lb} unifies and generalizes three existing lower bound results in the dynamic pricing literature. First, when $s=1$ and $V=0$, \Cref{thm:lb} reduces to the $\Omega(\sqrt{dT})$ bound for stationary contextual dynamic pricing with a $d$-dimensional context as in \cite{xu2024pricing,zhao2026contextual}. Second, \Cref{thm:lb} reduces to the $\Omega(\sqrt{sT})$ bound for context-free (i.e.\ $d=1$) dynamic pricing under structured unstationarity with $s$ stationary segments as in \cite{zhao2026high}. Third, \Cref{thm:lb} reduces to the $\Omega\left(\left\{V^{(2)}\right\}^{1/3}T^{2/3}\right)$ bound for context-free dynamic pricing under unstructured unstationarity with a variation budget $V^{(2)}$ (note that $V^{(2)}\asymp V$ for $d=1$) as in \cite{keskin2017chasing}. Importantly, \Cref{thm:lb} explicitly characterizes the impact of $d$ on the regret, which is the first time seen in the non-stationary dynamic pricing literature.

In particular, to reveal the problem difficulty in terms of the dimension $d$, we construct a new class of hard problem instances, which consists of a sequence of linear demand models that can be grouped into highly indistinguishable pairs. Our construction is complex and, specifically, is of size $2^d$ times larger than the ones constructed under the context-free setting in \cite{keskin2017chasing} and \cite{zhao2026high}. The lower bound in \Cref{thm:lb} is further established with the Assouad's lemma~\citep[e.g.][]{tsybakov2009} by connecting the regret with the error of a multiple-classification problem. In contrast, without contexts, \cite{keskin2017chasing} and \cite{zhao2026high} only require the simpler Le Cam's lemma, which quantifies the error of a two-sample testing problem. We refer to Section \ref{sec:proof_of_lb} of the supplement for the detailed proof of \Cref{thm:lb}.

\section{A Multiscale Change-point Detection Based Algorithm}
\label{sec:upper_bound}
In this section, we propose a multiscale change-point detection based dynamic pricing algorithm, hereafter MCP-DP, which simultaneously achieves adaptivity to the nature of the non-stationarity and the optimal regret upper bound (up to logarithmic terms).

The adaptivity and optimality of MCP-DP are achieved by combining an explore-then-commit~(ETC) scheme, which ensures policy stability in the exploitation stage, with a multiscale change-point detection procedure based on a newly designed likelihood-ratio test~(LRT), which monitors non-stationarity and triggers restarts when necessary. \Cref{subsec:mcpdp-description} gives an overview of the MCP-DP algorithm, \Cref{subsec:mcpdp-detection} presents the LRT-based multiscale change-point detection subroutine, and \Cref{subsec:mcpdp-upper} establishes the regret upper bound of MCP-DP with a sketch of proof.

 \subsection{Algorithm Overview}
\label{subsec:mcpdp-description}
MCP-DP is formally presented in \Cref{alg:mcpdp}. To build intuition, we first give a high-level overview with a graphical illustration in \Cref{fig:mcpdp-graph}, before going into details of the algorithm.

MCP-DP proceeds by epochs. Within each epoch (say epoch $i$), the horizon is partitioned into dyadic blocks ($\mathcal{B}_{i,j}, j\in \mathbb N)$. To optimally handle the potential non-stationarity, within each block $\mathcal{B}_{i,j}$, MCP-DP implements an ETC dynamic pricing policy equipped with a multiscale change-point detection subroutine. In particular, it first learns a reference demand model, denoted by $\widehat\theta_{i,j-1}$, from the price exploration set accumulated in the previous block. In parallel, it carefully schedules an index set $\mathcal{S}_{i,j}$ of multiscale exploration intervals, whose union is denoted by $\cJ_{i,j}$, and a mandatory block-end exploration interval $\mathcal{E}_{i,j}$. Then, for rounds in $\mathcal{B}_{i,j}\setminus \{\mathcal{J}_{i,j}\cup \mathcal{E}_{i,j}\}$, MCP-DP runs a greedy pricing exploitation based on the estimated reference model $\widehat\theta_{i,j-1}$; while for rounds in $\mathcal{J}_{i,j}\cup \mathcal{E}_{i,j}$, MCP-DP runs a localized price exploration procedure to accumulate information for conducting change-point tests and model estimation. If all change-point tests conducted on $\mathcal{J}_{i,j}\cup \mathcal{E}_{i,j}$ are passed, MCP-DP moves to the next block and repeats the process; otherwise, it ends the current epoch and restarts with a new epoch.

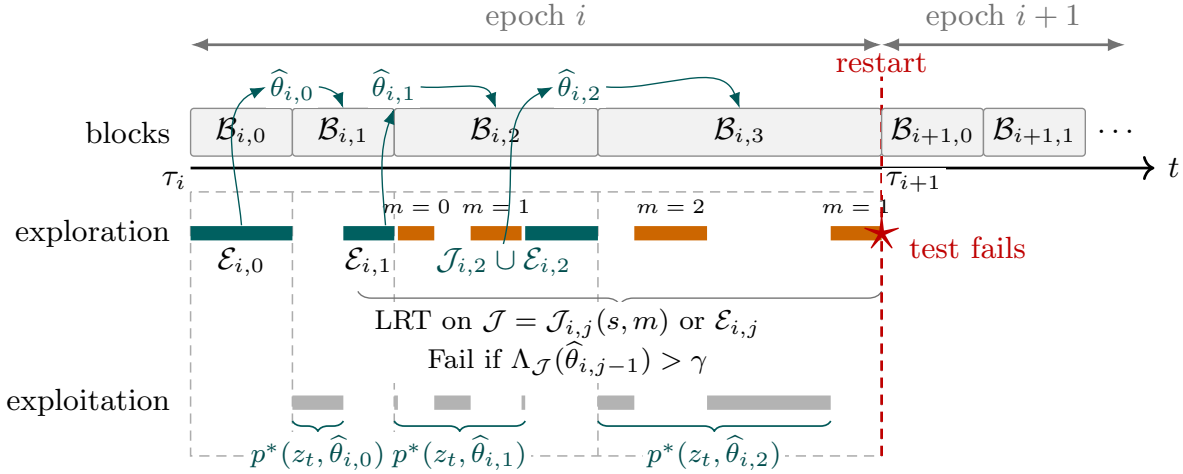
\begin{figure}[!htp]
\centering
\resizebox{0.96\linewidth}{!}{%
\begin{tikzpicture}[
    x=0.75cm,y=0.82cm,
    font=\footnotesize,
    every node/.style={align=center},
    axis/.style={->,thick},
    block/.style={draw=black!45,fill=black!5,rounded corners=1pt},
    base/.style={draw=black!45,fill=black!5,rounded corners=1pt},
    update/.style={draw=teal!70!black,fill=teal!17,rounded corners=1pt},
    greedy/.style={line width=4.2pt,black!30},
    sched/.style={line width=4.2pt,orange!80!black},
    schedlong/.style={line width=4.2pt,orange!80!black},
    endwin/.style={line width=4.2pt,teal!75!black},
    explabel/.style={anchor=base,inner sep=1pt,font=\footnotesize},
    mlabel/.style={fill=white,inner sep=0.2pt,font=\tiny},
    guide/.style={draw=black!32,densely dashed,line width=0.45pt},
    callout/.style={fill=white,text=black,inner sep=2pt,font=\scriptsize,align=center},
    estarrow/.style={-{Latex[length=1.6mm]},thin,teal!70!black},
    estbrace/.style={decorate,decoration={brace,mirror,amplitude=3pt},teal!70!black,line width=0.5pt},
    estlabel/.style={fill=white,text=teal!60!black,inner sep=1pt,font=\scriptsize},
    restart/.style={red!75!black,thick,densely dashed}
]
    \node[left] at (-0.1,0.45) {blocks};
    \node[left] at (-0.1,-0.82) {exploration};
    \node[left] at (-0.1,-2.95) {exploitation};
    \draw[axis] (0,0) -- (13.25,0) node[right] {$t$};
    \node[below left,fill=white,inner sep=1pt] at (0,0) {$\tau_i$};

    \draw[base] (0,0.15) rectangle (1.4,0.75);
    \node at (0.7,0.45) {$\mathcal B_{i,0}$};

    \draw[block] (1.4,0.15) rectangle (2.8,0.75);
    \node at (2.1,0.45) {$\mathcal B_{i,1}$};
    \draw[block] (2.8,0.15) rectangle (5.6,0.75);
    \node at (4.2,0.45) {$\mathcal B_{i,2}$};
    \draw[block] (5.6,0.15) rectangle (9.5,0.75);
    \node at (7.55,0.45) {$\mathcal B_{i,3}$};

    \draw[guide] (0,-0.30) rectangle (9.5,-3.62);
    \draw[guide] (1.4,-0.30) -- (1.4,-3.62);
    \draw[guide] (2.8,-0.30) -- (2.8,-3.62);
    \draw[guide] (5.6,-0.30) -- (5.6,-3.62);

    \draw[endwin] (0.00,-0.82) -- (1.40,-0.82);
    \node[explabel] at (0.70,-1.27) {$\mathcal E_{i,0}$};
    \draw[endwin] (2.10,-0.82) -- (2.80,-0.82);
    \node[explabel] at (2.45,-1.27) {$\mathcal E_{i,1}$};
    \draw[sched] (2.85,-0.82) -- (3.35,-0.82);
    \node[mlabel] at (3.10,-0.48) {$m=0$};
    \draw[schedlong] (3.85,-0.82) -- (4.55,-0.82);
    \node[mlabel] at (4.20,-0.48) {$m=1$};
    \draw[endwin] (4.60,-0.82) -- (5.60,-0.82);
    \node[explabel,text=teal!60!black] (pool2) at (4.30,-1.27)
        {$\mathcal J_{i,2}\cup\mathcal E_{i,2}$};
    \draw[schedlong] (6.10,-0.82) -- (7.10,-0.82);
    \node[mlabel] at (6.60,-0.48) {$m=2$};
    \draw[schedlong] (8.80,-0.82) -- (9.50,-0.82);
    \node[mlabel] at (9.15,-0.48) {$m=1$};
    \node[red!75!black,font=\huge] at (9.50,-0.82) {$\star$};
    \node[right,text=red!75!black] at (9.70,-1.00) {test fails};

    \draw[decorate,decoration={brace,mirror,amplitude=4pt},black!55]
        (2.30,-1.55) -- (9.50,-1.55)
        node[pos=0.40,below=3pt,callout]
        {LRT on $\cJ=\cJ_{i,j}(s,m)$ or $\mathcal E_{i,j}$\\
        Fail if $\Lambda_{\cJ}(\widehat\theta_{i,j-1})>\gamma$};

    \draw[greedy] (1.40,-2.95) -- (2.10,-2.95);
    \draw[greedy] (2.80,-2.95) -- (2.85,-2.95);
    \draw[greedy] (3.35,-2.95) -- (3.85,-2.95);
    \draw[greedy] (4.55,-2.95) -- (4.60,-2.95);
    \draw[greedy] (5.60,-2.95) -- (6.10,-2.95);
    \draw[greedy] (7.10,-2.95) -- (8.80,-2.95);

    \draw[estbrace] (1.40,-3.18) -- (2.10,-3.18)
        node[midway,below=3pt,estlabel] {$p^*(z_t,\widehat\theta_{i,0})$};
    \draw[estbrace] (2.80,-3.18) -- (4.60,-3.18)
        node[midway,below=3pt,estlabel] {$p^*(z_t,\widehat\theta_{i,1})$};
    \draw[estbrace] (5.60,-3.18) -- (8.80,-3.18)
        node[midway,below=3pt,estlabel] {$p^*(z_t,\widehat\theta_{i,2})$};

    \node[estlabel] (theta0) at (1.40,1.02) {$\widehat\theta_{i,0}$};
    \node[estlabel] (theta1) at (2.80,1.02) {$\widehat\theta_{i,1}$};
    \node[estlabel] (theta2) at (5.35,1.02) {$\widehat\theta_{i,2}$};
    \draw[estarrow] (0.70,-0.72) to[out=95,in=205] (theta0.west);
    \draw[estarrow] (theta0.east) to[out=-5,in=110] (2.10,0.75);
    \draw[estarrow] (2.68,-0.72) to[out=95,in=250] (theta1.south);
    \draw[estarrow] (theta1.east) to[out=-5,in=110] (4.20,0.75);
    \draw[estarrow] (pool2.north) to[out=75,in=205] (theta2.west);
    \draw[estarrow] (theta2.east) to[out=-5,in=110] (7.55,0.75);

    \draw[restart] (9.50,-3.62) -- (9.50,1.15);
    \node[above,fill=white,text=red!75!black,inner sep=1pt] at (9.50,1.15) {restart};
    \draw[base] (9.50,0.15) rectangle (10.90,0.75);
    \node at (10.20,0.45) {$\mathcal B_{i+1,0}$};
    \draw[block] (10.90,0.15) rectangle (12.30,0.75);
    \node at (11.60,0.45) {$\mathcal B_{i+1,1}$};
    \node at (12.75,0.45) {$\cdots$};
    \node[below right,fill=white,inner sep=1pt] at (9.50,0) {$\tau_{i+1}$};

    \draw[{Latex[length=1.8mm]}-{Latex[length=1.8mm]},thick,black!55]
        (0,1.55) -- (9.50,1.55) node[midway,above] {epoch $i$};
    \draw[{Latex[length=1.8mm]}-{Latex[length=1.8mm]},thick,black!55]
        (9.50,1.55) -- (12.85,1.55) node[midway,above] {epoch $i+1$};

\end{tikzpicture}%
}
\caption{An illustration of MCP-DP in \Cref{alg:mcpdp}. The top row shows the dyadic blocks within each epoch. The next two rows separate the rounds within the blocks: (i).\ the exploration row consists of the mandatory block-end exploration (or initial exploration) intervals ($\mathcal{E}_{i,j}$, colored in \protect\textcolor{teal!75!black}{\protect\rule{2ex}{1ex}}) and the multiscale exploration intervals scheduled via \Cref{alg:multiscale-sampling} ($\cJ_{i,j}(s,m)$, colored in \protect\textcolor{orange!80!black}{\protect\rule{2ex}{1ex}}); (ii).\ the exploitation row consists of the greedy-pricing periods (colored in \protect\textcolor{black!30}{\protect\rule{2ex}{1ex}}). The reference estimator $\widehat\theta_{i,j}$ is estimated based on price exploration in $\mathcal{J}_{i,j}\cup\mathcal{E}_{i,j}.$
Both the scheduled multiscale exploration intervals and the block-end exploration intervals are examined via the likelihood-ratio based change-point test~(LRT) in \Cref{alg:lrt}. A failed test, marked by a red star (\protect\textcolor{red!75!black}{\protect{$\star$}}), triggers a restart. Here $\mathcal S_{i,j}$ indexes the scheduled multiscale pairs $(s,m)$, whereas $\cJ_{i,j}$ denotes the union of all exploration rounds indexed by $\mathcal S_{i,j}$.}
\label{fig:mcpdp-graph}
\end{figure}

We now give a detailed walkthrough of MCP-DP. At the beginning of epoch $i$ with start time $\tau_i$, since no information is available, MCP-DP initializes with a base exploration interval, $\mathcal B_{i,0}=\mathcal E_{i,0}=[\tau_i+1,\tau_i+L],$ which purely conducts price exploration. Here, $L=c_Ld\log(dT)$ sets the base block length, where $c_L>0$ is a constant tuning parameter. The remaining periods of the epoch then continue via dyadic blocks
\[
    \mathcal B_{i,j}
    =
    [\tau_i+2^{j-1}L+1,\tau_i+2^jL],
    \qquad j\ge 1 .
\]
As is outlined in \Cref{alg:mcpdp}, there are four key components of the block-wise operation of MCP-DP within each $\mathcal{B}_{i,j}$: (I) reference parameter estimation; (II) multiscale exploration scheduling; (III) price exploration and exploitation; and (IV) change-point detection. We now elaborate on these components.

\noindent\textbf{(I) reference model estimation.}
For an interval $\cI\subseteq[T]$, define 
\begin{equation}\label{eq:likelihood}
       \mathcal{L}_{\cI}(\theta)
    =
    \sum_{t\in\cI}
    \psi(x_t^\top\theta)-y_t x_t^\top\theta,
    \qquad
    x_t=(z_t^\top,-p_tz_t^\top)^\top
\end{equation}
as the negative log-likelihood function. Let $\widehat\theta_{\cI}=\argmin_{\theta\in\Theta} \mathcal{L}_{\cI}(\theta)$
be the corresponding MLE (note that $\widehat\theta_{\cI}$ is well-defined as $\Theta$ is a closed convex set and $\mathcal{L}_{\cI}(\theta)$ is a strictly convex function).

MCP-DP sets the reference model for $\mathcal{B}_{i,j}$ as $\widehat\theta_{i,j-1}:=\widehat\theta_{\mathcal E_{i,j-1}\cup \mathcal{J}_{i,j-1}}$, which is the MLE estimated based on the price exploration set $\mathcal{E}_{i,j-1}\cup \mathcal{J}_{i,j-1}$ accumulated in the previous block (see (II) for the detailed definition). MCP-DP then uses $\widehat\theta_{i,j-1}$ for greedy pricing in the price exploitation rounds of $\mathcal{B}_{i,j}$, while at the same time monitoring its performance via an LRT change-point test in the price exploration rounds of $\mathcal{B}_{i,j}$.


\noindent\textbf{(II) multiscale exploration scheduling.}
Due to the potential non-stationarity, it is clearly sub-optimal to blindly trust the reference model $\widehat\theta_{i,j-1}$ without tracking its performance. To this end, at the start of each dyadic block $\mathcal B_{i,j}$, MCP-DP calls the scheduler in \Cref{alg:multiscale-sampling} to (randomly) sample an index set $\mathcal{S}_{i,j}$ of (start, scale) pairs. For each pair
$(s,m)\in\mathcal S_{i,j}$, $s$ denotes its starting time and $m\in \{0,1,\cdots,j-1\}$ denotes its scale. Together, it forms a price exploration interval
\begin{align}\label{eq:def_J}
    \cJ_{i,j}(s,m)=[s,s+\ell_m-1]\cap\mathcal B_{i,j}, \quad\text{ with } \ell_m=\sqrt{d\log(dT)\times 2^m L}.
\end{align}
We write $\cJ_{i,j}:=\bigcup_{(s,m)\in\mathcal S_{i,j}}\cJ_{i,j}(s,m)$ as the union of all such exploration rounds. 

On each sampled $\cJ_{i,j}(s,m)$, MCP-DP then conducts localized price exploration to accumulate information and conducts a change-point test to track the performance of the reference model $\widehat\theta_{i,j-1}.$ The sampling probabilities of $(s,m)$ in \Cref{alg:multiscale-sampling} are carefully chosen so that shorter intervals (smaller $m$) are sampled more frequently whereas longer intervals (larger $m$) less, which creates sufficient monitoring coverage over $\mathcal{B}_{i,j}$ while keeping the number of price exploration rounds manageable. The multiscale nature of the exploration intervals indexed by $\mathcal{S}_{i,j}$ serves as a key component for the adaptivity of MCP-DP and we provide a more detailed discussion of this point in \Cref{subsec:mcpdp-detection}.

In addition, a mandatory block-end exploration interval $\mathcal{E}_{i,j}=[\tau_i+2^jL-\ell_j+1,\tau_i+2^jL]$ is further scheduled, which is designed to guarantee a minimum sample size of the price exploration rounds in $\mathcal{B}_{i,j}$, i.e.\ at least $\ell_j$, and thus ensure that the next reference model estimator $\widehat\theta_{i,j}$ is of sufficiently low variance. 

\noindent\textbf{(III) price exploration and exploitation.}
For each round in $\mathcal{B}_{i,j}\setminus \{\mathcal{J}_{i,j}\cup \mathcal{E}_{i,j}\}$, MCP-DP runs an ETC-style price exploitation by posting the greedy price $p^*(z_t,\widehat\theta_{i,j-1})$ based on the reference model $\widehat\theta_{i,j-1}$. For each round in $\mathcal{J}_{i,j}\cup \mathcal{E}_{i,j}$, MCP-DP runs a localized price exploration scheme. In particular, given a {fixed} deviation $\rho \in (0, (u-l)/2),$ it posts a perturbed greedy price at
\begin{align}\label{eq:local_price_exploration}
    p_t= \operatorname{clip}_{[l+\rho,u-\rho]}\left\{p^*(z_t,\widehat\theta_{i,j-1})\right\}+\rho\eta_t,
\end{align}
where $\operatorname{clip}_{[l+\rho,u-\rho]}(p):=(p\vee (l+\rho))\wedge(u-\rho)$ and $\eta_t$ is a Rademacher random variable.


The statistical information accumulated in the price exploration rounds are then used to (i).\ continuously track the performance of the current reference model $\widehat\theta_{i,j-1}$ in block $\mathcal{B}_{i,j}$ via a change-point test (see (IV)); (ii).\ provide a good estimation of the reference model $\widehat\theta_{i,j}$ for the next block.

In the non-stationary dynamic pricing literature, price exploration is commonly implemented via a standard price experiment, where one uniformly samples price from a pre-fixed price set such as $[l,u]$~\citep[e.g.][]{keskin2017chasing,zhao2026high}. In contrast, the localized price exploration scheme proposed in \eqref{eq:local_price_exploration} is more delicate. Intuitively, it can incur substantially less regret, especially for periods with low non-stationarity, as the posted price $p_t$ now centers around the greedy price $p^*(z_t,\widehat\theta_{i,j-1})$, which is the estimated optimal price that incorporates the context $z_t$. 
We note that a similar localized perturbation scheme is used in \cite{chai2024localized} for stationary contextual dynamic pricing to reduce the revenue loss from learning. Importantly, \Cref{lem:greedy-centered-information} shows that the localized price exploration scheme remains statistically valid in that it accumulates the same amount of statistical information~(measured by the eigenvalues of the design matrix) as a standard price experiment, which is key to the success of the LRT based change-point test.

\begin{lemma}\label{lem:greedy-centered-information}
Suppose Assumptions \ref{ass_context}-\ref{ass_boundedness} hold. There exist absolute constants $0<\underline\kappa\le\overline\kappa<\infty$, depending only on $u,l,\rho$, such that, for any
$\theta\in\Theta$ that is measurable w.r.t. $\mathcal{F}_{t-1}\backslash \{z_t\}$, we have
\begin{align}\label{eq:eigenvalue_lb_x}
    \underline\kappa \widetilde\Sigma_z
    \preceq
    \mathbb E[x_tx_t^\top\mid \theta]
    \preceq
    \overline\kappa \widetilde\Sigma_z,
\end{align}
where $p_t=\operatorname{clip}_{[l+\rho,u-\rho]}\{p^*(z_t,\theta)\}+\rho\eta_t$ is the perturbed price, $x_t=(z_t^\top,-p_tz_t^\top)^\top$ and $\widetilde\Sigma_z=I_2\otimes\Sigma_z$. 
\end{lemma}



\noindent\textbf{(IV) change-point test.}
When $t$ reaches the end of a scheduled multiscale exploration interval indexed by $\mathcal{S}_{i,j}$ or the end of the mandatory block-end exploration interval $\mathcal{E}_{i,j}$, MCP-DP conducts an LRT-based change-point test in \Cref{alg:lrt} to track the performance of $\widehat\theta_{i,j-1}$ on this interval (i.e.\ the regret it incurs). If the test fails, which suggests the presence of excessive non-stationarity in $\{\theta_t\}$, the current epoch is then terminated and MCP-DP restarts a new epoch. Otherwise, the algorithm proceeds to the next round. The change-point test serves as a key component for the optimality of MCP-DP and we provide a more detailed discussion in \Cref{subsec:mcpdp-detection}.

\begin{algorithm}[tbp]
\begin{algorithmic}[1]
    \State \textbf{Input}: Total rounds $T$, price interval $[l,u]$,
    parameter space $\Theta$, tuning parameters $c_L,c_{\gamma}$, $\rho$.
    \vskip 1mm
    
    \State \textbf{Initialization}: set $t=1$, $i=1$; set base block length as $L=c_Ld\log(dT)$, set multiscale block length as
    $\ell_m= \sqrt{d\log(dT)\times 2^mL}$ for $m\geq 0$, and set LRT threshold as $\gamma=c_{\gamma}d\log(dT)$.
    \vskip 1mm
    
    \For{each epoch $i=1,2,3\cdots$ (until $T$ is reached)}
        \State [$\triangleright$ Epoch initialization]
        \Indent
        \State set $\tau_i=t-1$, $\mathcal{B}_{i,0}=[\tau_i+1,\tau_i+L]$, $\mathcal E_{i,0}:=\mathcal B_{i,0}$, $\mathcal{J}_{i,0}=\varnothing.$
        \For{$t=\tau_i+1,\cdots,\tau_i+L$}
        \State uniformly sample price $p_t$ from $[l,u]$, record $y_t$ and $x_t=[z_t^\top,-z_t^\top p_t]^\top$.
        \EndFor
        \EndIndent
        \State [$\triangleright$ Block-wise ETC dynamic pricing with multiscale change-point detection]
        \Indent
            \For{$j=1,2,\cdots$}
        \State set the $j$th block as $\mathcal{B}_{i,j}=[\tau_i+2^{j-1}L+1,\tau_i+2^jL]$.

        \State \fbox{$\triangleright$ I.\ reference parameter estimation}
        \State set $\widehat\theta_{i,j-1}$ as the MLE of \eqref{eq:likelihood} based on the price exploration data in $\mathcal{J}_{i,j-1}\cup \mathcal{E}_{i,j-1}$.

        \State \fbox{$\triangleright$ II.\ multiscale exploration scheduling}
        \State set $\mathcal{S}_{i,j}$ as the output of the multiscale scheduling scheme $\text{MSS}(\tau_i,j,L)$ in \Cref{alg:multiscale-sampling}.
        \State set $\mathcal{J}_{i,j}=\bigcup_{(s,m)\in \mathcal{S}_{i,j}} \mathcal{J}_{i,j}(s,m)$ with $\mathcal{J}_{i,j}(s,m)=[s,s+\ell_m-1]\cap \mathcal{B}_{i,j}$.
        \State set the mandatory block-end exploration interval as $\mathcal{E}_{i,j}=[\tau_i+2^jL-\ell_j+1,\tau_i+2^jL].$     

        \State \fbox{$\triangleright$ III.\ price exploration and exploitation}
        \While{$t\le \tau_i+2^jL$}
            \State let $M_t=\big\{m:\text{there exists }(s,m)\in\mathcal{S}_{i,j}\text{ such that } s\le t\le s+\ell_m-1\big\}$.
            
            \If{$M_t\neq \varnothing$  or $t\in \mathcal{E}_{i,j}$}
                \State set $p_t= \operatorname{clip}_{[l+\rho,u-\rho]}\left\{p^*(z_t,\widehat\theta_{i,j-1})\right\}+\rho\eta_t$, where $\eta_t$ is i.i.d. Rademacher r.v's.
            \Else
                \State set $p_t=p^*(z_t,\widehat\theta_{i,j-1})$.
            \EndIf
            
            \State record $y_t$ and $x_t=[z_t^\top,-z_t^\top p_t]^\top$.
            
            \State \fbox{$\triangleright$ IV.\ change-point detection}
         \For{$(s,m)\in \mathcal{S}_{i,j}$ with $t=s+\ell_m-1$}
                \State let
                $\cJ_{i,j}(s,m)=[s,s+\ell_m-1]\cap\mathcal B_{i,j}$.
                \If{the change-point test LRT$(\widehat\theta_{i,j-1},
                \cJ_{i,j}(s,m))$ in \Cref{alg:lrt} fails}
                    \State set $t\leftarrow t+1$; restart and go to the next epoch (line 4).
                \EndIf
            \EndFor
            
            \If{$t=\tau_i+2^jL$ and LRT$(\widehat\theta_{i,j-1},
                \mathcal E_{i,j})$ in \Cref{alg:lrt} fails}
                    \State set $t\leftarrow t+1$; restart and go to the next epoch (line 4).
                \EndIf
            \State set $t\leftarrow t+1$.
        \EndWhile
    \EndFor
    \EndIndent
    \EndFor    
    \caption{The MCP-DP algorithm.}
    \label{alg:mcpdp}
\end{algorithmic}
\end{algorithm}

\begin{algorithm}[tbp]
\caption{The multiscale scheduling scheme. MSS$(\tau, j, L)$}
\label{alg:multiscale-sampling}
\begin{algorithmic}[1]
    \State \textbf{Input:} epoch start time $\tau$, block index $j$, base length $L$.
    \State Set $\mathcal S\leftarrow\varnothing$.
    \For{$k=0,1,\ldots,2^{j-1}-1$}
        \State Let $s_k=\tau+2^{j-1}L+1+kL$.
        \State Draw
        \[
        E_k\sim
        \operatorname{Bernoulli}\!\left(
        2^{-j/2}\sum_{m=0}^{j-1}2^{-m/2}
        \right).
        \]
        \If{$E_k=1$}
            \State Draw $m_k\in\{0,1,\ldots,j-1\}$ with
            $\mathbb P(m_k=r)\propto 2^{-r/2}$.
            \State Add $(s_k,m_k)$ to $\mathcal S$.
        \EndIf
    \EndFor
    \State Return $\mathcal S$. %
\end{algorithmic}
\end{algorithm}

\subsection{The Multiscale Change-Point Detection Subroutine}\label{subsec:mcpdp-detection}

In this subsection, we further unpack the two key elements of MCP-DP: the multiscale scheduling scheme in \Cref{alg:multiscale-sampling} and the likelihood ratio based change-point test in \Cref{alg:lrt}.

\begin{algorithm}[tbp]
\caption{The likelihood-ratio based change-point test. LRT$(\widehat\theta_{\mathrm{pre}},\cJ)$}
\label{alg:lrt}
\begin{algorithmic}[1]
     \State \textbf{Input:} detection threshold $c_{\gamma}$, reference parameter $\widehat{\theta}_{\rm pre}$, testing interval $\cJ$.
    \State Compute
    $\widehat\theta_{\cJ}=\argmin_{\theta\in\Theta}
    \mathcal{L}_{\cJ}(\theta)$.
    \If{$\mathcal{L}_{\cJ}(\widehat\theta_{\mathrm{pre}})
    -\mathcal{L}_{\cJ}(\widehat\theta_{\cJ})>c_{\gamma}d\log(dT)$}
        \State Return Fail.
    \EndIf
    \State Return Pass.
\end{algorithmic}
\end{algorithm}

\subsubsection{The multiscale scheduling scheme (\texorpdfstring{\Cref{alg:multiscale-sampling}}{Algorithm 2})}

The multiscale scheduling scheme~(MSS) is designed to bypass the challenge that MCP-DP does not know either the time or the magnitude of the non-stationarity in a given block $\mathcal{B}_{i,j}$. Instead of using a set of pre-fixed monitoring intervals on $\mathcal{B}_{i,j}$, MSS schedules a set of randomly sampled intervals with a wide range of {different} lengths (i.e.\ multiscale). 

The multiscale idea stems from the change-point estimation literature in statistics~\citep[e.g.][]{fryzlewicz2014wild, frick2014multiscale, yu2020review} and has been adopted in the non-stationary bandit literature~\citep[e.g.][]{auer19,chen2019new,wei2021non,suk2022tracking} in various forms. The key intuition is that a change of size $\delta$ in general requires a sample of size $\Omega(\delta^{-2})$ to be reliably detected~\citep{yu2020review}. Thus, one can use short windows to target large abrupt changes and long windows to target small smooth changes, hence multiscale.

Consider block $\mathcal B_{i,j}=[\tau_i+2^{j-1}L+1,\tau_i+2^jL]$. To generate the multiscale intervals, MSS first discretizes $\mathcal{B}_{i,j}$ into an $L$-spaced grid $\left\{s_k=\tau_i+2^{j-1}L+1+kL; ~k=0,1,\ldots,2^{j-1}-1\right\}$, which serves as the candidate set for the starting time of the multiscale interval. At each candidate $s_k$, MSS first draws
\begin{align}\label{eq:MSS_sampling_prob}
   E_k\sim
    \operatorname{Bernoulli}\!\left(
    2^{-j/2}\sum_{m=0}^{j-1}2^{-m/2}
    \right). 
\end{align}
If $E_k=0$, nothing is scheduled. If $E_k=1$, it then samples a scale $m_k\in \{0\}\cup [j-1]$ with probability $\mathbb P(m_k=r)\propto 2^{-r/2}$ for  $r=0,1,\ldots,j-1$, and adds $(s_k,m_k)$ to $\mathcal S_{i,j}$, which corresponds to the monitoring interval $\cJ_{i,j}(s_k,m_k)   =[s_k,s_k+\ell_{m_k}-1]\cap \mathcal B_{i,j}$. 

We remark that the sampling probabilities of $(s,m)$ in MSS are chosen so that shorter intervals (smaller $m$) are sampled more frequently whereas longer intervals (larger $m$) less, which creates sufficient monitoring coverage over $\mathcal{B}_{i,j}$ while keeping the number of price exploration rounds manageable. Indeed, we can show that, in high probability, the ratio between the size of the price exploration set $|\mathcal{J}_{i,j}\cup\mathcal{E}_{i,j}|$ and the block length $|\mathcal{B}_{i,j}|$ is of order $\sqrt{d/|\mathcal{B}_{i,j}|}$ (up to logarithmic terms), which is the optimal exploration-exploitation ratio for stationary contextual dynamic pricing~\citep{zhao2026contextual} and thus is important for the optimality of MCP-DP.


\begin{rem}\label{rem:compare_CPDP_MWDP}
An important reason why existing algorithms in the non-stationary dynamic pricing literature are non-adaptive is that they employ pre-fixed monitoring/estimation intervals designed to counter a specific type of change. For example, to handle unstructured changes with a known variation budget $V$, the MWDP algorithm in \cite{keskin2017chasing} schedules a fixed price exploration set every $(T/V)^{1/3}$ rounds; and to target abrupt changes (with additional conditions on minimum change size and stationary segment length), the CPDP algorithm in \cite{zhao2026high} schedules a fixed price exploration set every $\sqrt{T}$ rounds. In contrast, MCP-DP achieves adaptivity and removes the requirements on change size and segment length thanks to the use of the multiscale exploration intervals, which allows it to search non-stationarity over intervals of different lengths. We refer to the numerical experiments in \Cref{sec:experiments} for a detailed empirical comparison with CPDP and MWDP, which showcases the robustness and efficiency of MCP-DP.
\end{rem}

\subsubsection{The likelihood ratio based change-point test (\texorpdfstring{\Cref{alg:lrt}}{Algorithm 3})}
To avoid incurring large regret due to the non-stationarity of $\{\theta_t\}$, it is crucial to monitor the performance (i.e.\ its regret) of the reference model $\widehat\theta_{i,j-1}$, as it is used for price exploitation on $\mathcal{B}_{i,j}$. To this end, we design a novel likelihood ratio based change-point test~(LRT), which helps track the (unobserved) regret of $\widehat\theta_{i,j-1}$  on $\mathcal{B}_{i,j}$ when equipped with the scheduled multiscale intervals $\mathcal{S}_{i,j}$.

In particular, denote the reference model as $\widehat\theta_{\text{pre}}$, for any price exploration interval $\cJ$, we define the likelihood-ratio statistic
\begin{equation}\label{eq:lrt}
    \Lambda_{\cJ}(\widehat\theta_{\rm pre})
    =
    \mathcal{L}_{\cJ}(\widehat\theta_{\rm pre})
    -
    \mathcal{L}_{\cJ}(\widehat\theta_{\cJ}), \text{ with }
    \widehat\theta_{\cJ}=\arg\min_{\theta\in\Theta} \mathcal{L}_{\cJ}(\theta),
\end{equation}    
which measures the goodness of fit of $\widehat\theta_{\rm pre}$ relative to the best GLM estimator (i.e.\ $\widehat\theta_{\cJ}$) on the interval $\cJ$. Given a test threshold $\gamma=c_{\gamma}d\log(dT)$, where $c_\gamma>0$ is a constant tuning parameter, MCP-DP rejects the reference model and restarts whenever $\Lambda_{\cJ}(\widehat\theta_{\rm pre})>\gamma.$ Otherwise, the test passes and the current epoch continues. The newly designed LRT in \eqref{eq:lrt} is simple and computationally efficient as it anchors upon $\widehat\theta_{\rm pre}$ and requires only  computing the MLE $\widehat\theta_{\cJ}$ once per $\cJ$. In contrast, the classical LRT in the change-point literature~\citep[e.g.][]{yau2016inference, yu2020review} generally requires an additional scan over all time points in $\cJ$, which incurs substantially higher computational cost. 

Intuitively, the LRT in \eqref{eq:lrt} can help monitor potential changes in $\{\theta_t\}$: if the non-stationarity is mild, the reference model $\widehat\theta_{\rm pre}$ can still fit the observations in $\cJ$ sufficiently well so $\Lambda_{\cJ}(\widehat\theta_{\rm pre})$ is small, while if there is significant non-stationarity, $\widehat\theta_{\rm pre}$ will deviate from $\widehat\theta_{\cJ}$ and thus leads to a large $\Lambda_{\cJ}(\widehat\theta_{\rm pre})$. In the following, we formalize this intuition and, moreover, show that $\Lambda_{\cJ}(\widehat\theta_{\rm pre})$ in fact can be viewed as a surrogate for the (unobserved) exploitation regret $\widehat\theta_{\rm pre}$ incurred on a larger interval that $\cJ$ is embedded into.


We start by providing a high probability bound for the MLE $\widehat\theta_{\cJ}$ and the negative log-likelihood function $\mathcal{L}_{\cJ}(\theta)$, which can be of independent interest as it gives a general result that applies to non-stationary $\cJ$ with a mixture of GLMs. 

\begin{prop}\label{prop:prediction}
There exist absolute constants $c_1,c_2, m_{\psi 2}, M_{\psi 2}>0$ that only depend on Assumptions \ref{ass_context}-\ref{ass_boundedness}, such that with probability larger than $1-3/(dT^4)$, the following results hold for any price exploration interval $\cJ$ with $|\cJ|\ge 32c_z^2(\underline{\kappa}\lambda_z)^{-1}\log(dT)$,
\begin{itemize}
    \item[(i)]  for any $\theta^*\in\Theta$,  $
   m_{\psi2}\|\widehat\theta_\cJ-\theta^*\|^2_{H_{\cJ}}/2\leq  \mathcal{L}_{\cJ}(\theta^*)-  \mathcal{L}_{\cJ}(\widehat{\theta}_{\cJ})\leq  ({2m_{\psi2}})^{-1}\|\nabla \mathcal{L}_{\cJ}(\theta^*)\|^2_{H_{\cJ}^{-1}};$
\item[(ii)]  moreover, $\|\nabla \mathcal{L}_{\cJ}(\theta^*)\|_{H_{\cJ}^{-1}}^2
    \le
    c_1d\log(dT)+c_2|\cJ|V_{\cJ},$
where $\theta^*$ can be any $\theta_s$ with $s\in\cJ$.
\end{itemize}
Here, $H_{\cJ}=\sum_{t\in\cJ}x_tx_t^{\top}$ is the design matrix of $\mathcal{L}_{\cJ}(\theta)$.
\end{prop}

\Cref{prop:prediction}(i) and (ii) together provide a high probability bound on the \textit{prediction} error (i.e.\ weighted by the design matrix) of MLE for a mixture of GLMs, i.e.
\begin{align}\label{eq:prediction_error}
    \|\widehat\theta_{\cJ}-\theta_s\|_{H_{\cJ}}^2
    \lesssim
    d\log(dT)+|\cJ|V_{\cJ} \text{ for any } s \in \cJ.
\end{align}
The first term in \eqref{eq:prediction_error} characterizes the prediction error due to statistical variance, while the second term characterizes the prediction error due to bias caused by non-stationarity within $\{\theta_s, s\in \cJ\}$. In particular, under stationarity (i.e.\ $V_{\cJ}=0$), \eqref{eq:prediction_error} recovers the well-known optimal prediction error bound of MLE for GLM in the statistical literature. We remark that since the prediction error is weighted by the design matrix, its upper bound is therefore free of the minimum eigenvalue $\lambda_z$ of the context covariance matrix $\Sigma_z$.

By \Cref{lem:greedy-centered-information} and the classical matrix concentration result, one can show that $H_{\cJ}/|\cJ| \asymp \widetilde{\Sigma}_z=I_2\otimes \Sigma_z$ for any price exploration interval $\cJ$ with high probability, where recall  $H_{\cJ}=\sum_{t\in\cJ}x_tx_t^{\top}$ is the design matrix of $\mathcal{L}_{\cJ}(\theta)$. Combining this result with \Cref{prop:prediction}, we further establish \Cref{prop:mcpdp-lrt}, which gives an upper and lower bound of the LRT $\Lambda_{\cJ}(\widehat\theta_{i,j-1})$.

\begin{corollary}\label{prop:mcpdp-lrt}
  There exist absolute constants $c_{\Lambda1}, c_{\Lambda2}, C_{\Lambda1}, C_{\Lambda2}>0$ that only depend on Assumptions \ref{ass_context}-\ref{ass_boundedness}, such that with probability larger than $1-3/(dT^{4})$, it holds
\begin{align}
    \Lambda_{\cJ}(\widehat\theta_{i,j-1})
    &\le
    C_{\Lambda1}
    \sum_{t\in \cJ} \|\widehat\theta_{i,j-1}-\theta_t\|_{\widetilde{\Sigma}_z}^2
    +
    C_{\Lambda2} \left\{
   {d\log(dT)}+|\cJ|V_{\cJ} \right\},
    \label{eq:lrt-upper-calibration}\\
    \Lambda_{\cJ}(\widehat\theta_{i,j-1})
    &\ge
    c_{\Lambda1}
    \sum_{t\in \cJ} \|\widehat\theta_{i,j-1}-\theta_t\|_{\widetilde{\Sigma}_z}^2
    -c_{\Lambda2}\left\{{d\log(dT)}+|\cJ|  V_{\cJ}\right\}.
    \label{eq:lrt-lower-calibration}
\end{align}
\end{corollary}

\Cref{prop:mcpdp-lrt} indicates that the LRT $\Lambda_{\cJ}(\widehat\theta_{i,j-1})$ can be viewed as a surrogate of the design-adjusted prediction error $\sum_{t\in \cJ} \|\widehat\theta_{i,j-1}-\theta_t\|_{\widetilde{\Sigma}_z}^2$ of the reference model $\widehat\theta_{i,j-1}$ on any price exploration interval $\cJ$ in block $\mathcal{B}_{i,j}$, up to two terms due to statistical variance $d\log(dT)$ and non-stationarity related bias $|\cJ|V_{\cJ}.$ As will be seen, this serves as the foundation for using $\Lambda_{\cJ}(\widehat\theta_{i,j-1})$ to track the performance of $\widehat\theta_{i,j-1}$ on block $\mathcal{B}_{i,j}$. Below we unpack the intuition and high-level implication of \Cref{prop:mcpdp-lrt} and refer to the sketch of proof in \Cref{subsec:mcpdp-upper} for more formal arguments.

\textsc{[Control on restart]} First, the upper bound of LRT in \eqref{eq:lrt-upper-calibration} can be used to show that the total number of restarts of MCP-DP can be controlled by the statistical error and the intrinsic non-stationarity of the environment, as measured by the number of stationary segments $s_T$ and the variation budget $V_T$. This then helps bound the total regret of MCP-DP due to price exploration. We refer to Step~1 of the sketch of proof in \Cref{subsec:mcpdp-upper} for details.

\textsc{[Power against large exploitation regret]} Second, we further show that the lower bound of LRT in \eqref{eq:lrt-lower-calibration} implies that $\Lambda_{\cJ}(\widehat\theta_{i,j-1})$ can effectively track the exploitation regret of the reference model $\widehat\theta_{i,j-1}$ on any interval $\cI\subseteq \mathcal{B}_{i,j}$ that embeds $\mathcal{J}$ (i.e.\ $\mathcal{J}\subseteq \cI)$ and exhibits a low-level non-stationarity (i.e.\ $V_{\cI}\leq \sqrt{d/|\cI|}$). Importantly, we only need to focus on such $\cI$ as any block $\mathcal{B}_{i,j}$ can always be partitioned into such intervals (see Step 2 of the sketch of proof in \Cref{subsec:mcpdp-upper}). 

A key property of dynamic pricing is that the exploitation regret of $\widehat\theta_{i,j-1}$ on any interval $\cI\subseteq \mathcal{B}_{i,j}$ can be connected to its prediction error on $\cI$. More specifically, we have 
\begin{align}\label{eq:regret_to_prediction_error}
R(\cI)&:=\sum_{t\in\cI} r\big(p^*\left(z_t,\theta_t\right),z_t,\theta_t\big)-r\big(p^*(z_t,\widehat{\theta}_{i,j-1}),z_t,\theta_t\big)\nonumber\asymp \sum_{t\in\cI}  \left(p^*\left(z_t,\theta_t\right)-p^*(z_t,\widehat{\theta}_{i,j-1})\right)^2
\\
&\asymp \sum_{t\in\cI}(\theta_t-\widehat{\theta}_{i,j-1})^{\top} (I_2\otimes z_tz_t^{\top})(\theta_t-\widehat{\theta}_{i,j-1})\leq \sum_{t\in\cI}\|\theta_t-\widehat{\theta}_{i,j-1})\|^2_{\widetilde{\Sigma}_z}+c_{1*}\sqrt{|\cI|\log T},
\end{align}
where the first asymptotic equality follows from a Taylor expansion and the fact that $p^*(z_t,\theta_t)$ is the optimal price, the second asymptotic equality follows from the fact that the optimal price $p^*(z_t,\theta_t)$ is a Lipschitz function of $(z_t^\top \alpha_t, z_t^\top \beta_t)$ (see \Cref{subsec:regret_decompose_supplement} of the supplement), and the last inequality follows from Hoeffding's inequality and the fact that $\widehat\theta_{i,j-1}$ is independent of $z_t, t\in \cI.$


Combining \eqref{eq:lrt-lower-calibration} and \eqref{eq:regret_to_prediction_error}, with simple algebra and the fact $\mathcal{J}\subseteq \cI$, we can easily establish
\begin{align}\label{eq:SNR}
    \Lambda_{\cJ}(\widehat\theta_{i,j-1})\geq \underbrace{c_{\Lambda1}'|\cJ|\frac{R(\cI)}{|\cI|}}_{\text{signal due to regret}} - \underbrace{c_{\Lambda2}'\left(|\cJ|\sqrt{\frac{\log T}{|\cI|}} +|\cJ|V_{\cI}+d\log (dT)\right)}_{\text{noise due to statistical estimation and non-stationarity bias}},
\end{align}
for some absolute constant $c_{\Lambda1}', c_{\Lambda2}'>0$. Moreover, note that since $V_{\cI}\leq \sqrt{d/|\cI|}$, the noise term in \eqref{eq:SNR} can be well-controlled by $O(d\log(dT))$ given that $|\cJ|\leq \sqrt{d|\cI|\log (dT)}$. Therefore, the LRT $\Lambda_{\cJ}(\widehat\theta_{i,j-1})$ can be viewed as a surrogate for the (unobserved) regret $\widehat\theta_{i,j-1}$ incurred on $\cI$.

To achieve optimal regret (up to logarithmic terms), the LRT would need to, in high probability, flag intervals $\cI$ that admits excessive regret such that $R(\cI)\geq c_{2*}\sqrt{d|\cI|\log(dT)}$ for some $c_{2*}>0.$ Importantly, thanks to the multiscale nature of the price exploration mechanism of MCP-DP, we can show that for any such interval $\cI$, there always exist a multiscale level $m\in \{0,1,\cdots,j-1\}$ such that a scale-$m$ price exploration interval $\cJ\subseteq \cI$ is both sufficiently long for the regret signal in \eqref{eq:SNR} to exceed the LRT test threshold and sufficiently short to keep the noise in \eqref{eq:SNR} under control. In other words, if the reference model $\widehat\theta_{i,j-1}$ admits excessive exploitation regret on $\cI$, it will be flagged by the LRT $\Lambda_{\cJ}(\widehat\theta_{i,j-1})$ once a scale-$m$ price exploration interval $\cJ$ is scheduled within $\cI$. This is the foundation for the success of MCP-DP and we refer to Step 3 of the sketch of proof in \Cref{subsec:mcpdp-upper} for more formal arguments. We conclude this section with two remarks. 


\begin{rem}[ETC structure and MSS]
The success of LRT lies both on the ETC structure and the multiscale sampling scheme~(MSS) of MCP-DP. First, the ETC structure ensures that the greedy pricing policy remains unchanged over $\mathcal{B}_{i,j}$ (i.e.\ based on a fixed $\widehat{\theta}_{i,j-1}$), and thus ensures that $\cJ$ and $\cI$ with $\cJ\subseteq\cI$ share the same reference model $\widehat\theta_{i,j-1}$, which is the key for $\Lambda_{\cJ}(\widehat{\theta}_{i,j-1})$ to approximate the exploitation regret $R(\cI)$. Second, the sampling probability of MSS in \eqref{eq:MSS_sampling_prob} is designed to ensure that in high probability, a price exploration interval $\cJ$ with a suitable scale will eventually be scheduled inside an interval $\cI$ with excessive regret and thus restart MCP-DP, without letting too many such intervals slip through. 
\end{rem}

\begin{rem}[Comparison with the adaptive bandit literature]\label{rem:comparison_RL}
As discussed before, existing works in the adaptive non-stationary bandit literature~\citep[e.g.][]{auer19,chen2019new,wei2021non,suk2022tracking} cannot be used for contextual dynamic pricing. In particular, these methods track the changes of the reward achieved by the best-arm in a fixed action set (be it MAB or linear contextual bandit). However, this strategy does not apply to contextual dynamic pricing, as both the best arm (i.e.\ the optimal price) and its reward (i.e.\ the optimal revenue) change as the context $z_t$ varies, even on a stationary segment. We therefore propose a newly designed LRT to continuously track the regret of our pricing policy: we show that the LRT can be connected to the prediction error of GLM and thus can be used to track the regret of the greedy-pricing policy based on a reference model $\widehat\theta_{i,j-1}$, which serves as the foundation for the proposed MCP-DP algorithm.
\end{rem}



\subsection{Upper Bound Results}\label{subsec:mcpdp-upper}

We now provide the formal regret guarantee for MCP-DP. Recall we denote $s_T$ as the total number of stationary segments and $V_T$ as the design-adjusted variation budget. 
\begin{thm}
\label{thm:mcpdp-upper}
Suppose Assumptions \ref{ass_context}-\ref{ass_boundedness} hold.  Denote $\pi$ as the MCP-DP algorithm with tuning parameters 
$\rho\in(0,(u-l)/2)$,  $L=c_Ld\log(dT)$, $\gamma=c_{\gamma}d\log(dT)$, for some sufficiently large absolute constants $c_L,c_{\gamma}>0$. 
It holds that for all sufficiently large $T$, with probability larger than $1-9T^{-1}$,
\[
    R_T^{\pi}
    \le
    C_{\pi}\log^{3/2}(dT)
        \left(\sqrt{ds_TT}
        \wedge
        \left\{
        \sqrt{dT}
        +
        d^{1/3}V_T^{1/3}T^{2/3}
        \right\}
    \right),
\]
where $C_{\pi}>0$ is an absolute constant.
\end{thm}
\Cref{thm:mcpdp-upper} matches the lower bound in \Cref{thm:lb} up to logarithmic factors and thus shows that MCP-DP achieves minimax optimality in all key quantities including the horizon $T,$ the context dimensionality $d$, the number of stationary segments $s_T$, and the variation budget $V_T.$

We remark that the upper bound in \Cref{thm:mcpdp-upper} does not rely on exact recovery of all change-points. This is important since under smooth drift or frequent small changes, identifying every change-point is either ill-posed or statistically impossible. Instead, MCP-DP only restarts when the reference model learned from past exploration significantly deviates from the current environment and causes substantial inflation of the exploitation regret. Changes that are too short or too small to be detected are instead absorbed into non-stationarity measured by $s_T$ or $V_T.$
This differs from change-point localization based pricing policies such as CPDP \citep{zhao2026high} and the CU policy of \cite{chen2019dynamic}, which require additional conditions on the change size and minimum spacing between change-points as their success crucially hinges on accurate identification of \textit{all} change-points.



\noindent\textbf{A Sketch of Proof.}
We first collect the high probability events used throughout the proof. By \Cref{lem:greedy-centered-information} and a  Chernoff type concentration bound for random matrices in \Cref{lem:mcpdp-tech-hessian} of the supplement, the design matrices of all price exploration intervals are well-conditioned. Combined with \Cref{prop:prediction}, this gives uniform prediction error bounds for all the MLEs computed in MCP-DP. On the same event,  \Cref{prop:mcpdp-lrt} provides upper and lower bounds for all the LRTs
used in MCP-DP.  A union bound gives the high probability in \Cref{thm:mcpdp-upper}.


\textsc{[Regret decomposition.]} For each block $\mathcal{B}_{i,j}$, we decompose it into $\mathcal{B}_{i,j}=\cJ_{i,j}\cup \mathcal{E}_{i,j}\cup\mathcal{G}_{i,j}$, where $\cJ_{i,j}=\bigcup_{(s,m)\in\mathcal{S}_{i,j}}\cJ_{i,j}(s,m)$ is the set of price exploration rounds scheduled via the multiscale sampling scheme~(MSS), $\mathcal{E}_{i,j}$ is the mandatory block-end
exploration set, and  $\mathcal G_{i,j}:=\mathcal B_{i,j}\setminus(\cJ_{i,j}\cup\mathcal E_{i,j})$ is the set of price exploitation rounds.  Denote the total number of epoch as $E$ and denote the total number of blocks in the $i$th epoch $\cI_i=[\tau_i+1,\tau_{i+1}]$ as $B_i$, where we define $\tau_1:=0$ and $\tau_{E+1}:=T.$ We can decompose the regret as   
\begin{align}\label{eq:regret_decompose_context_LRT}
    R_T&=\sum_{t=1}^T r(p^*_t;z_t, \theta_t)-r(p_t; z_t,\theta_t)\notag\\&=\left(\sum_{i=1}^E\sum_{t\in \mathcal{B}_{i,0}}+\sum_{i=1}^E\sum_{j=1}^{B_i}\sum_{t\in \cJ_{i,j}\cup\mathcal{E}_{i,j}}+\sum_{i=1}^E\sum_{j=1}^{B_i}\sum_{t\in \mathcal{G}_{i,j}}\right)\left[r(p^*_t; z_t, \theta_t)-r( p_t;z_t,  \theta_t)\right]
    \\&\notag=: R_{T,1}+R_{T,2}+R_{T,3},
\end{align}
where $R_{T,i},i=1,2,3$ correspond to regret generated by: (1) the initial exploration rounds after each restart, (2) the MSS-scheduled and block-end exploration rounds used for multiscale change-point testing and reference model estimation, and  (3) the
exploitation rounds, respectively. The first two terms can be controlled by upper bounding the number of epochs and thus the number of total exploration rounds. The main technical challenge is to control the exploitation regret $R_{T,3}$ via $s_T$ and $V_T$ without requiring the identification of all change-points.

\noindent\textsc{[Step 1: Bound total number of epochs and exploration regret.]}


Based on the upper bound of LRT in \Cref{prop:mcpdp-lrt}, \Cref{lem:mcpdp-restart-trigger} of the supplement establishes that
an epoch starting at $\tau_i+1$ cannot restart by time $t$ whenever
\begin{equation}\label{eq:upper-sketch-no-restart-radius}
    V_{[\tau_i+1,t]}
    \le
    \sqrt{\frac{d}{t-\tau_i}}.
\end{equation}
We now use \eqref{eq:upper-sketch-no-restart-radius} to bound the number of epochs $E$. First, if no change occurs during an epoch (i.e.\ stationary), the epoch cannot end with a restart. Hence every restart must be associated with at least one change-point, which gives $E\leq s_T$. Second, denote $\cI_1,\ldots,\cI_{E-1}$ as the completed epochs that end with a restart (i.e.\ $\cI_i=[\tau_{i}+1,\tau_{i+1}]$).  By the contrapositive of \eqref{eq:upper-sketch-no-restart-radius}, each $\cI_i$ must satisfy $    V_{\cI_i}>
    \sqrt{{d}/{|\cI_i|}}$.
Therefore, by H\"older's inequality,
\begin{align*}
    E-1
    &=
    \sum_{i=1}^{E-1}
    \left(\frac{|\cI_i|}{d}\right)^{1/3}
    \left(\frac{|\cI_i|}{d}\right)^{-1/3}\le
    d^{-1/3}
    \left(\sum_{i=1}^{E-1}|\cI_i|\right)^{1/3}
    \left(\sum_{i=1}^{E-1}\sqrt{\frac{d}{|\cI_i|}}\right)^{2/3}
    \notag\\
    &\le
    d^{-1/3}T^{1/3}
    \left(\sum_{i=1}^{E-1}V_{\cI_i}\right)^{2/3}
    \le
    d^{-1/3}T^{1/3}V_T^{2/3}.
\end{align*}
Hence, we have in \Cref{cor:mcpdp-number-epoch} of the supplement that
\begin{equation}\label{eq:upper-sketch-epoch-variation}
E\leq s_T\wedge (1+d^{-1/3}T^{1/3}V_T^{2/3}).
\end{equation}

Since each epoch begins with an initial exploration of $L$ rounds and $EL\leq T$, we have that
$$R_{T,1}\lesssim EL= O(\sqrt{ETL}).$$
For the multiscale exploration intervals, inside a block $\mathcal{B}_{i,j}$ of length $2^{j-1}L$, a scale-$m$ interval of length $\ell_m=\sqrt{d\log(dT)\,2^mL}$ is scheduled at a given $L$-spaced grid point with probability $q_{j,m}=2^{-(m+j)/2}$. Therefore, by Bernstein's inequality, we have with high probability, it holds that \begin{equation}\label{bound_explore}
    N_{i,j}\leq \widetilde{O}\left(\sum_{k=0}^{2^{j-1}-1}\sum_{m=0}^{j-1}q_{j,m}\times \sqrt{d\log(dT)2^m L}\right)=\widetilde{O}\left(\sqrt{d2^j L\log(dT)}\right),
\end{equation}
where $N_{i,j}$ is the total number of multiscale price exploration rounds in $\mathcal{B}_{i,j}$. Note that the mandatory block-end exploration set is also of length $\sqrt{d2^j L\log(dT)}$. By \eqref{bound_explore} and the fact that $\sum_{j=1}^{B_i}2^jL\leq 2|\cI_i|$ and $B_i\leq \log_2(T)$, we can show that 
\begin{align*}
 R_{T,2}\leq &\widetilde{O}\left(\sum_{i=1}^E\sum_{j=1}^{B_i}\sqrt{d2^j L\log(dT)}\right)\leq \widetilde{O}\left(\sqrt{d\log(dT)}\sum_{i=1}^E\sqrt{B_i\sum_{j=1}^{B_i}2^j L}\right)\\\leq& 
\widetilde{O}\left(\sqrt{d\log(dT)}\sum_{i=1}^E\sqrt{2|\cI_i|}\right)\leq\widetilde{O}(\sqrt{ETL}),  
\end{align*}
where the second and the last inequality holds by the H\"older's inequality and that $L\asymp d\log(dT).$

Hence, with the bound in \eqref{eq:upper-sketch-epoch-variation}, we obtain that
\[
    R_{T,1}+R_{T,2}
    \le
    \widetilde O(\sqrt{dTE})
    \le
    \widetilde O\!\left(
    \sqrt{ds_TT}
    \wedge
    \{\sqrt{dT}+d^{1/3}V_T^{1/3}T^{2/3}\}
    \right).
\]

\noindent\textsc{[Step 2: connect exploitation regret with prediction error.]}

By \eqref{eq:regret_to_prediction_error}, we can convert the exploitation regret on a block $\mathcal{B}_{i,j}$ to the design-adjusted prediction error of $\widehat\theta_{i,j-1}$. Therefore, to bound the total exploitation regret $R_{T,3}$, it suffices for us to analyze and bound  $\sum_{t\in\mathcal{B}_{i,j}}\|\theta_t-\widehat{\theta}_{i,j-1}\|_{\widetilde{\Sigma}_z}^2.$ (We remark that the full proof of \Cref{thm:mcpdp-upper} actually analyzes the realized block $\cI^{*}:=\mathcal B_{i,j}\cap[\tau_{i}+1,\tau_{i+1}]$, as a scheduled block $\mathcal{B}_{i,j}$ can be cut short due to the restart at $\tau_{i+1}$. However, here we use $\mathcal{B}_{i,j}$ for simplicity and intuition, albeit with less rigor.)



\Cref{lem:mcpdp-partition} of the supplement shows that there always exists a partition
$\mathcal{B}_{i,j}=\cI_1'\cup\cdots\cup\cI'_\Gamma$ such that $V_{\cI_k'}\le \sqrt{{d}/{|\cI_k'|}}$ (i.e.\ low non-stationarity) for $k=1,\cdots,\Gamma$ and $\Gamma \lesssim s_{\mathcal{B}_{i,j}} \wedge \left\{1+d^{-1/3}|\mathcal{B}_{i,j}|^{1/3}V_{\mathcal{B}_{i,j}}^{2/3}\right\}$. Thus, every partition element has low non-stationarity. Here, same as the result in Step 1,  
the first bound on $\Gamma$ counts the number of stationary segments within $\mathcal{B}_{i,j}$ and the second bound is related to the variation budget within $\mathcal{B}_{i,j}$.  

For notational simplicity, denote $\cI'=[s,e]$ as a generic interval $\cI'_k, k\in [\Gamma]$ in the partition of $\mathcal{B}_{i,j}$. Define $\epsilon_{\cI'}=\|\theta_s-\widehat\theta_{i,j-1}\|_{\widetilde\Sigma_z}^2.$ Given a threshold $D_0=c_D \sqrt{\log(dT)}$ for some constant $c_D>0$, it is easy to see that the prediction error of the reference model $\widehat\theta_{i,j-1}$ on $\mathcal{I}'$ can be decomposed as
\begin{align}
 \notag \sum_{t\in \cI'} \|\theta_t-\widehat\theta_{i,j-1}\|^2_{\widetilde{\Sigma}_z}&\leq 2\sum_{t\in \cI'} \|\theta_t-\theta_s\|^2_{\widetilde{\Sigma}_z} + 2\sum_{t\in \cI'} \|\theta_s-\widehat\theta_{i,j-1}\|^2_{\widetilde{\Sigma}_z} \\&\leq 2|\cI'|V_{\cI'} + 2\sqrt{d|\cI'|}D_0  + 2|\cI'|\epsilon_{\cI'} \mathbb I\{\epsilon_{\cI'}>D_0\sqrt{d/|\cI'|}\}.\label{eq:upper-sketch-local-split}
\end{align}
This is the result of Lemma \ref{lem:mcpdp-bound-on-I} in the supplement.

By construction, $V_{\cI'}\leq \sqrt{d/|\cI'|}$. Therefore, the first two terms in \eqref{eq:upper-sketch-local-split} contribute  $\widetilde{O}(\sqrt{d|\cI'|})$ to the regret. Combined with the bound of $\Gamma$, we have that 
\begin{equation}\label{bound_step2}
\sum_{k=1}^{\Gamma}\widetilde{O}(\sqrt{d|\cI_k'|})\leq \widetilde{O}\left(\sqrt{d\Gamma \sum_{k=1}^{\Gamma}|\cI_k'|}\right)=\widetilde{O}\left(\sqrt{ds_{\mathcal{B}_{i,j}}|\mathcal{B}_{i,j}|}\wedge \{\sqrt{d|\mathcal{B}_{i,j}|}+d^{1/3}|\mathcal{B}_{i,j}|^{2/3}V_{\mathcal{B}_{i,j}}^{1/3}\}\right).
\end{equation}
It remains to control the last term in \eqref{eq:upper-sketch-local-split}, which corresponds to exploitation regret on intervals where the reference model incurs large prediction error (i.e.\ $\epsilon_{\cI'} > D_0\sqrt{d/|\cI'|}$).

\noindent\textsc{[Step 3: Bound large prediction error via multiscale change-point testing.]}

Suppose there is large prediction error for some $\cI'$ in the partition of $\mathcal{B}_{i,j}$ with $V_{\cI'}\leq \sqrt{d/|\cI'|}$. In other words, the last term in \eqref{eq:upper-sketch-local-split} is active, i.e.\ $\epsilon_{\cI'}>D_0\sqrt{d/|\cI'|}$. We now show that the lower bound of LRT in \eqref{eq:lrt-lower-calibration} indicates that this excessive prediction error is detectable as long as MSS schedules an exploration interval of a suitable scale within $\cI'$.

In particular, a key result is \Cref{lem:detection_interval_context_LRT} in the supplement, which shows that, for all threshold $D_0=c_D\sqrt{\log(dT)}$ with a sufficiently large absolute constant $c_D>0$, there exists a scale $m\in\{0,1,\ldots,j-2\}$, whose exact value depends on $\epsilon_{\cI'}$, such that the following results hold,
\begin{align}\label{eq:key_lemma_D5_statement_sketch_of_proof}
D_0\sqrt{\frac{d}{2^{m+1}L}} \leq \epsilon_{\cI'} \leq D_0\sqrt{\frac{d}{2^mL}}
\quad \text{and} \quad 2^mL \leq |\cI'|.
\end{align}

Now, consider a scale-$m$ exploration interval ${\cJ}\subseteq \mathcal{I}'$. Such interval exists as by \eqref{eq:key_lemma_D5_statement_sketch_of_proof}, $|{\cJ}|=\ell_m=\sqrt{d2^mL\log(dT)}\leq
\sqrt{d|{\cI'}|\log(dT)}< |\cI'|$ for all grid size $L=c_Ld\log(dT)$ with $c_L>1.$ Moreover, this bound on $|\cJ|$ also implies that the noise term due to statistical variance and non-stationarity related bias in \eqref{eq:lrt-lower-calibration} can be upper bounded by $c_{3*}d\log(dT)$ for some absolute constant $c_{3*}>0$. At the same time, the lower bound on $\epsilon_{{\cI'}}$ in \eqref{eq:key_lemma_D5_statement_sketch_of_proof} ensures that, for all sufficiently large $c_D>0$, the signal in \eqref{eq:lrt-lower-calibration} can be lower bounded by $c_{4*}d\log(dT)$ where $c_{4*}>0$ is an absolute constant satisfying $c_{4*}-c_{3*}>c_\gamma$. Thus, it follows that
\[
\Lambda_{{\cJ}}(\widehat\theta_{i,j-1})
\geq(c_{4*}-c_{3*})d\log(dT)>\gamma.
\]
In other words, every scale-$m$ exploration interval $\mathcal{J}$ scheduled within $\cI'$ will trigger a restart.

Therefore, if MCP-DP proceeds onto an interval ${\cI_k'}=[s_k,e_k]$ with large prediction error $\epsilon_{\cI_k'}>D_0\sqrt{d/|\cI_k'|}$, it means that the Bernoulli draw of MSS must have missed every eligible scale-$m_k$ detection opportunity on $\cI'_k$, and all such opportunities on all $\cI'_i$ with $\epsilon_{\cI_i'}>D_0\sqrt{d/|\cI_i'|}$ and $i<k$. 

Now, define $R_{i,j}^{\text{miss}}=\sum_{k=1}^\Gamma 2|\cI'_k|\epsilon_{\cI'_k} \mathbb I\{\epsilon_{\cI'_k}>D_0\sqrt{d/|\cI'_k|}\} \mathbb I(e_k\leq \tau_{i+1})$, which denotes the sum of the last term in \eqref{eq:upper-sketch-local-split} over all $\cI'_k$ with $\epsilon_{\cI'_k}>D_0\sqrt{d/|\cI'_k|}$ across $\mathcal{B}_{i,j}$ before a restart is triggered. To bound $R_{i,j}^{\text{miss}}$, we essentially need to bound the regret accumulated on each $\cI'_k$ before a scale-$m_k$ exploration interval is scheduled on it (which then leads to a restart). In particular, since MSS samples every $L$ rounds, a missed scale-$m_k$ opportunity incurs at most $c_{5*} L\epsilon_{\cI'_k}$ regret where $c_{5*}$ is some absolute constant. By \eqref{eq:key_lemma_D5_statement_sketch_of_proof}, it holds that $c_{5*} L\epsilon_{\cI'_k} \asymp LD_0\sqrt{d/(2^{m_k}L)}$. Moreover, at each $L$-spaced grid, the scale-$m_k$ exploration interval is sampled with probability $2^{-(j+m_k)/2}$. Importantly, note that their ratio satisfies
\[
    \frac{LD_0\sqrt{d/(2^{m_k}L)}}{2^{-(j+m_k)/2}}
    =D_0\sqrt{d2^jL},
\]
which is independent of $m_k$ and thus independent of $\cI'_k$. Without going into technical details, this implies that the exploitation regret incurred on different $\cI'_k$ in $R_{i,j}^{\text{miss}}$ can be analyzed jointly. 

In particular, via an argument based on geometric distribution, \Cref{lem:mcpdp-bernoulli-miss} in the supplement shows that  
\[
    \mathbb P\!\left(
   R^{\rm miss}_{i,j}>
    4D_0\sqrt{d2^jL}\bigl(1+3\log(T)\bigr)
    \right)
    \leq T^{-3}.
\]
Thus $R^{\rm miss}_{i,j}\le\widetilde O(\sqrt{d2^jL})$ uniformly over all blocks with a union bound argument. Combining this with  \eqref{eq:upper-sketch-local-split} and \eqref{bound_step2} yields that, for every block $\mathcal{B}_{i,j}$ 
\[
    \sum_{t\in \mathcal{B}_{i,j}}
    \|\theta_t-\widehat\theta_{i,j-1}\|_{\widetilde\Sigma_z}^2
    \le
    \widetilde O\!\left(
    \left[
    \sqrt{ds_{\mathcal{B}_{i,j}}|\mathcal{B}_{i,j}|}
    \wedge
  \{ \sqrt{d|\mathcal{B}_{i,j}|}+ d^{1/3}|\mathcal{B}_{i,j}|^{2/3}V_{\mathcal{B}_{i,j}}^{1/3}\}
    \right]
    +
    \sqrt{d2^jL}
    \right).
\]

Finally, after summing over all dyadic blocks, and applying the upper bound of epoch number in \eqref{eq:upper-sketch-epoch-variation}, we can obtain 
\[
    R_{T,3}
    \le
    \widetilde O\!\left(
    \sqrt{ds_TT}
    \wedge
    \{\sqrt{dT}+d^{1/3}V_T^{1/3}T^{2/3}\}
    \right).
\]
This completes the proof.

\section{Numerical Experiments}\label{sec:experiments}
In this section, we conduct extensive numerical experiments to investigate the performance of the proposed MCP-DP algorithm for GLM-based dynamic pricing under non-stationarity. \Cref{subsec:generalsetting} discusses the general simulation settings, followed by numerical results in Section~\ref{subsec:num_synthetic}.

\subsection{General Simulation Settings}\label{subsec:generalsetting}

\textbf{Choice of tuning parameters $(c_L, c_\gamma)$.}  There are two key tuning parameters of MCP-DP: (i).\ the block base length $L=c_Ld\log (Td)$, and (ii).\ the change-point detection threshold $\gamma=c_{\gamma} d\log (Td).$ Note that Theorem \ref{thm:mcpdp-upper} holds for all sufficiently large $c_{L}$ and $c_{\gamma}$. In practice, we find that the performance of MCP-DP is robust to the choice of $c_L$ and thus fix $c_L=2$. We recommend setting $\gamma=\left\lceil d\{\log(Td)\}^{1.1}\right\rceil$, which avoids the tuning of $c_{\gamma}$. For any sufficiently large $Td$, this provides valid theoretical results and only inflates the regret by a factor of $\{\log(Td)\}^{0.1}$.

\medskip
\noindent \textbf{Benchmark methods.} We consider two {non-adaptive} methods for handling non-stationary dynamic pricing and use them as benchmarks to showcase the adaptivity and robustness of MCP-DP.
\begin{itemize}
    \item CPDP (change-point DP): CPDP is proposed in \cite{zhao2026high} for contextual dynamic pricing under structured changes. CPDP partitions the horizon $[T]$ into pre-fixed alternating cycles of price experimentation (of length $m$) and price exploitation (of length $n$). At the end of each experimentation cycle, it runs change-point detection on the accumulated price experiments so far to determine whether to continue and conduct greedy pricing in the upcoming exploitation cycle or to restart the algorithm, depending on if a change is detected. With a proper tuning of $m, n$ and a change detection threshold $\gamma$, CPDP can achieve a regret of order $\widetilde O(\sqrt{s_TdT})$ under abrupt changes {without} the knowledge of $s_T.$ In particular, CPDP can achieve the optimal $\widetilde O(\sqrt{dT})$ rate under the stationary setting.

    \item MWDP (moving-window DP): MWDP is proposed in \cite{keskin2017chasing} for covariate-free dynamic pricing under unstructured changes. MWDP also partitions the horizon $[T]$ into alternating cycles of price experimentation (of length $m$) and price exploitation (of length $n$). However, unlike CPDP, it adopts a moving-window strategy to handle non-stationarity, where for each upcoming exploitation cycle, it only uses the most recent $k$ cycles of price experiments for model estimation and greedy pricing. With a proper tuning of $m,n, k$, MWDP can be extended to non-stationary contextual dynamic pricing and achieve a regret of order $\widetilde O(d^{1/3}V_T^{1/3}T^{2/3})$ for a \textit{known} variation budget $V_T$. Note that by design, MWDP will always incur a regret of order $\widetilde O(T^{2/3})$ and is thus not optimal under stationarity.
\end{itemize}
For more details of CPDP and MWDP, we refer to Section \ref{sec:implementation_details} of the supplement.

\medskip
\noindent\textbf{Data generating process.} We consider both a linear demand model and a logistic demand model, where given $(z_t,p_t)$, the consumer demand $y_t$ is a Gaussian or Bernoulli random variable such that
\begin{align*}
    \mathbb E(y_t|z_t,p_t)=\psi'\left(\alpha_t^\top z_t-(\beta_t^\top z_t)\cdot p_t\right),
\end{align*}
where $\psi'(\cdot)$ is the identity or logistic function, respectively. The feature vectors $\{z_t=(z_{t,1},\cdots,z_{t,d})^{\top}\}$ are i.i.d.\ random vectors across time. We consider two simulation scenarios: 
\begin{itemize}
    \item \textbf{(Z1)}: $z_t=(1,z_{t,2},\cdots,z_{t,d})$ where $\{z_{t,i}\}_{i=2}^d$ are i.i.d.\ uniform$(0, 2/\sqrt{d-1})$ random variables;
    \item \textbf{(Z2)}: $z_t$ is uniformly drawn from the $d$ standard orthonormal basis $\{e_1,\ldots, e_d\}$ in $\mathbb R^d.$
\end{itemize}
Note that the behavior of the design matrix $\Sigma_z$ is significantly different under (Z1) and (Z2), where $\lambda_{\max}(\Sigma_z)\asymp 1$ and $\lambda_{\min}(\Sigma_z)\asymp 1/d$ under (Z1), while $\lambda_{\max}(\Sigma_z)=\lambda_{\min}(\Sigma_z) =1/d $ under (Z2). We defer the design of the non-stationarity of $\{\theta_t=(\alpha_t^\top,\beta_t^\top)^\top\}$ to the next subsection.
\medskip

\noindent\textbf{Implementation details.} We set the price range to be $[l,u]=[0.1,10]$ across all settings and methods. For the linear demand, we set $\ell_j=\sqrt{1/2\cdot d\log(dT)\times 2^jL}$ for MCP-DP. Following \cite{zhao2026high}, we set the initial exploration-exploitation tuning to be $m=2d(\log(dT))^{1.1}$ and $n=\sqrt{dT}$, and the change-point detection threshold $\gamma=d(\log(dT))^{1.1}$ for CPDP (see Section \ref{sec:implementation_details} of the supplement for more details). Since $V_T$ is unknown in practice, we use the fixed tuning proxy $V_{\rm tune}=1$ and set $m=2d(\log(dT))^{1.1}$, $n=d^{2/3}T^{1/3}$ and $k=d^{-1/3}T^{1/3}$ for MWDP, which is optimal when  the true variation budget is constant and comparable to $V_{\rm tune}$. For a fair comparison, we set 
$p_t=\text{clip}_{[l+\rho,u-\rho]}(\widehat p_t)+\rho\eta_t$ with $\rho=1$ for all three methods during price experimentation, where $\widehat p_t$ is the estimated optimal price at $t$ and $\eta_t$ is i.i.d.\ Rademacher random variable. For logistic demand, which is more challenging in terms of model estimation, we increase the price exploration length by a factor of $\sqrt{2}$ for all three methods, i.e.\ we set $\ell_j=\sqrt{d\log(dT)\times 2^jL}$ for MCP-DP, set $m=2\sqrt{2}d(\log(dT))^{1.1}$ for CPDP and MWDP. 

\subsection{Experiment Results}\label{subsec:num_synthetic}
We organize the numerical experiments into two schemes of non-stationarity. The first scheme considers baseline settings, where the non-stationarity is moderate and the non-adaptive benchmark methods are expected to remain competitive. The second scheme considers more complex settings, where the environment exhibits richer forms of non-stationarity that are more challenging to handle. We focus on the numerical results under simulation scenario {(Z1)} and provide remarks for its comparison with the results under {(Z2)}, which is detailed in the supplement.

\subsubsection{Baseline settings}\label{subsec:baseline_exp}
We set $\theta^{(1)}=(\alpha^{(1)\top}, \beta^{(1)\top})^\top$ with $\alpha^{(1)}=(0, 2/\sqrt{d-1}\cdot \textbf{1}_{d-1}^\top)^\top$ and $\beta^{(1)}=(0, 1/\sqrt{d-1}\cdot \textbf{1}_{d-1}^\top)^\top$, and set $\theta^{(2)}=(\alpha^{(2)\top}, \beta^{(2)\top})^\top$ with $\alpha^{(2)}=(0, 4/\sqrt{d-1}\cdot \textbf{1}_{d-1}^\top)^\top$ and $\beta^{(2)}=(0, 0.75/\sqrt{d-1}\cdot \textbf{1}_{d-1}^\top)^\top$. Note that $\theta^{(1)}$ denotes a low-demand period with low attractiveness $\mathbb E(\alpha_t^\top z_t)=2$ and high price sensitivity $\mathbb E(\beta_t^\top z_t)=1$, while $\theta^{(2)}$ denotes a high-demand season with high attractiveness $\mathbb E(\alpha_t^\top z_t)=4$ and low sensitivity $\mathbb E(\beta_t^\top z_t)=0.75$. We consider three non-stationarity settings:
\begin{itemize}
    \item \textbf{Op1} (no change): $\theta_t\equiv\theta^{(1)};$
    \item \textbf{Op2} (abrupt change):
    \[
    \theta_t=\begin{cases} 
		\theta^{(1)}, & t\in [1,T/4]\cup (T/2,3T/4],\\
		\theta^{(2)}, & t\in (T/4,T/2]\cup (3T/4,T];\\
	\end{cases}
    \]
    \item \textbf{Op3} (smooth change):
    \[
    \theta_t= (1+\cos(10\pi t/T))/2\cdot \theta^{(1)}
    + (1-\cos(10\pi t/T))/2 \cdot \theta^{(2)}.
    \]
\end{itemize}

Note that Op1 is the stationary case, while Op2 experiences three abrupt changes and is the ideal setting for CPDP, and Op3 undergoes smooth changes with a constant variation budget (i.e.\ $V_T\asymp 1$) and is the ideal setting for MWDP. Therefore, we run all three methods for Op1, MCP-DP and CPDP for Op2, and MCP-DP and MWDP for Op3. For more intuition, \Cref{fig:Theta_Path} of the supplement plots the sequence of true model parameters $\{\theta_t\}_{t=1}^T$ for Op1-Op3.

We set $T \in \{(1, 2,\ldots, 9,10,15,20,\cdots,40)\times 10^4\}$ and vary $d\in \{2,4,6,10,16\}$. Given $(T,d)$, we conduct 200 experiments for each algorithm and record the realized regrets $\{R_{T,d}^{(i)}\}_{i=1}^{200}$. \Cref{fig:linear_regular_case_z1} summarizes the performance for each algorithm under the linear demand model, where we report the sample mean regret $\overline{R}_{T,d}=\sum_{i=1}^{200}R_{T,d}^{(i)}/200$ and the 99\% confidence interval \ $I_{T,d}=[\overline{R}_{T,d}-3S_{T,d}/\sqrt{200}, ~\overline{R}_{T,d}+3S_{T,d}/\sqrt{200}]$, where $S_{T,d}$ is the sample standard deviation of $\{R_{T,d}^{(i)}\}_{i=1}^{200}$.

Several comments are in order. First, for all three settings, MCP-DP gives sublinear regret in $T$ while the regret also increases with $d$. In addition, its regret increases from stationarity (i.e.\ Op1) to structured changes (i.e.\ Op2) to unstructured changes (i.e.\ Op3), which agrees with Theorem \ref{thm:mcpdp-upper}. Second and moreover, MCP-DP provides performance on par with the best non-adaptive algorithm specifically designed for each non-stationary setting, i.e.\ CPDP under Op2 and MWDP under Op3, which showcases its adaptivity and efficiency. Third, under the stationary setting Op1, MCP-DP and CPDP provide similar and much better performance than MWDP, as both methods can achieve optimality under stationarity while MWDP cannot. Fourth, interestingly, the mean regret path of CPDP under Op2 is not always monotone. This is due to that CPDP uses pre-fixed price experimentation cycle, which can impact detection delay depending on the relative location of the experimentation cycles to the true change-points. In contrast, MCP-DP avoids such phenomenon thanks to its multiscale sampling scheme, which utilizes randomly sampled detection intervals.

\begin{figure}[ht]
\centering

\makebox[\textwidth][c]{%
\begin{minipage}{1.0\textwidth}
\centering

\begin{subfigure}[t]{0.43\linewidth}
    \centering
    \includegraphics[width=\linewidth]{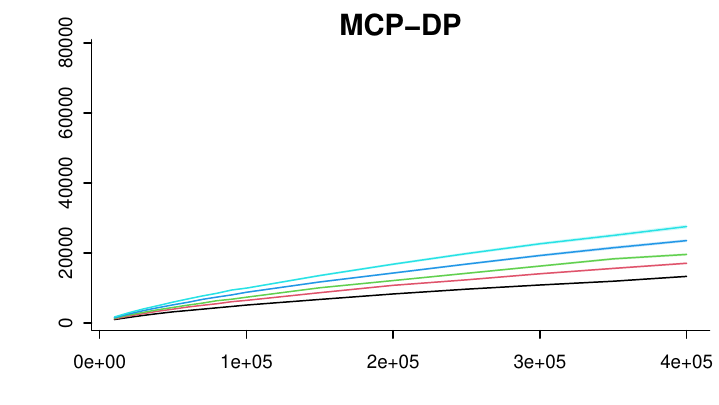}
\end{subfigure}%
\hspace{-1mm}%
\begin{subfigure}[t]{0.43\linewidth}
    \centering
    \includegraphics[width=\linewidth]{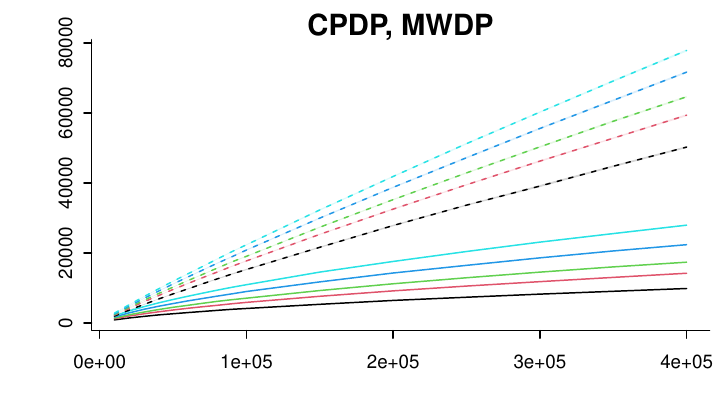}
\end{subfigure}%
\hspace{-1mm}%
\begin{subfigure}[t]{0.12\linewidth}
    \centering
    \includegraphics[width=\linewidth]{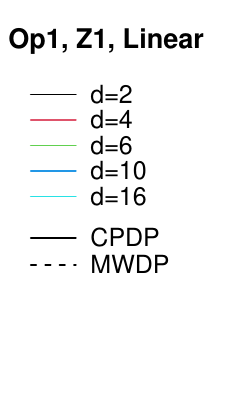}
\end{subfigure}

\vspace{-3mm}

\begin{subfigure}[t]{0.43\linewidth}
    \centering
    \includegraphics[width=\linewidth]{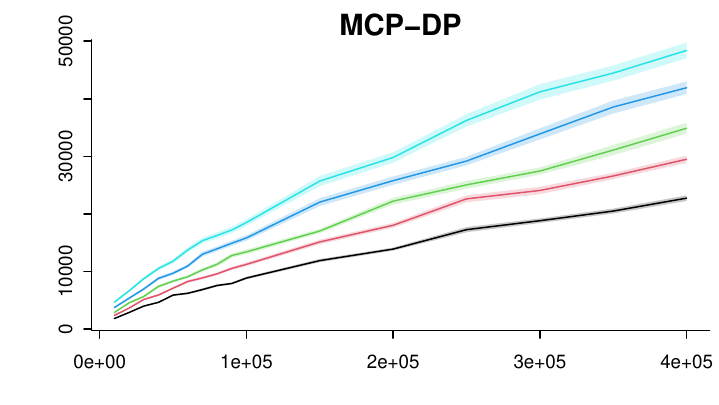}
\end{subfigure}%
\hspace{-1mm}%
\begin{subfigure}[t]{0.43\linewidth}
    \centering
    \includegraphics[width=\linewidth]{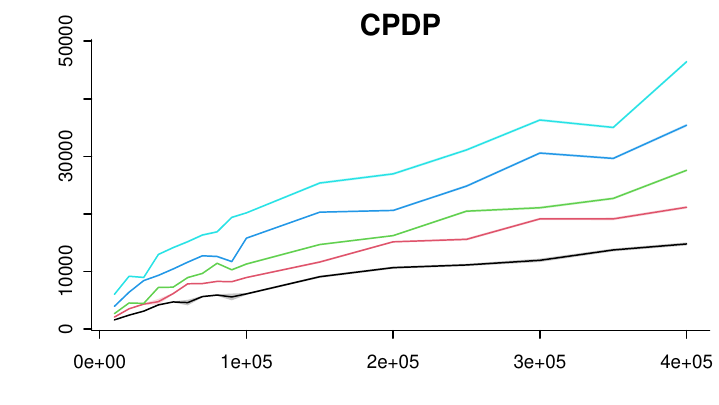}
\end{subfigure}%
\hspace{-1mm}%
\begin{subfigure}[t]{0.12\linewidth}
    \centering
    \includegraphics[width=\linewidth]{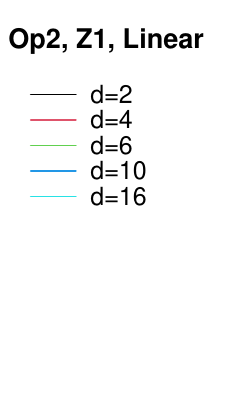}
\end{subfigure}

\vspace{-3mm}

\begin{subfigure}[t]{0.43\linewidth}
    \centering
    \includegraphics[width=\linewidth]{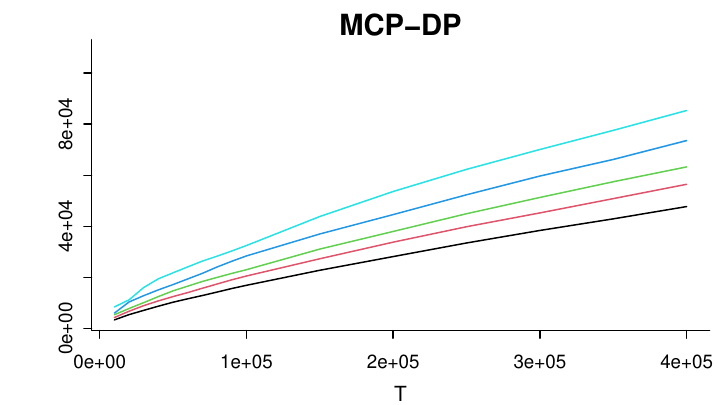}
\end{subfigure}%
\hspace{-1mm}%
\begin{subfigure}[t]{0.43\linewidth}
    \centering
    \includegraphics[width=\linewidth]{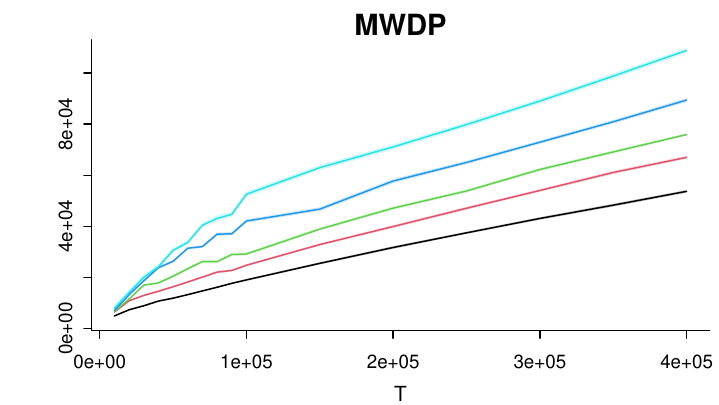}
\end{subfigure}%
\hspace{-1mm}%
\begin{subfigure}[t]{0.12\linewidth}
    \centering
    \includegraphics[width=\linewidth]{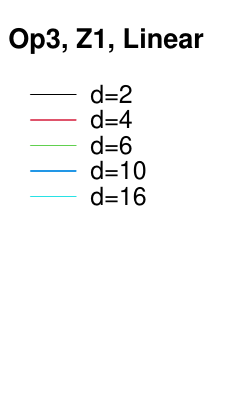}
\end{subfigure}

\end{minipage}%
}

\caption{Performance under baseline settings with linear demand and uniform contexts (Z1) in Op1 (no change), Op2 (abrupt change), Op3 (smooth change). Each row corresponds to a non-stationarity setting, with identical $y$-axis within the row for better visualization.}
\label{fig:linear_regular_case_z1}
\end{figure}

For more intuition of the adaptivity of MCP-DP, \Cref{fig:Regret_Path} plots its average regret path $\{\overline{R}_{t,d}, t\in [T]\}$ under Op1-Op3 for $T=4\times 10^5$ and $d\in \{2,4,6,10,16\}$. Furthermore, Table \ref{tab:mcp-dp} reports the average number of restarts triggered by failed LRTs of MCP-DP under Op1-Op3 for $d=10$. As can be seen, MCP-DP correctly identifies that Op1 is stationary while Op2 has three abrupt changes. On the other hand, to handle the smooth changes in Op3, MCP-DP restarts more often as $T$ increases.

\Cref{fig:logistic_regular_case_z1} summarizes the performance of all methods under the logistic demand, which gives essentially the same observations as that in \Cref{fig:linear_regular_case_z1} under the linear demand, further confirming the robustness of our theoretical results and numerical findings for the GLM setting.


\begin{table}[ht]
\centering
\begin{tabular}{lcccccccccc}
\hline \hline
$T$  & 10000 & 20000 & 40000 & 60000 & 80000 & 100000 & 200000 & 300000 & 400000 \\\hline
Op1 & 0.00 & 0.00 & 0.00 & 0.00 & 0.00 & 0.00 & 0.00 & 0.00 & 0.00 \\ 
Op2 & 3.00 & 3.00 & 3.00 & 3.00 & 3.00 & 3.00 & 3.00 & 3.00 & 3.00 \\ 
Op3 & 9.74 & 10.79 & 18.18 & 19.21 & 19.93 & 20.55 & 26.77 & 29.29 & 30.73\\
\hline \hline
\end{tabular}
\caption{Average number of restarts by MCP-DP under Op1, Op2, Op3 for $d=10$ with linear demand and uniform context (Z1).}
\label{tab:mcp-dp}
\end{table}

\begin{figure}[ht]
\centering

\makebox[\textwidth][c]{%
\begin{minipage}{1.12\textwidth}
\centering

\begin{subfigure}[t]{0.335\linewidth}
    \centering
    \includegraphics[width=\linewidth]{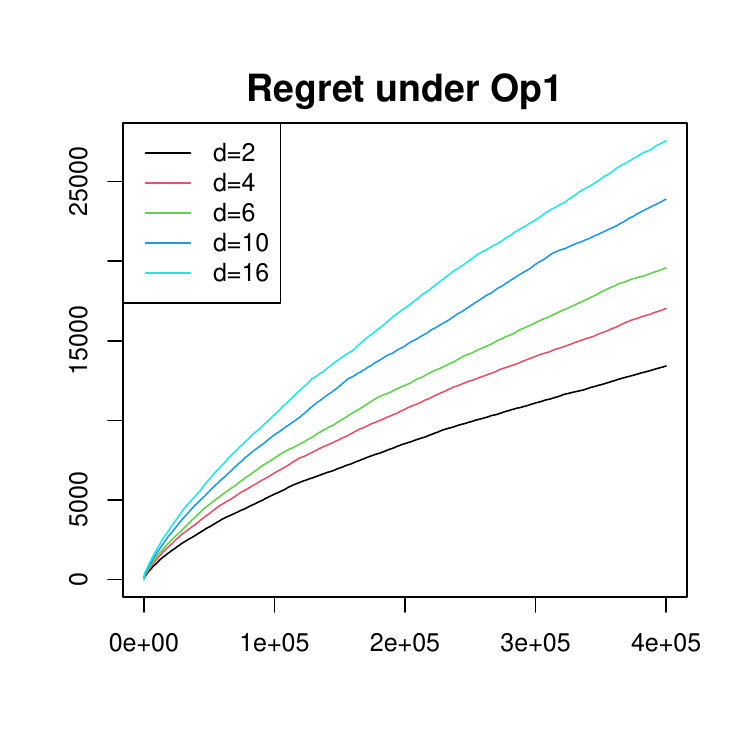}
\end{subfigure}%
\hspace{-7mm}%
\begin{subfigure}[t]{0.335\linewidth}
    \centering
    \includegraphics[width=\linewidth]{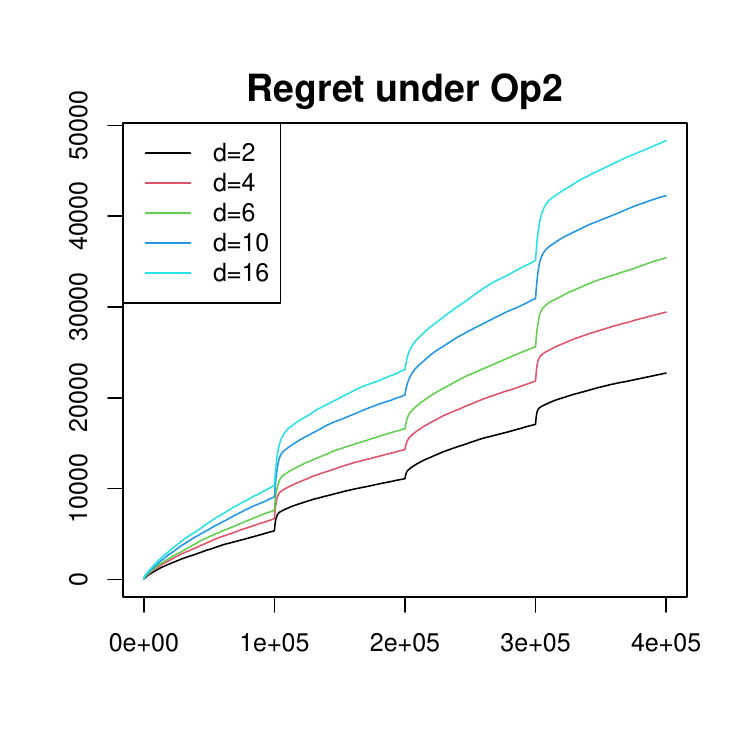}
\end{subfigure}%
\hspace{-7mm}%
\begin{subfigure}[t]{0.335\linewidth}
    \centering
    \includegraphics[width=\linewidth]{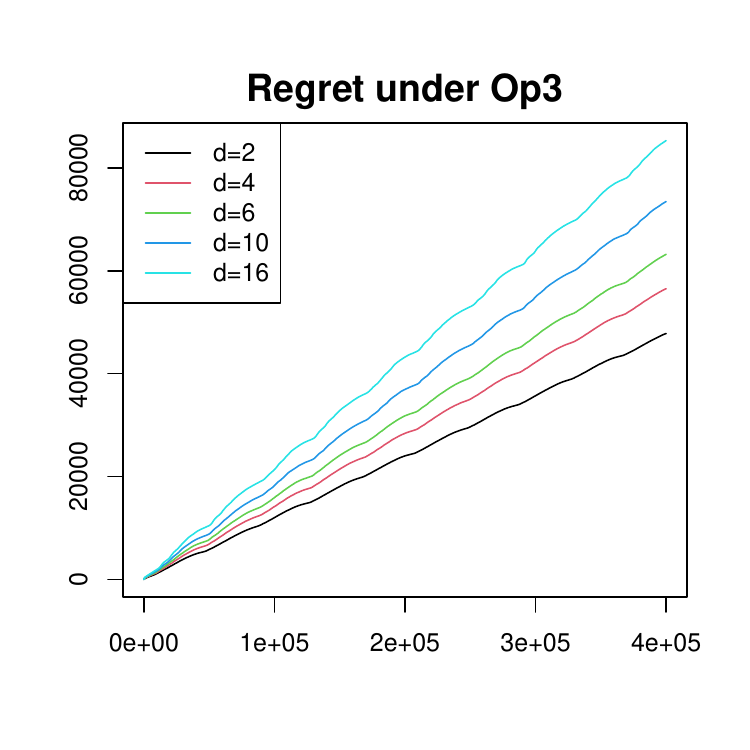}
\end{subfigure}
\end{minipage}%
}
\vspace{-8mm}

\caption{The average regret path of MCP-DP under Op1-Op3 with $T=4\times 10^5$.}
\label{fig:Regret_Path}
\end{figure}

\begin{figure}[ht]
\centering

\makebox[\textwidth][c]{%
\begin{minipage}{1.0\textwidth}
\centering

\begin{subfigure}[t]{0.43\linewidth}
    \centering
    \includegraphics[width=\linewidth]{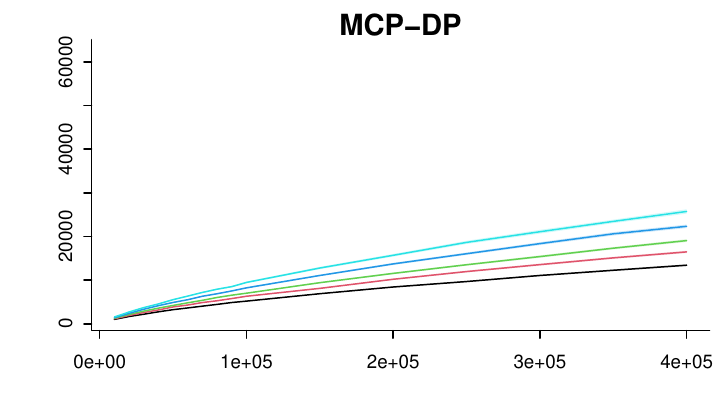}
\end{subfigure}%
\hspace{-1mm}%
\begin{subfigure}[t]{0.43\linewidth}
    \centering
    \includegraphics[width=\linewidth]{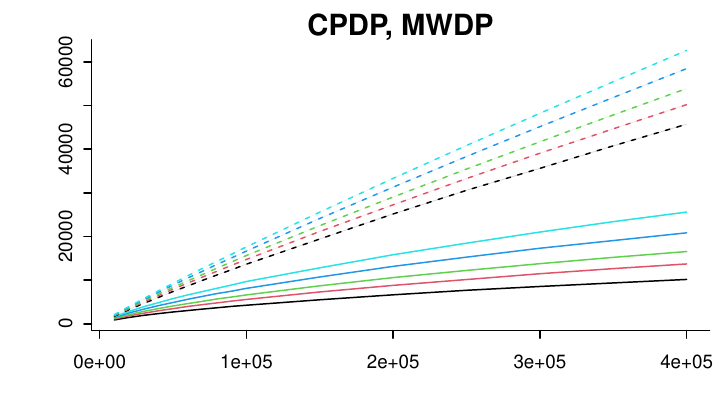}
\end{subfigure}%
\hspace{-1mm}%
\begin{subfigure}[t]{0.12\linewidth}
    \centering
    \includegraphics[width=\linewidth]{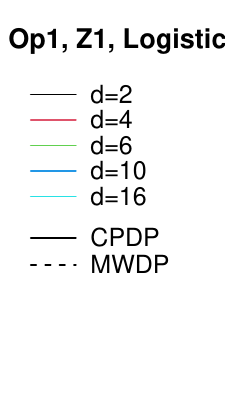}
\end{subfigure}

\vspace{-3mm}

\begin{subfigure}[t]{0.43\linewidth}
    \centering
    \includegraphics[width=\linewidth]{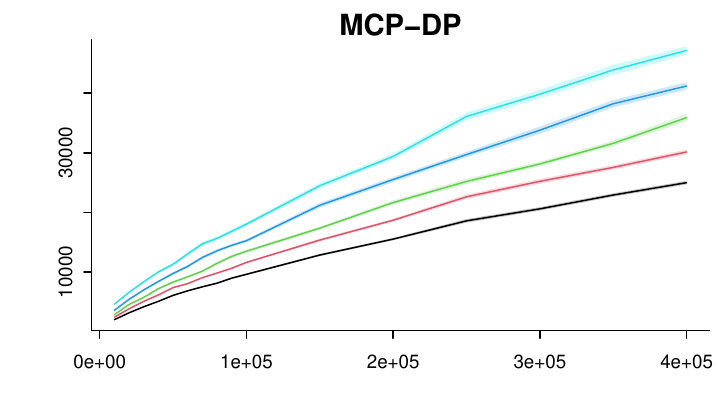}
\end{subfigure}%
\hspace{-1mm}%
\begin{subfigure}[t]{0.43\linewidth}
    \centering
    \includegraphics[width=\linewidth]{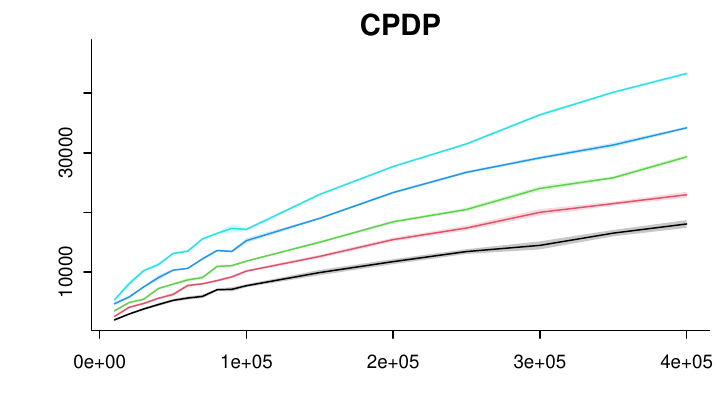}
\end{subfigure}%
\hspace{-1mm}%
\begin{subfigure}[t]{0.12\linewidth}
    \centering
    \includegraphics[width=\linewidth]{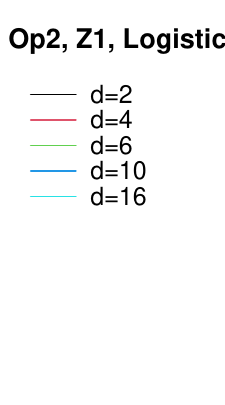}
\end{subfigure}

\vspace{-3mm}

\begin{subfigure}[t]{0.43\linewidth}
    \centering
    \includegraphics[width=\linewidth]{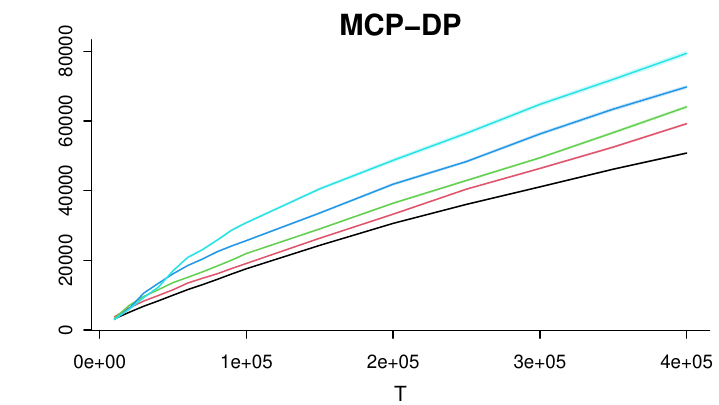}
\end{subfigure}%
\hspace{-1mm}%
\begin{subfigure}[t]{0.43\linewidth}
    \centering
    \includegraphics[width=\linewidth]{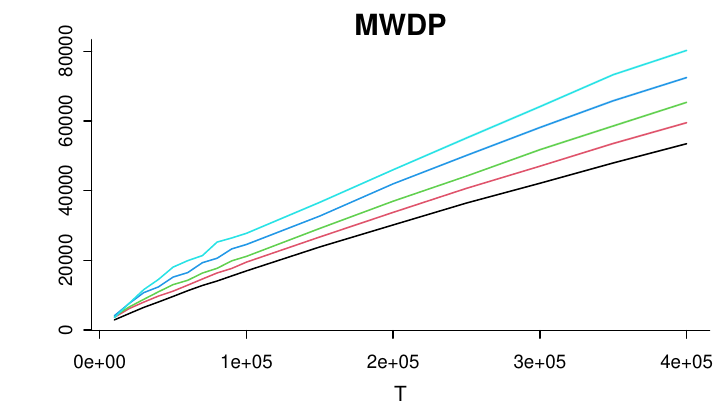}
\end{subfigure}%
\hspace{-1mm}%
\begin{subfigure}[t]{0.12\linewidth}
    \centering
    \includegraphics[width=\linewidth]{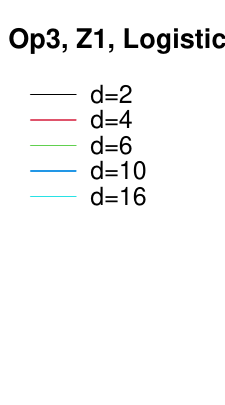}
\end{subfigure}

\end{minipage}%
}

\caption{Performance under baseline settings with logistic demand and uniform contexts (Z1) in Op1 (no change), Op2 (abrupt change), Op3 (smooth change).}
\label{fig:logistic_regular_case_z1}
\end{figure}

\textbf{Numerical results under (Z2).} We set $\theta^{(1)'}=(2\cdot\textbf{1}_d^\top, \textbf{1}_d^\top)^\top$ and $\theta^{(2)'}=(4\cdot\textbf{1}_d^\top, 2\cdot\textbf{1}_d^\top)^\top$ for (Z2). Note that by design, the expected attractiveness and price sensitivity of $\theta^{(1)'}$ and $\theta^{(2)'}$ under (Z2) match that of $\theta^{(1)}$ and $\theta^{(2)}$ under (Z1), respectively. The detailed results under (Z2) can be found in Section \ref{sec:add_simu} of the supplement, where the same phenomenon as that under (Z1) is seen, i.e.\ MCP-DP is adaptive and matches the performance of optimal benchmark methods under both structured changes Op2 and unstructured changes Op3 for linear and logistic demand.

We remark that, together, the results under (Z1) and (Z2) further confirm the validity of our theoretical results in Theorem \ref{thm:mcpdp-upper} and the tightness of our proposed notion of {design-adjusted} variation budget in \Cref{dfn_variation}. In particular, the design matrix $\Sigma_z$ is significantly different under (Z1) and (Z2), where $\lambda_{\max}(\Sigma_z)\asymp 1$ and $\lambda_{\min}(\Sigma_z)\asymp 1/d$ under (Z1), while $\lambda_{\max}(\Sigma_z)=\lambda_{\min}(\Sigma_z) =1/d $ under (Z2). However, the performance of MCP-DP remain stable across (Z1) and (Z2) for all three non-stationarity settings Op1-Op3. Moreover, the $L_2$-based variation budget for Op3 is $O(1)$ under (Z1) but $O(d)$ under (Z2), which cannot explain the stable performance of MCP-DP across (Z1) and (Z2), while in contrast, the {design-adjusted} variation budget for Op3 is $O(1)$ under both (Z1) and (Z2) and thus offers a cleaner theoretical explanation.

\subsubsection{Complex settings} 
We now consider more challenging settings where the number of stationary segments $s_T$ or the variation budget $V_T$ diverges with $T$ to further examine the robustness of MCP-DP. For illustration, we focus on the linear demand model under (Z1). We set $\theta^{(3)}=(\alpha^{(3)\top}, \beta^{(3)\top})^\top$ with $\alpha^{(3)}=(0, 6/\sqrt{d-1}\cdot \textbf{1}_{d-1}^\top)^\top$ and $\beta^{(3)}=(0, 0.5/\sqrt{d-1}\cdot \textbf{1}_{d-1}^\top)^\top$ for a more heightened high-demand season. Denote $m_{d,T}:=2d(\log(dT))^{1.1}$ and $n_{d,T}:=\sqrt{dT}$, which is the initial exploration-exploitation schedule of CPDP. We consider three non-stationarity settings:
\begin{itemize}
    \item \textbf{Op4} (CPDP-blind abrupt change):
    \[
    \theta_t=\begin{cases} 
		\theta^{(1)}, & \{(t-1)\bmod{(m_{d,T}+n_{d,T})}\}+1\leq m_{d,T},\\
		\theta^{(3)}, & \{(t-1)\bmod{(m_{d,T}+n_{d,T})}\}+1 > m_{d,T};\\
	\end{cases}
    \]
    \item \textbf{Op5} (diverging abrupt change):
    \[
    \theta_t=\begin{cases} 
		\theta^{(1)}, & (t-1) \,\operatorname{div}\, 4\sqrt{T} \text{ is even},\\
		\theta^{(3)}, & (t-1) \,\operatorname{div}\, 4\sqrt{T} \text{ is odd};\\
	\end{cases}
    \]
    \item \textbf{Op6} (diverging smooth change):
    \[
    \theta_t= (1+\cos(\pi t/(4\sqrt{T})))/2\cdot \theta^{(1)}
    + (1-\cos(\pi t/(4\sqrt{T})))/2 \cdot \theta^{(3)}.
    \]
\end{itemize}

Op4 is designed as an adversarial setting for CPDP, where the stationary segments coincide with the {pre-fixed} experimentation-exploitation cycle of CPDP, making the process stationary on the price experimentation cycles, hence CPDP-blind. Op5 is a more challenging case of Op2, where the number of change-points $s_T \asymp\sqrt{T}$ instead of being constant. Op6 is a more challenging case of Op3, where the variation budget $V_T\asymp\sqrt{T}$ instead of being constant. For more intuition, \Cref{fig:Theta_Path} of the supplement plots the sequence of true model parameters $\{\theta_t\}_{t=1}^T$ for Op4-Op6.

\Cref{fig:linear_extreme_case} summarizes the performance for each algorithm under Op4-Op6. First, as is expected, CPDP completely fails Op4, as it cannot detect any change-points, while MCP-DP provides much better performance thanks to the robustness of its multiscale sampling scheme. Second, under Op5 where $s_T\asymp \sqrt{T}$, MCP-DP gives notably better performance than CPDP, especially for large $d$, where its regret is 25-45\% less than that of CPDP. The reason is that the pre-fixed detection schedule of CPDP may cause longer detection delay for change-points, which has a greater impact on regret  when there is a large number of change-points. Third, under Op6 where $V_T\asymp \sqrt{T}$, the performance of MWDP is clearly sub-optimal. This is to be expected as MWDP is non-adaptive and assumes the unknown $V_T$ to be 1 in its implementation, mismatching the true $V_T$, while MCP-DP is adaptive and does not require the knowledge of $V_T$, thus providing robust performance.

\begin{figure}[ht]
\centering

\makebox[\textwidth][c]{%
\begin{minipage}{1.0\textwidth}
\centering

\begin{subfigure}[t]{0.43\linewidth}
    \centering
    \includegraphics[width=\linewidth]{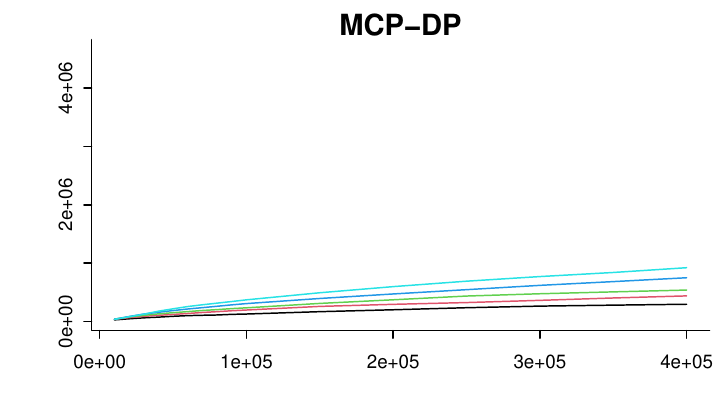}
\end{subfigure}%
\hspace{-1mm}%
\begin{subfigure}[t]{0.43\linewidth}
    \centering
    \includegraphics[width=\linewidth]{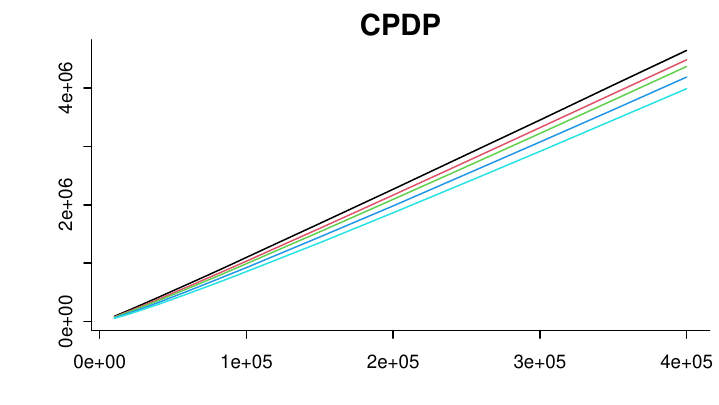}
\end{subfigure}%
\hspace{-1mm}%
\begin{subfigure}[t]{0.12\linewidth}
    \centering
    \includegraphics[width=\linewidth]{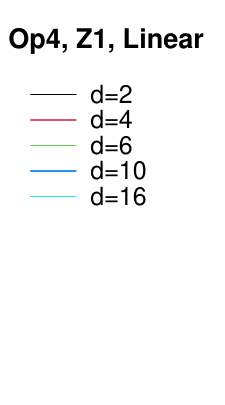}
\end{subfigure}

\vspace{-3mm}

\begin{subfigure}[t]{0.43\linewidth}
    \centering
    \includegraphics[width=\linewidth]{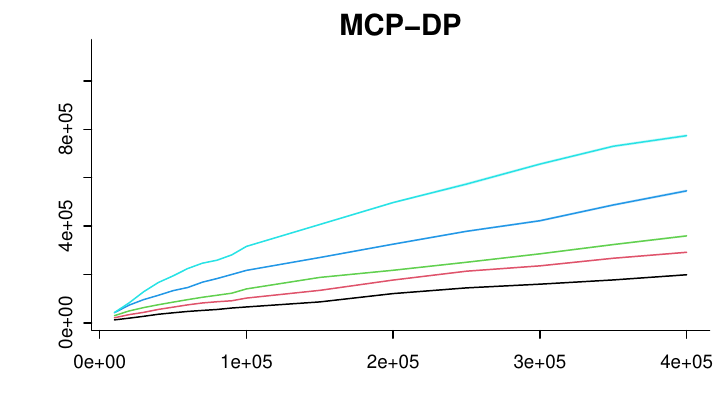}
\end{subfigure}%
\hspace{-1mm}%
\begin{subfigure}[t]{0.43\linewidth}
    \centering
    \includegraphics[width=\linewidth]{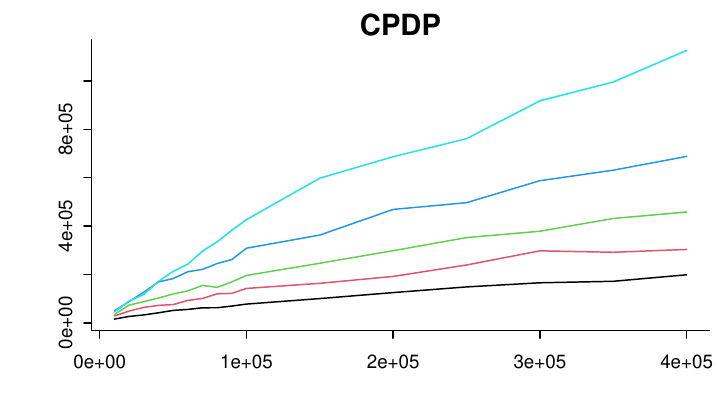}
\end{subfigure}%
\hspace{-1mm}%
\begin{subfigure}[t]{0.12\linewidth}
    \centering
    \includegraphics[width=\linewidth]{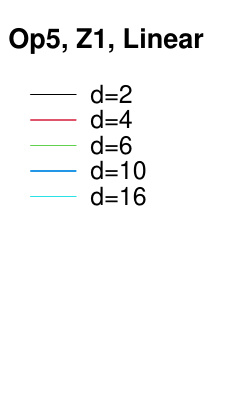}
\end{subfigure}

\vspace{-3mm}

\begin{subfigure}[t]{0.43\linewidth}
    \centering
    \includegraphics[width=\linewidth]{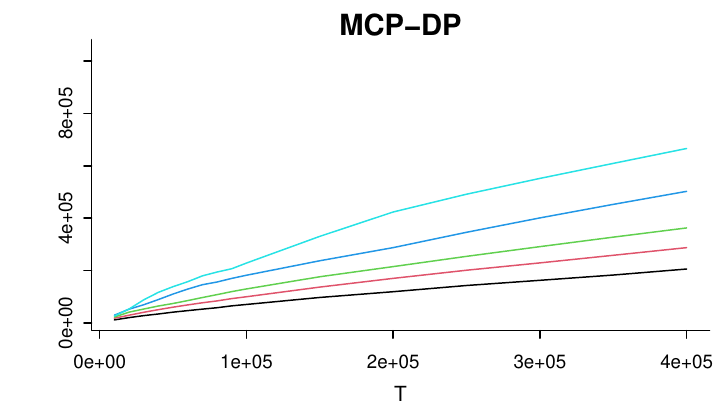}
\end{subfigure}%
\hspace{-1mm}%
\begin{subfigure}[t]{0.43\linewidth}
    \centering
    \includegraphics[width=\linewidth]{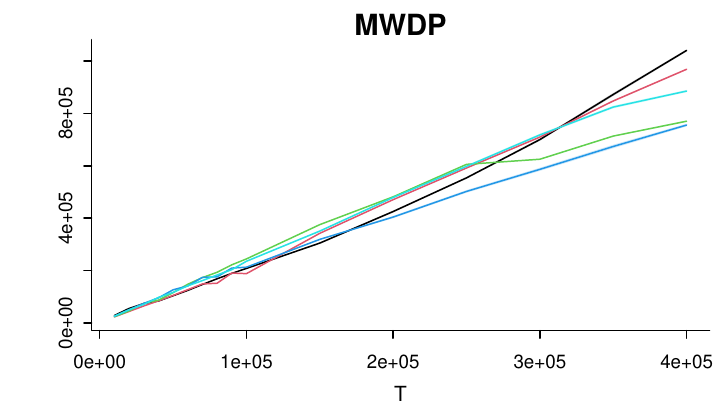}
\end{subfigure}%
\hspace{-1mm}%
\begin{subfigure}[t]{0.12\linewidth}
    \centering
    \includegraphics[width=\linewidth]{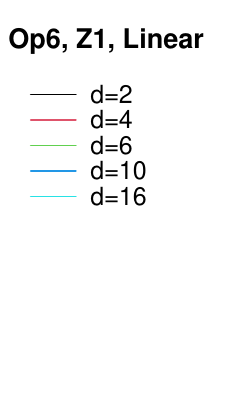}
\end{subfigure}

\end{minipage}%
}

\caption{Performance under complex settings with linear demand and uniform contexts in Op4 (CPDP-blind abrupt change), Op5 (diverging abrupt change), Op6 (diverging smooth change).}
\label{fig:linear_extreme_case}
\end{figure}


\FloatBarrier
\section{Conclusion}\label{sec:conclusion}

In this paper, we studied the non-stationary dynamic pricing problem with a GLM demand model and a stochastic context of dimension $d$. Based on a novel LRT test, we propose MCP-DP, a multiscale change-point detection based algorithm that achieves the optimal regret of order $\widetilde{O}(\sqrt{s_TdT}\wedge d^{1/3}V_T^{1/3}T^{2/3})$. Importantly, MCP-DP is the first algorithm that is adaptive to both structured and unstructured changes without requiring the knowledge of $s_T$ or $V_T.$ We establish a minimax lower bound and illustrate the utility of MCP-DP via extensive numerical experiments.

We note that MCP-DP can be readily extended to the high-dimensional dynamic pricing setting when equipped with a well-designed high-dimensional LRT test. In addition, we conjecture that MCP-DP can be extended to the non-stationary contextual bandit with varying contexts given the commonly used covariate-diversity condition. Another interesting future direction is to further relax the stochastic context assumption (i.e.\ $\lambda_{\min}(\Sigma_z)>0)$ and study adversarial contexts. However, this assumption is central to the success of the LRT test, thus relaxing it would require fundamentally different upper bound algorithms (and possibly lead to different lower bounds). Therefore, we leave a thorough investigation for future work.

\FloatBarrier
\bibliographystyle{apalike}
\bibliography{reference}

\clearpage
\setcounter{section}{0}
\setcounter{equation}{0}
\setcounter{table}{0}
\setcounter{figure}{0}
\setcounter{algorithm}{0}
\renewcommand{\thesection}{S.\arabic{section}}
\renewcommand{\theequation}{S.\arabic{equation}}
\renewcommand{\thetable}{S.\arabic{table}}
\renewcommand{\thefigure}{S.\arabic{figure}}
\renewcommand{\thealgorithm}{S.\arabic{algorithm}}
\renewcommand{\theHsection}{supp.\arabic{section}}
\renewcommand{\theHequation}{supp.\arabic{equation}}
\renewcommand{\theHtable}{supp.\arabic{table}}
\renewcommand{\theHfigure}{supp.\arabic{figure}}
\providecommand{\theHalgorithm}{}
\renewcommand{\theHalgorithm}{supp.\arabic{algorithm}}

\phantomsection
\pdfbookmark[0]{Supplementary Material}{supplementary-material}
\begin{center}
{\LARGE\bfseries Supplementary Material -- On non-stationary dynamic pricing: adaptivity and optimality\par}
\vspace{1.5em}
{\large Feiyu Jiang\textsuperscript{1} ~and Zifeng Zhao\textsuperscript{2}\par}
\vspace{0.5em}
{\normalsize \textsuperscript{1}School of Management, Fudan University\par}
{\normalsize \textsuperscript{2}Mendoza College of Business, University of Notre Dame\par}
\vspace{0.5em}
{\normalsize \today\par}
\end{center}

The supplementary material is organized as follows.  \Cref{sec:add_numerical} provides additional results for simulation.  \Cref{sec:proof_of_lb} provides technical details for the minimax lower bound results in \Cref{sec:lb} of the main text. \Cref{sec:proof1} provides formal proofs for technical details in Section \ref{subsec:mcpdp-detection}. \Cref{sec:proof2} provides upper bound results for MCP-DP in \Cref{thm:mcpdp-upper}.


\section{Additional Numerical Results}\label{sec:add_numerical}

\subsection{More details of CPDP and MWDP}\label{sec:implementation_details}

\noindent\textbf{CPDP}: To be adaptive to the number of change-points $s_T-1$, CPDP in \cite{zhao2026high} fixes the experimentation length $m=c_m d(\log (dT))^{1.1}$ but varies the exploitation length $n.$ In particular, denote $k_t$ as the detected number of change-points up to time point $t$ (i.e.\ how many times CPDP has restarted), it sets $n_t=\sqrt{dT/(k_t+1)}$ for the exploitation cycle that starts at $t.$ Note that before any change-point is detected, $k_t=0$ and thus $n_t=\sqrt{dT}$. Thus, the initial exploitation length is $\sqrt{dT}$ but then shortens as more change-points $k_t$ are detected (i.e.\ the exploration-exploitation ratio increases as $k_t$ increases). We refer to \cite{zhao2026high} for more details of CPDP.

\textbf{MWDP}: MWDP is originally proposed in \cite{keskin2017chasing} and uses fixed length of experimentation and exploitation throughout $[T]$. In particular, given the knowledge of the true variation budget $V_T$, if we set $m=c_md(\log(dT))^{1.1}$, $n=d^{2/3}T^{1/3}/V_T^{1/3}$, and $k=d^{-1/3}V_T^{-1/3}T^{1/3}$, it achieves the optimal regret of order $\widetilde O(d^{1/3}V_T^{1/3}T^{2/3})$. However, since $V_T$ is unknown in practice, our implementation replaces it  by the  proxy $V_{\rm tune}=1$, which can be suboptimal when  $V_T$ diverges.

\subsection{Additional simulation results}\label{sec:add_simu}
\vspace{-5mm}
\begin{figure}[H]
\centering

\makebox[\textwidth][c]{%
\begin{minipage}{1.12\textwidth}
\centering

\begin{subfigure}[t]{0.335\linewidth}
    \centering
    \includegraphics[width=\linewidth]{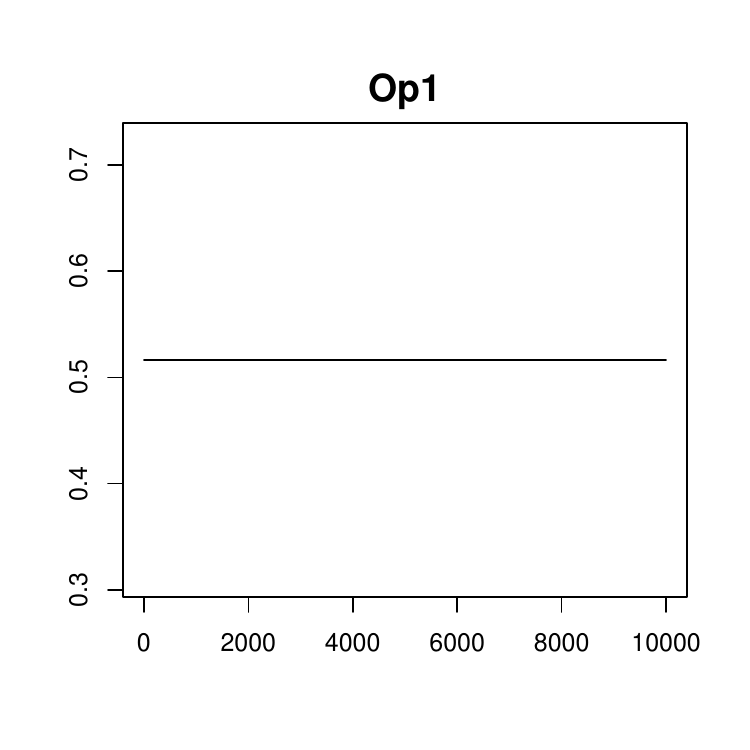}
\end{subfigure}%
\hspace{-5mm}%
\begin{subfigure}[t]{0.335\linewidth}
    \centering
    \includegraphics[width=\linewidth]{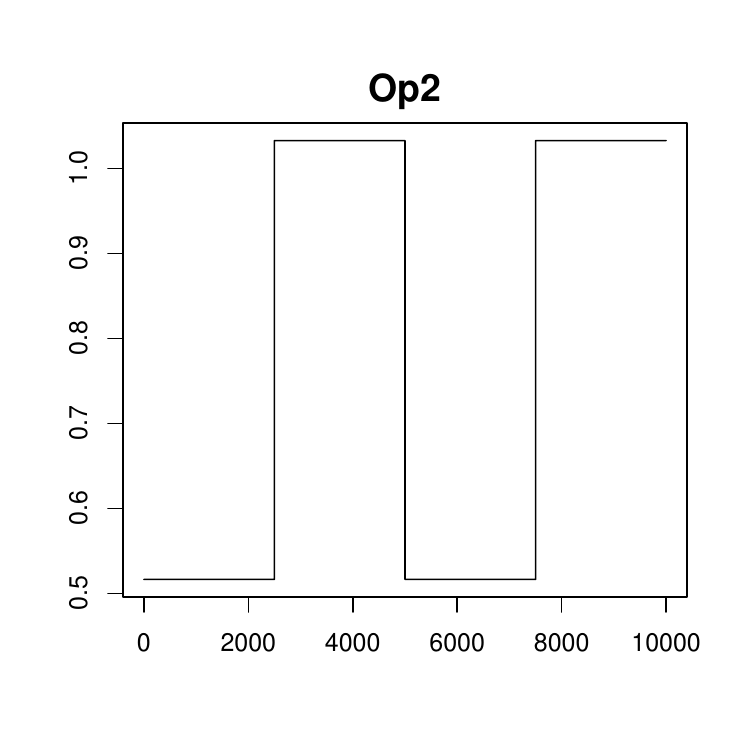}
\end{subfigure}%
\hspace{-5mm}%
\begin{subfigure}[t]{0.335\linewidth}
    \centering
    \includegraphics[width=\linewidth]{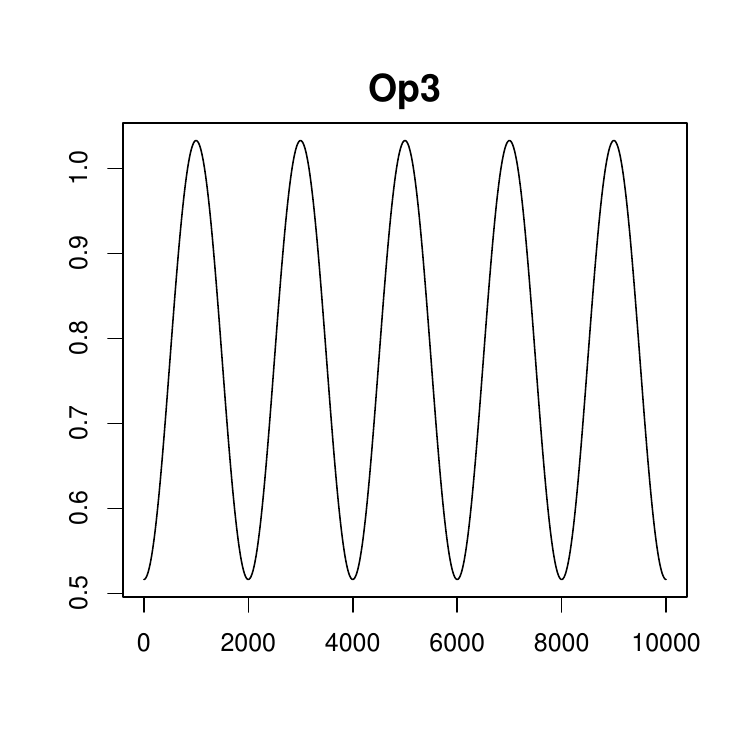}
\end{subfigure}

\vspace{-10mm}

\begin{subfigure}[t]{0.335\linewidth}
    \centering
    \includegraphics[width=\linewidth]{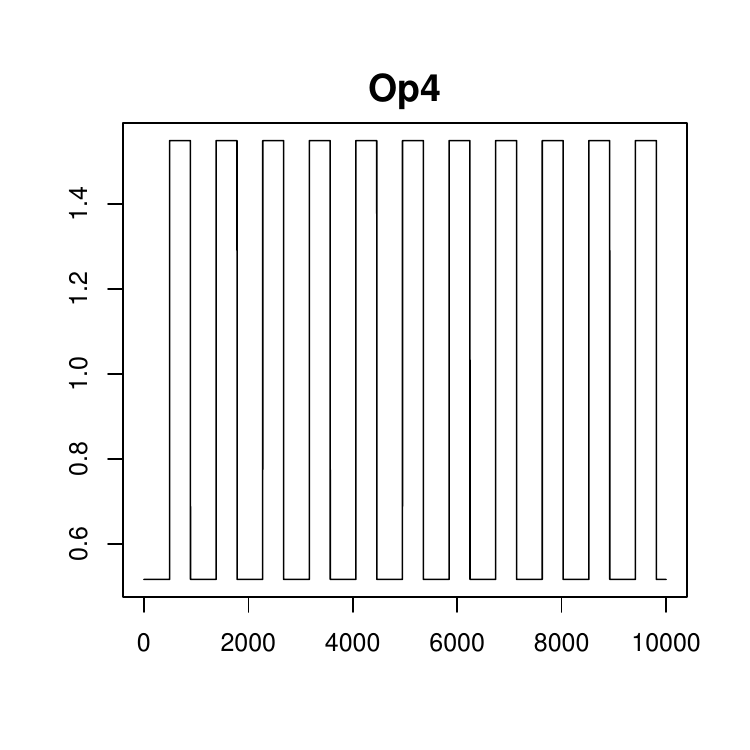}
\end{subfigure}%
\hspace{-5mm}%
\begin{subfigure}[t]{0.335\linewidth}
    \centering
    \includegraphics[width=\linewidth]{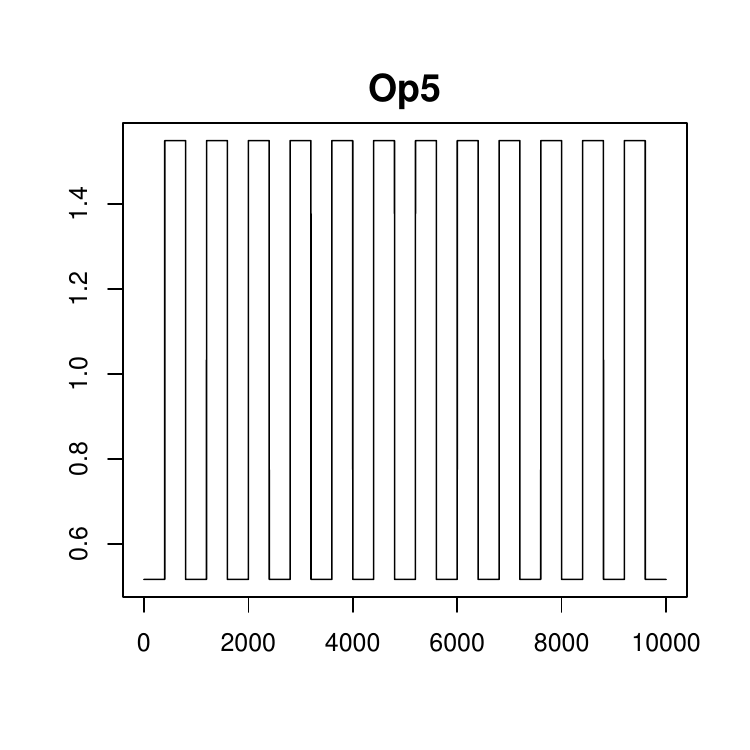}
\end{subfigure}%
\hspace{-5mm}%
\begin{subfigure}[t]{0.335\linewidth}
    \centering
    \includegraphics[width=\linewidth]{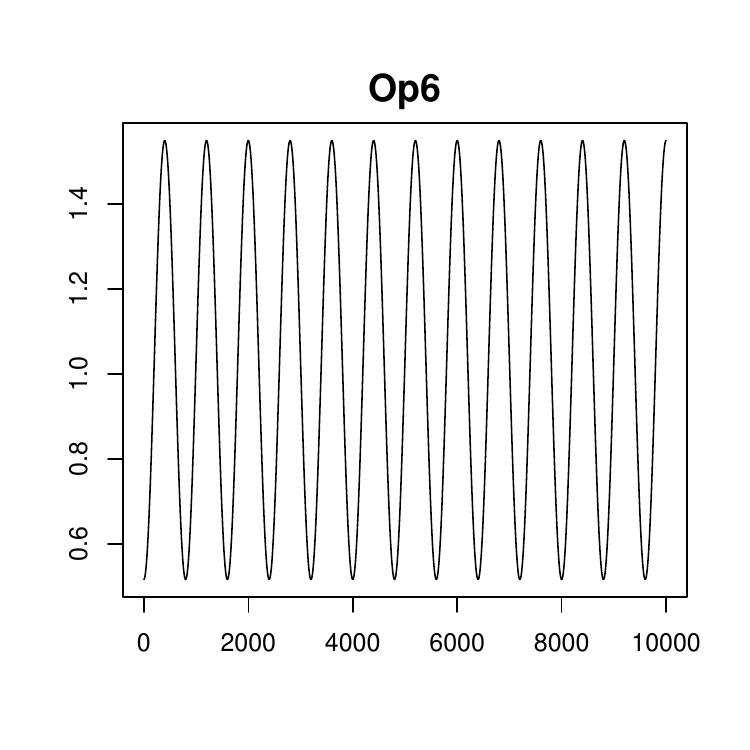}
\end{subfigure}

\end{minipage}%
}
\vspace{-8mm}

\caption{The sequence of $\{\theta_{t,2}\}_{t=1}^T$ for $d=16$ and $T=10^4$ under Op1-Op6.}
\label{fig:Theta_Path}
\end{figure}

\begin{figure}[!htbp]
\centering

\makebox[\textwidth][c]{%
\begin{minipage}{1.0\textwidth}
\centering

\begin{subfigure}[t]{0.43\linewidth}
    \centering
    \includegraphics[width=\linewidth]{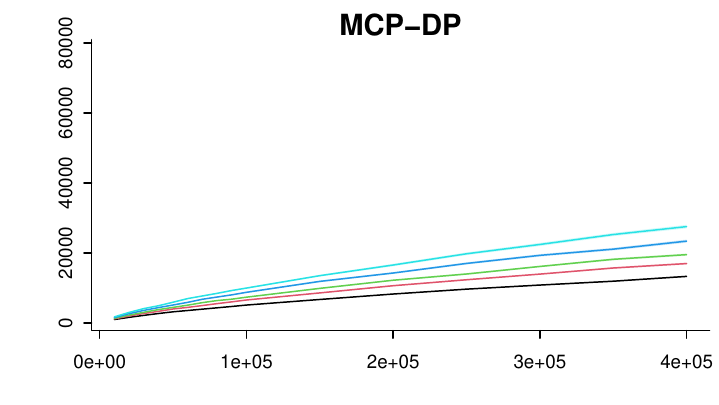}
\end{subfigure}%
\hspace{-1mm}%
\begin{subfigure}[t]{0.43\linewidth}
    \centering
    \includegraphics[width=\linewidth]{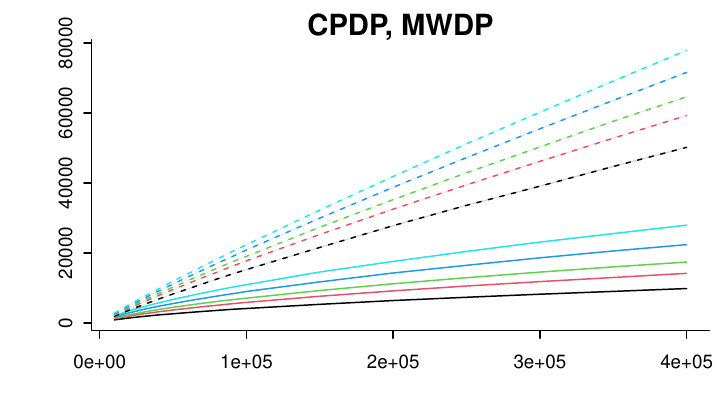}
\end{subfigure}%
\hspace{-1mm}%
\begin{subfigure}[t]{0.12\linewidth}
    \centering
    \includegraphics[width=\linewidth]{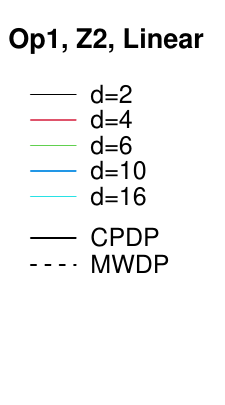}
\end{subfigure}

\vspace{-3mm}

\begin{subfigure}[t]{0.43\linewidth}
    \centering
    \includegraphics[width=\linewidth]{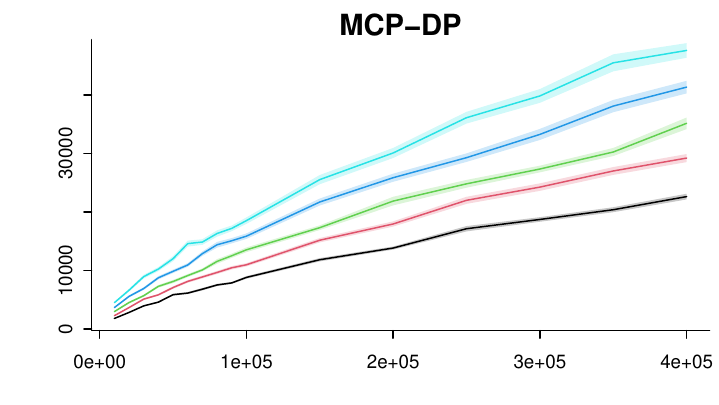}
\end{subfigure}%
\hspace{-1mm}%
\begin{subfigure}[t]{0.43\linewidth}
    \centering
    \includegraphics[width=\linewidth]{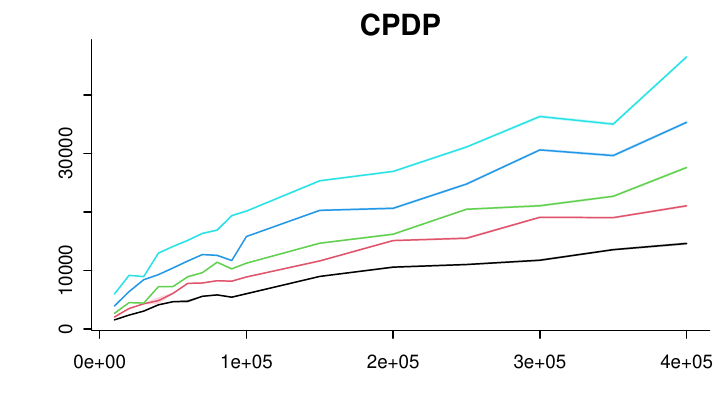}
\end{subfigure}%
\hspace{-1mm}%
\begin{subfigure}[t]{0.12\linewidth}
    \centering
    \includegraphics[width=\linewidth]{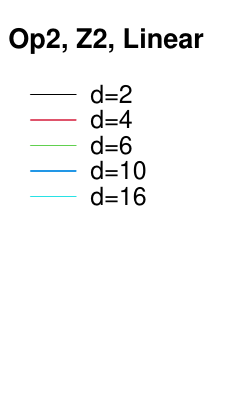}
\end{subfigure}

\vspace{-3mm}

\begin{subfigure}[t]{0.43\linewidth}
    \centering
    \includegraphics[width=\linewidth]{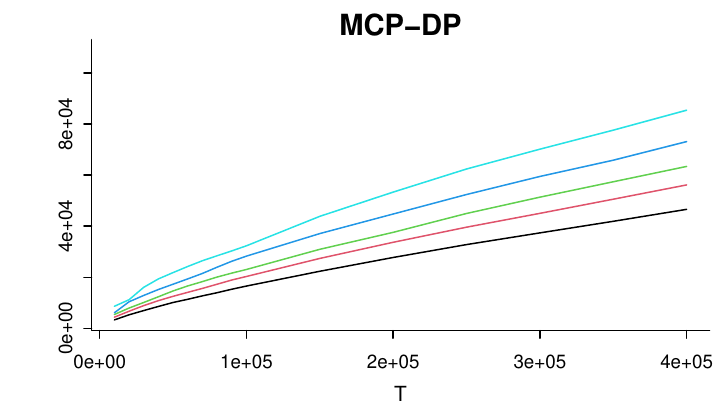}
\end{subfigure}%
\hspace{-1mm}%
\begin{subfigure}[t]{0.43\linewidth}
    \centering
    \includegraphics[width=\linewidth]{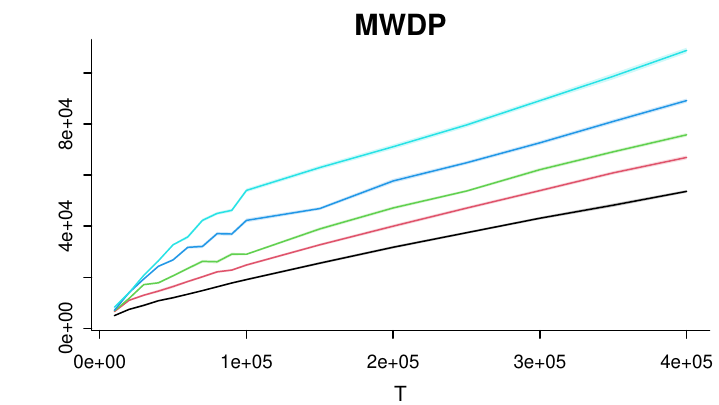}
\end{subfigure}%
\hspace{-1mm}%
\begin{subfigure}[t]{0.12\linewidth}
    \centering
    \includegraphics[width=\linewidth]{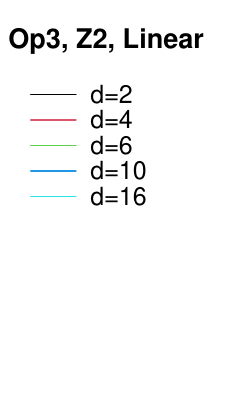}
\end{subfigure}

\end{minipage}%
}

\caption{Performance under baseline settings with linear demand and multinomial contexts (Z2) in Op1 (no change), Op2 (abrupt changes), Op3 (smooth changes).}
\label{fig:linear_regular_case_z2}
\end{figure}

\begin{figure}[!htbp]
\centering

\makebox[\textwidth][c]{%
\begin{minipage}{1.0\textwidth}
\centering

\begin{subfigure}[t]{0.43\linewidth}
    \centering
    \includegraphics[width=\linewidth]{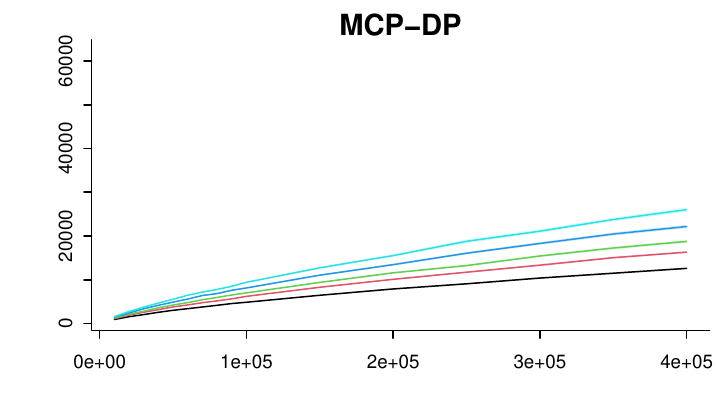}
\end{subfigure}%
\hspace{-1mm}%
\begin{subfigure}[t]{0.43\linewidth}
    \centering
    \includegraphics[width=\linewidth]{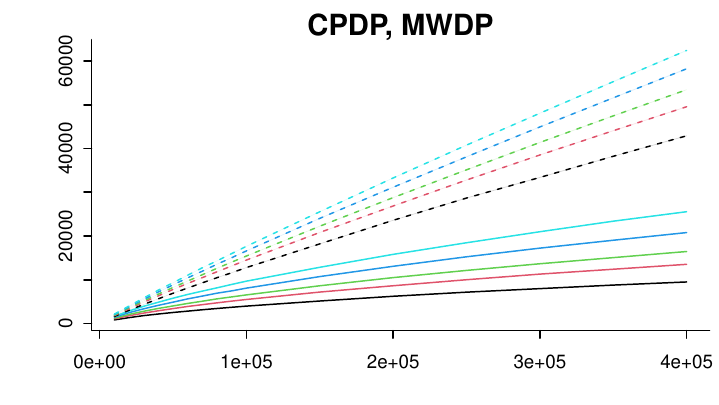}
\end{subfigure}%
\hspace{-1mm}%
\begin{subfigure}[t]{0.12\linewidth}
    \centering
    \includegraphics[width=\linewidth]{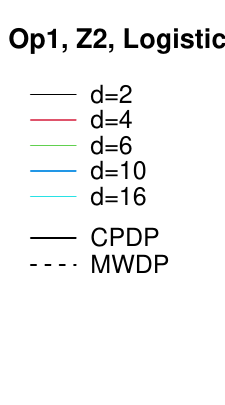}
\end{subfigure}

\vspace{-3mm}

\begin{subfigure}[t]{0.43\linewidth}
    \centering
    \includegraphics[width=\linewidth]{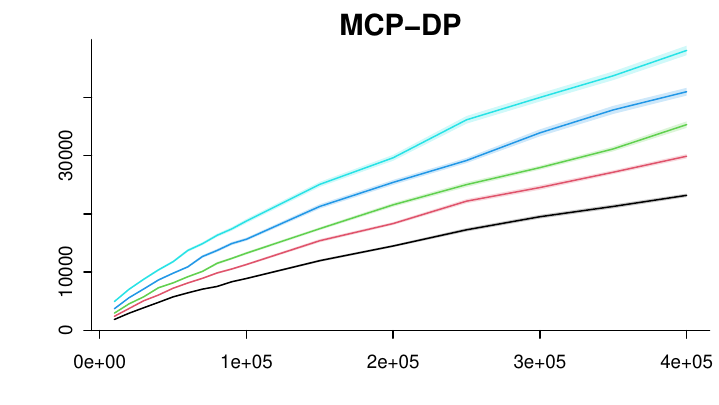}
\end{subfigure}%
\hspace{-1mm}%
\begin{subfigure}[t]{0.43\linewidth}
    \centering
    \includegraphics[width=\linewidth]{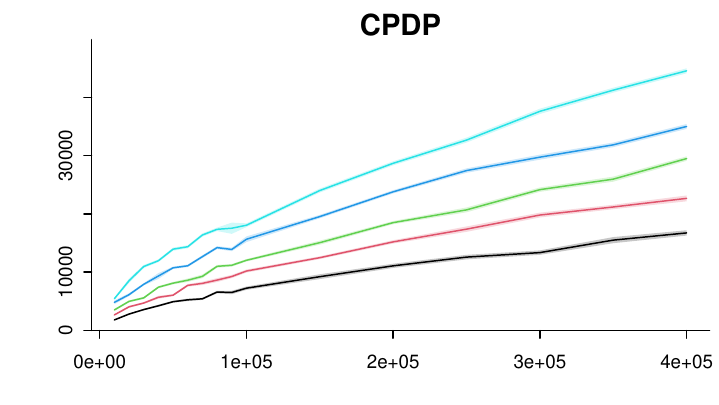}
\end{subfigure}%
\hspace{-1mm}%
\begin{subfigure}[t]{0.12\linewidth}
    \centering
    \includegraphics[width=\linewidth]{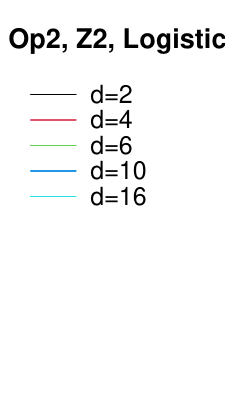}
\end{subfigure}

\vspace{-3mm}

\begin{subfigure}[t]{0.43\linewidth}
    \centering
    \includegraphics[width=\linewidth]{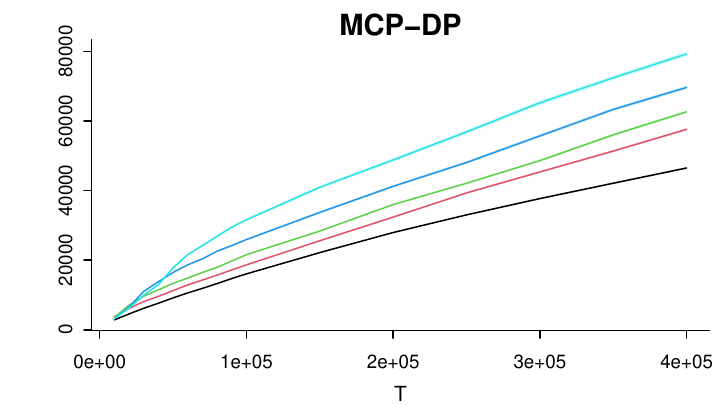}
\end{subfigure}%
\hspace{-1mm}%
\begin{subfigure}[t]{0.43\linewidth}
    \centering
    \includegraphics[width=\linewidth]{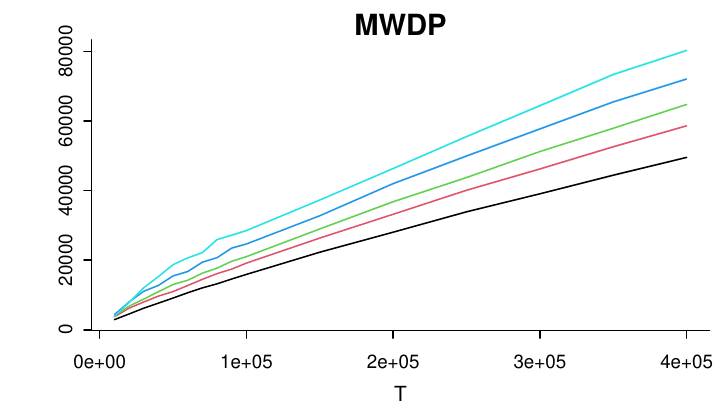}
\end{subfigure}%
\hspace{-1mm}%
\begin{subfigure}[t]{0.12\linewidth}
    \centering
    \includegraphics[width=\linewidth]{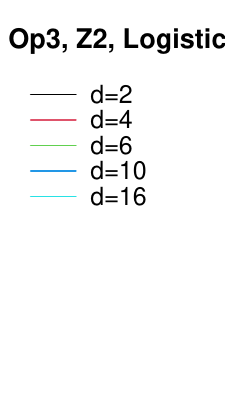}
\end{subfigure}

\end{minipage}%
}

\caption{Performance under baseline settings with logistic demand and multinomial contexts (Z2) in Op1 (no change), Op2 (abrupt changes), Op3 (smooth changes).}
\label{fig:logistic_regular_case_z2}
\end{figure}

\FloatBarrier
\section{Proof of Theorem \ref{thm:lb}}\label{sec:proof_of_lb}
 \paragraph{Step 0: Lower bound construction.}

Fix integers $T,d,s\ge 1$. 
Assume for simplicity that $N$ and $n$ are positive integers, where 
\[
N:=T/s,
\qquad
n:=N/d = \frac{T}{sd}.
\]
Consider the following subclass $\mathcal{M}(T,d,s)$ of piecewise-stationary contextual pricing problems.

{\bf Contexts:} For each segment $k\in[s]:=\{1,\dots,s\}$, define
\[
I_k:=\{(k-1)N+1,\dots,kN\}.
\]
For each $t\in[T]$, let 
\begin{equation}\label{lb_zt}
    i_t \stackrel{i.i.d.}{\sim} \mathrm{Unif}([d]),
\qquad
z_t=e_{i_t},
\end{equation}
where $e_i$ is the $i$th canonical basis vector of $\mathbb{R}^d$ for each
$i\in[d]$.  Hence the contexts are i.i.d. and uniformly sampled over the
canonical basis, satisfying $\Sigma_z=I_d/d.$

{\bf Demand Model:} Let
\[
\Omega:=\{0,1\}^{s\times(d-1)},
\qquad
\omega=(\omega_{k,j})_{k\in[s],\,j\in[d-1]}\in\Omega.
\]
For each $\omega\in\Omega$, define
\[
\widetilde\omega_{k,1}:=\mathbb I\{k \text{ is odd}\},
\qquad
\widetilde\omega_{k,i}:=\omega_{k,i-1},\quad i=2,\dots,d.
\]
For $t\in I_k$ and $z=e_i$, define
\[
\lambda_{\omega,t}(z,p)
:=
2-p+\Delta\widetilde\omega_{k,i}(1-p).
\]
Equivalently,
\[
\lambda_{\omega,t}(z,p)
=
2-p+\Delta\sum_{k=1}^{s}\sum_{i=1}^d
\widetilde\omega_{k,i}(1-p)\mathbb I\{t\in I_k,\ z=e_i\}.
\]
This is of the same form as the main model, since on segment $k$ it can be
written as
$z^\top\alpha^{(k)}-(z^\top\beta^{(k)})p$ with
$\alpha_i^{(k)}=2+\Delta\widetilde\omega_{k,i}$ and
$\beta_i^{(k)}=1+\Delta\widetilde\omega_{k,i}$.
At time $t$, after observing $z_t$, the policy chooses a price
$p_t\in[1/2,3/2]$ and then observes demand
\[
y_t
=\lambda_{\omega,t}(z_t,p_t)+\varepsilon_t,
\qquad
\varepsilon_t\stackrel{i.i.d.}{\sim}\mathcal{N}(0,1),
\]
where $\Delta>0$ is chosen later. 
In other words, if $t\in I_k$ and $z_t=e_i$, then
$$
y_t
=2-p_t+\Delta \widetilde\omega_{k,i}(1-p_t)+
\varepsilon_t.
$$
The instantaneous revenue function, conditional on $z_t$ and $p_t$, is
$$r_t(p_t,z_t,\theta_t)=p_t\lambda_{\omega,t}(z_t,p_t).$$
Note that $\widetilde\omega_{k,i}\in\{0,1\}$.  To facilitate analysis, for $u\in\{0,1\}$, define
\[
\mu_u(p):=p\{2+\Delta u-(1+\Delta u)p\}.
\]
Then the oracle price under state $u$ is
\[
q_u:=\argmax_{p\in[1/2,3/2]}\mu_u(p)=\frac{2+\Delta u}{2(1+\Delta u)},
\]
so that we can show,
\[
q_0=1,\qquad q_1=\frac{2+\Delta}{2+2\Delta},\qquad q_0-q_1=\frac{\Delta}{2(1+\Delta)}.
\]
Moreover, we have 
\[
\mu_u(q_u)-\mu_u(p)=(1+\Delta u)(p-q_u)^2.
\]

\paragraph{Step 1: Exact number of stationary pieces.}
Because
\[
\widetilde\omega_{k,1}=\mathbb I\{k \text{ is odd}\},
\]
the parameter attached to coordinate $i=1$ changes at every segment boundary.
Hence the demand model changes at each boundary between $I_k$ and $I_{k+1}$,
$k=1,\dots,s-1$. Therefore every model in the family has exactly $s$
stationary pieces, or equivalently $s-1$ change-points.  


\paragraph{Step 2: Good event for stochastic contexts.}
For each $(k,i)\in[s]\times[d]$, define
\[
\mathcal T_{k,i}:=\{t\in I_k:z_t=e_i\},
\qquad
M_{k,i}:=|\mathcal T_{k,i}|.
\]
Then we can see that $M_{k,i}\sim \mathrm{Binom}(N,1/d)$ with $ \E(M_{k,i})=n$.
Define the good event
\[
\mathcal A
:=
\bigcap_{k=1}^{s}\bigcap_{i=2}^d
\left\{
\frac n2\le M_{k,i}\le \frac{3n}{2}
\right\}.
\]
By the Chernoff bounds, for each fixed $(k,i)$ with $i\ge 2$,
\[
\mathbb{P}\left(|M_{k,i}-n|>\frac n2\right)\le 2 \exp(-n/12),
\]
and by a union bound argument, 
\[
\mathbb{P}(\mathcal A^c)\le 2sd\exp(-n/12).
\]
Therefore, if $n\geq 12\log(32sd)$, then  $\mathbb{P}(\mathcal A^c)\leq 1/16.$
\paragraph{Step 3: Regret implies classification error.}
Fix $(k,j)\in[s]\times[d-1]$, and recall that the corresponding context coordinate is $i=j+1$.
Define
\[
L_{k,j}(u):=\sum_{t\in\mathcal T_{k,j+1}}(p_t-q_u)^2,
\qquad u\in\{0,1\},
\]
and let
\[
\widehat\omega_{k,j}:=\arg\min_{u\in\{0,1\}}L_{k,j}(u).
\]
Recall that $\delta:=q_0-q_1=\Delta/{[2(1+\Delta)]}.$

Note that for any real number $a$,
\[
(a-q_0)^2+(a-q_1)^2
=
2\Bigl(a-\frac{q_0+q_1}{2}\Bigr)^2+\frac{\delta^2}{2}
\ge \frac{\delta^2}{2}.
\]
Applying this with $a=p_t$ and summing over $t\in\mathcal T_{k,j+1}$ yields
\[
L_{k,j}(0)+L_{k,j}(1)\ge \frac{\delta^2}{2}M_{k,j+1}.
\]
If $L_{k,j}(0)>L_{k,j}(1)$, then $\widehat\omega_{k,j}=1$. Hence, on the event $\{\omega_{k,j}=0\}$,
\[
L_{k,j}(\omega_{k,j})=L_{k,j}(0)
\ge \frac{L_{k,j}(0)+L_{k,j}(1)}{2}
\ge \frac{\delta^2}{4}M_{k,j+1}.
\]
Similarly, if $L_{k,j}(1)>L_{k,j}(0)$, then $\widehat\omega_{k,j}=0$. Hence, on the event $\{\omega_{k,j}=1\}$,
\[
L_{k,j}(\omega_{k,j})=L_{k,j}(1)
\ge \frac{L_{k,j}(0)+L_{k,j}(1)}{2}
\ge \frac{\delta^2}{4}M_{k,j+1}.
\]
Therefore, on the event $\{\widehat\omega_{k,j}\neq \omega_{k,j}\}$, we must have
\[
L_{k,j}(\omega_{k,j})\ge \frac{L_{k,j}(0)+L_{k,j}(1)}{2}
\ge \frac{\delta^2}{4}M_{k,j+1}.
\]
In other words, 
\[
L_{k,j}(\omega_{k,j})
\ge
\frac{\delta^2}{4}M_{k,j+1}\mathbb I\{\widehat\omega_{k,j}\neq \omega_{k,j}\}.
\]
Since $
\mu_u(q_u)-\mu_u(p)=(1+\Delta u)(p-q_u)^2\ge (p-q_u)^2, $
we obtain that
\[
\sum_{t\in\mathcal T_{k,j+1}}
\Bigl(\mu_{\omega_{k,j}}(q_{\omega_{k,j}})-\mu_{\omega_{k,j}}(p_t)\Bigr)
\ge
L_{k,j}(\omega_{k,j})
\ge
\frac{\delta^2}{4}M_{k,j+1}\mathbb I\{\widehat\omega_{k,j}\neq \omega_{k,j}\}.
\]
On the good event $\mathcal A$, we have $M_{k,j+1}\ge n/2$, and therefore
\[
\sum_{t\in\mathcal T_{k,j+1}}
\Bigl(\mu_{\omega_{k,j}}(q_{\omega_{k,j}})-\mu_{\omega_{k,j}}(p_t)\Bigr)
\ge
\frac{n\delta^2}{8}\mathbb I\{\widehat\omega_{k,j}\neq \omega_{k,j}\}.
\]
Summing over $(k,j)\in[s]\times[d-1]$ gives
\begin{equation}\label{eq:regret-class}
	R_T^\pi(\omega)
	\ge
	\frac{n\delta^2}{8}
	\sum_{k=1}^{s}\sum_{j=1}^{d-1}
	\mathbb I\{\widehat\omega_{k,j}\neq \omega_{k,j},\mathcal A\}.   
\end{equation}

\paragraph{Step 4: Assouad reduction.}
Let
$
o_t:=(z_t,p_t,y_t)
$, and $
o_{1:T}:=(o_1,\dots,o_T).$
Then $P_\omega^\pi$ denotes the law of $o_{1:T}$ under model $\omega$ and policy $\pi$ (recall that by design $\widetilde \omega_{k,1}$ is fixed, thus we skip it in the notation). For each free bit $(k,j)$, define the mixture distributions
\[
P_{+(k,j)}^\pi:=\frac{1}{2^{s(d-1)-1}}\sum_{\omega:\,\omega_{k,j}=1}P_\omega^\pi,
\qquad
P_{-(k,j)}^\pi:=\frac{1}{2^{s(d-1)-1}}\sum_{\omega:\,\omega_{k,j}=0}P_\omega^\pi.
\]
By Assouad's lemma, Pinsker's inequality, and Cauchy--Schwarz,
\begin{align}\label{eq:assouad}
	\begin{split}
		&\sup_{\omega\in\Omega}
		\E_\omega^\pi\!\left[
		\sum_{k=1}^{s}\sum_{j=1}^{d-1}\mathbb I\{\widehat\omega_{k,j}\neq \omega_{k,j}\}
		\right]
		\geq \frac{1}{2}\sum_{k=1}^{s}\sum_{j=1}^{d-1}  (1-\mathrm{TV}(P_{+(k,j)}^\pi,P_{-(k,j)}^\pi)
		\\\geq& \frac{1}{2}\sum_{k=1}^{s}\sum_{j=1}^{d-1}  (1-\sqrt{D_{\rm KL}^{\rm sy}(P_{+(k,j)}^\pi,P_{-(k,j)}^\pi)/4})\\\geq& 
		\frac{s(d-1)}{2}
		\left(
		1-\sqrt{
			\frac{1}{4s(d-1)}
			\sum_{k=1}^{s}\sum_{j=1}^{d-1}
			D_{\rm KL}^{\rm sy}\!\left(P_{+(k,j)}^\pi,P_{-(k,j)}^\pi\right)
		}
		\right),
	\end{split}
\end{align}
where $D_{\rm KL}^{sy}(P,Q)=D_{\rm KL}(P|Q)+D_{\rm KL}(Q|P)$ is the symmetric
Kullback--Leibler (KL) divergence between probability measures $P$ and $Q$.

\paragraph{Step 5: Bound the total symmetric KL.}
For notational convenience, let $b=(k,j)$ denote a free bit, and write
$  
\omega\oplus b$
for the parameter vector obtained from $\omega$ by flipping only the bit $\omega_{k,j}$.

We first upper bound the symmetric KL divergence by averaging over neighboring pairs. By convexity of KL divergence in each argument,
\begin{align}
	D_{\mathrm{KL}}^{\mathrm{sy}}\!\left(P_{+b}^\pi,P_{-b}^\pi\right)
	&=
	D_{\mathrm{KL}}(P_{+b}^\pi|P_{-b}^\pi)
	+
	D_{\mathrm{KL}}(P_{-b}^\pi|P_{+b}^\pi)
	\nonumber\\
	&\le
	\frac{1}{2^{s(d-1)-1}}
	\sum_{\omega:\,\omega_b=1}
	\Bigl\{
	D_{\mathrm{KL}}(P_\omega^\pi|P_{\omega\oplus b}^\pi)
	+
	D_{\mathrm{KL}}(P_{\omega\oplus b}^\pi|P_\omega^\pi)
	\Bigr\}.
	\label{eq:KL-convexity}
\end{align}
Summing \eqref{eq:KL-convexity} over all free bits $b=(k,j)$ yields
\begin{align}
	\sum_{k=1}^{s}\sum_{j=1}^{d-1}
	D_{\mathrm{KL}}^{\mathrm{sy}}\!\left(P_{+(k,j)}^\pi,P_{-(k,j)}^\pi\right)
	\le
	\frac{1}{2^{s(d-1)}}
	\sum_{\omega\in\Omega}
	\sum_{k=1}^{s}\sum_{j=1}^{d-1}
	\Bigl\{
	D_{\mathrm{KL}}(P_\omega^\pi|P_{\omega\oplus(k,j)}^\pi)
	+
	D_{\mathrm{KL}}(P_{\omega\oplus(k,j)}^\pi|P_\omega^\pi)
	\Bigr\}.
	\label{eq:sumKL-average-pairs}
\end{align}

Next we bound each pairwise KL divergence using the chain rule. Fix $\omega$ and a bit $b=(k,j)$.
By the chain rule for KL divergence,
\[
D_{\mathrm{KL}}(P_\omega^\pi|P_{\omega\oplus b}^\pi)
=
\sum_{t=1}^T
\E_\omega^\pi\!\left[
D_{\mathrm{KL}}
\Bigl(
P_\omega^\pi(o_t\mid o_{<t}),
P_{\omega\oplus b}^\pi(o_t\mid o_{<t})
\Bigr)
\right].
\]
Since the policy is non-anticipating, conditional on $o_{<t}$ the price $p_t$ is the same measurable
function under both models; moreover, the context distribution is also the same under both models.
Hence the two conditional laws can differ only through the conditional law of $y_t$ given $(z_t,p_t)$.

Now the two models $\omega$ and $\omega\oplus b$ differ only on the block corresponding to segment $k$
and coordinate $j+1$. Thus, if either $t\notin I_k$ or $z_t\neq e_{j+1}$, then the conditional laws are identical and the one-step KL contribution is zero. 

Suppose now that $t\in I_k$ and $z_t=e_{j+1}$. Under our demand construction,
the conditional mean difference between $\omega$ and $\omega\oplus(k,j)$ is
$\Delta(1-p_t)$. Moreover, since both conditional laws are Gaussian with
variance $1$, the one-step KL divergence equals
\[
\frac12\bigl(\Delta(1-p_t)\bigr)^2
=
\frac{\Delta^2}{2}(1-p_t)^2.
\]
Hence
\[
D_{\mathrm{KL}}(P_\omega^\pi,P_{\omega\oplus(k,j)}^\pi)
=
\frac{\Delta^2}{2}\,
\E_\omega^\pi\!\left[
\sum_{t=1}^T
(1-p_t)^2
\mathbb I\{t\in I_k,\ z_t=e_{j+1}\}
\right].
\]
By the same argument, with the two models reversed,
\[
D_{\mathrm{KL}}(P_{\omega\oplus(k,j)}^\pi,P_\omega^\pi)
=
\frac{\Delta^2}{2}\,
\E_{\omega\oplus(k,j)}^\pi\!\left[
\sum_{t=1}^T
(1-p_t)^2
\mathbb I\{t\in I_k,\ z_t=e_{j+1}\}
\right].
\]
Substituting these into \eqref{eq:sumKL-average-pairs},  we obtain
\begin{align}
	&\sum_{k=1}^{s}\sum_{j=1}^{d-1}
	D_{\mathrm{KL}}^{\mathrm{sy}}\!\left(P_{+(k,j)}^\pi,P_{-(k,j)}^\pi\right)
	\nonumber\\
	&\le
	\frac{\Delta^2}{2^{s(d-1)}}
	\sum_{\omega\in\Omega}
	\sum_{k=1}^{s}\sum_{j=1}^{d-1}
	\E_\omega^\pi\!\left[
	\sum_{t=1}^T
	(1-p_t)^2
	\mathbb I\{t\in I_k,\ z_t=e_{j+1}\}
	\right]
	\nonumber\\
	&=
	\frac{\Delta^2}{2^{s(d-1)}}\sum_{\omega\in\Omega}
	\E_\omega^\pi\!\left[
	\sum_{t=1}^T
	(1-p_t)^2
	\sum_{k=1}^{s}\sum_{j=1}^{d-1}\mathbb I\{t\in I_k,\ z_t=e_{j+1}\}
	\right].
	\label{eq:sumKL-before-collapse}
\end{align}
For each fixed $t$, there is exactly one segment $k$ such that $t\in I_k$, and among
$j=1,\dots,d-1$, the indicator $\mathbb I\{z_t=e_{j+1}\}$ equals $1$ if and only if $z_t\neq e_1$.
Therefore
\[
\sum_{k=1}^{s}\sum_{j=1}^{d-1}\mathbb I\{t\in I_k,\ z_t=e_{j+1}\}
=
\mathbb I\{z_t\neq e_1\}.
\]
Thus \eqref{eq:sumKL-before-collapse} simplifies to
\begin{align}\label{eq:sumKL-final-raw}
	\begin{split}
		\sum_{k=1}^{s}\sum_{j=1}^{d-1}
		D_{\mathrm{KL}}^{\mathrm{sy}}\!\left(P_{+(k,j)}^\pi,P_{-(k,j)}^\pi\right)
		\le &
		\frac{\Delta^2}{2^{s(d-1)}} \sum_{\omega\in\Omega}
		\E_\omega^\pi\!\left[
		\sum_{t=1}^T
		(1-p_t)^2\mathbb I\{z_t\neq e_1\}
		\right]\\
		\leq &\Delta^2\sup_{\omega\in\Omega}\E_{{\omega}}^\pi
		\sum_{t=1}^T
		(1-p_t)^2.
	\end{split}
\end{align}
Finally, we connect  \eqref{eq:sumKL-final-raw} to the regret. For any free bit state $u\in\{0,1\}$, we have 
\[
\mu_u(q_u)-\mu_u(p)=(1+\Delta u)(p-q_u)^2\ge (p-q_u)^2,\quad \text{and}\quad (1-p)^2
\le
2(p-q_u)^2+2(1-q_u)^2.
\]
Since
$q_0=1, q_1={(2+\Delta)}/({2+2\Delta})$, we have 
$1-q_u\le q_0-q_1={\Delta}/{2(1+\Delta)}\le {\Delta}/{2},
$
hence 
\[
(1-p)^2
\le
2\{\mu_u(q_u)-\mu_u(p)\}+\frac{\Delta^2}{2}.
\]
Therefore, applying the above inequality to \eqref{eq:sumKL-final-raw}, we obtain
\begin{align}
	\sum_{k=1}^{s}\sum_{j=1}^{d-1}
	D_{\mathrm{KL}}^{\mathrm{sy}}\!\left(P_{+(k,j)}^\pi,P_{-(k,j)}^\pi\right)
	&\le
	2\Delta^2 \sup_{\omega\in\Omega} \mathbb{E}_{\omega}^{\pi} R_T(\omega) + \frac{T\Delta^4}{2}.
	\label{eq:sumKL-regret}
\end{align}

\paragraph{Step 6: Close the argument.}
Using \eqref{eq:regret-class}, $\mathbb P(\mathcal A^c)\le 1/16$, and the fact that
\[
\sum_{k,j}\mathbb I\{\widehat\omega_{k,j}\neq \omega_{k,j},\mathcal A\}
\ge
\sum_{k,j}\mathbb I\{\widehat\omega_{k,j}\neq \omega_{k,j}\}-s(d-1)\mathbb I_{\mathcal A^c},
\]
we obtain  
\[
\sup_{\omega\in\Omega}\mathbb{E}_{\omega}^{\pi}R_T^\pi(\omega)
\ge
\frac{n\delta^2}{8}
\left(
\sup_{\omega\in\Omega}\E_\omega^\pi\sum_{k=1}^{s}\sum_{j=1}^{d-1}\mathbb I\{\widehat\omega_{k,j}\neq \omega_{k,j}\}
-\frac{s(d-1)}{16}
\right).
\]

Combining this with \eqref{eq:assouad} and \eqref{eq:sumKL-regret} yields
\begin{align}\label{reg_LB}
	\begin{split}
		\sup_{\omega\in\Omega}\mathbb{E}_{\omega}^{\pi}R_T^\pi(\omega)
		\ge&
		\frac{n\delta^2 s(d-1)}{16}
		\left(
		1-\sqrt{
			\frac{\Delta^2\sup_{\omega\in\Omega}\mathbb{E}_{\omega}^{\pi}R_T^\pi(\omega)}{2s(d-1)}
			+\frac{T\Delta^4}{8 s(d-1)}}
		\right)
		-\frac{n\delta^2 s(d-1)}{128}
		\\\geq & \frac{T\Delta^2}{288}\left(
		1-\sqrt{
			\frac{\Delta^2\sup_{\omega\in\Omega}\mathbb{E}_{\omega}^{\pi}R_T^\pi(\omega)}{sd}
			+\frac{T\Delta^4}{4sd}}
		\right)- \frac{T\Delta^2}{{512}}, 
	\end{split}
\end{align}
where the second inequality holds because $ns(d-1)\geq nsd/2=T/2$ for
$d\geq 2$, and
$\delta=\Delta/[2(1+\Delta)]\in[\Delta/3,\Delta/2]$ for $\Delta\leq 1/2$.
The case $d=1$ is the covariate-free case and follows from
\cite{zhao2026high}.

We claim that there exist constants $c_0,\eta_0>0$ such that, whenever
$T\Delta^4/(sd)\leq \eta_0$,
\[
\sup_{\omega\in\Omega}\mathbb{E}_{\omega}^{\pi}R_T^\pi(\omega)
\geq c_0T\Delta^2 .
\]
Indeed, suppose for contradiction that
\[
\sup_{\omega\in\Omega}\mathbb{E}_{\omega}^{\pi}R_T^\pi(\omega)
\leq T\Delta^2/{1024} .
\]
Then
\[
\frac{
\Delta^2\sup_{\omega\in\Omega}\mathbb{E}_{\omega}^{\pi}R_T^\pi(\omega)
}{sd}
\leq
\frac{\eta_0}{{1024}}.
\]
Substituting this bound into \eqref{reg_LB} gives
\[
\sup_{\omega\in\Omega}\mathbb{E}_{\omega}^{\pi}R_T^\pi(\omega)
\geq
\frac{T\Delta^2}{{288}}
\left(1-\sqrt{\eta_0/1024+\eta_0/{4}}\right)
-
\frac{T\Delta^2}{512}.
\]
Choosing $\eta_0>0$ sufficiently small makes the right-hand side strictly
larger than $T\Delta^2/{1024}$, a contradiction.  Hence the claim follows.  In
particular, taking $\Delta^2\asymp\sqrt{sd/T}$ yields
\[
\sup_{\omega\in\Omega}\mathbb{E}_{\omega}^{\pi}R_T^\pi(\omega)
\geq c\sqrt{sdT}.
\]

\paragraph{Step 7: Embedding the construction into the variation class.}
First, we note that  for any parameter displacement $h\in\mathbb{R}^{2d}$, 
$h^\top\widetilde{\Sigma}_z h$ can be written as
\[
h^\top\widetilde{\Sigma}_z h
=
(\widetilde{\Sigma}_z^{1/2}h)^\top (\widetilde{\Sigma}_z^{1/2}h).
\]
This implies that the design-adjusted variation can be transformed to the  Euclidean quadratic variation
after the linear change of $\widetilde{x}_t=\widetilde{\Sigma}_z^{-1/2}x_t,$ or equivalently, $\widetilde{z}_t={\Sigma}_z^{-1/2}z_t$. Hence, we shall focus on the context design in \eqref{lb_zt}.

We now calibrate the hard family for a given variation budget upper bound
$V$.  Indeed,  we can set $s^*=s\wedge V^{2/3}T^{1/3}d^{-1/3}$ and choose a sufficiently small signal
$\Delta^2\asymp\sqrt{s_\star d/T}$.  Then the constructed family has at most
$s$ stationary segments and satisfies
\[
    V_T\lesssim s^*\Delta^2
    \asymp s^{*3/2}\sqrt{d/T}
    \le V .
\]
If $s^*=1$ because $(V^2T/d)^{1/3}<1$, we instead use the stationary
subclass, for which $V_T=0\le V$ and the stationary lower bound
$\sqrt{dT}$ dominates $d^{1/3}V^{1/3}T^{2/3}$. 
Under the sample-size condition $T/(d s^*)\gtrsim\log(dT)$ required in Step 2,
Step 6 gives
\[
\sup_{\omega\in\Omega}\mathbb{E}_{\omega}^{\pi}R_T^\pi(\omega)
\ge cT\Delta^2
\asymp
c\sqrt{d s^*T}
\ge
c\left\{\sqrt{sdT}\wedge d^{1/3}V^{1/3}T^{2/3}\right\}.
\]

\section{Technical details in Section \ref{subsec:mcpdp-detection}}\label{sec:proof1}
Define $\Lambda =\{z^\top\alpha-(z^\top\beta)p: z\in\mathcal Z,\ p\in[l,u],\ \theta=(\alpha^\top,\beta^\top)^\top\in\Theta\}.$ Then, under Assumptions \ref{ass_context}-\ref{ass_boundedness}, Assumption \ref{ass_glm}(b) implies that there exist absolute constants 
$$
M_{\psi 2}=\sup_{\lambda\in\Lambda}\psi''(\lambda)\vee 1,\quad m_{\psi 2}=\inf_{\lambda\in\Lambda}\psi''(\lambda)\wedge 1.$$

\subsection{Proof of Proposition \ref{prop:prediction}}
We give a more technical version of Proposition \ref{prop:prediction} as follows. 

\begin{thmbis}{prop:prediction}\label{prop:prediction2}
There exist absolute constants $c_1,c_2, m_{\psi 2}, M_{\psi 2}>0$ that only depend on Assumptions \ref{ass_context}-\ref{ass_boundedness}, such that, for any fixed price exploration interval $\cJ$ with $|\cJ|\ge 32c_z^2(\underline{\kappa}\lambda_z)^{-1}\log(dT)$, the following results hold with probability larger than $1-3/(dT^4)$:  
\begin{itemize}
    \item[(i)]  for any $\theta^*\in\Theta$,  $
   m_{\psi2}\|\widehat\theta_\cJ-\theta^*\|^2_{H_{\cJ}}/2\leq  \mathcal{L}_{\cJ}(\theta^*)-  \mathcal{L}_{\cJ}(\widehat{\theta}_{\cJ})\leq  ({2m_{\psi2}})^{-1}\|\nabla \mathcal{L}_{\cJ}(\theta^*)\|^2_{H_{\cJ}^{-1}};$
   \item[(ii)]  for any $\theta_1^*,\theta_2^*\in\Theta$,  $\|\nabla \mathcal{L}_{\cJ}(\theta_1^*)
-
\nabla \mathcal{L}_{\cJ}(\theta_2^*)\|_{H_{\cJ}^{-1}}^2
\le
M_{\psi2}^2\|\theta_1^*-\theta_2^*\|_{H_{\cJ}}^2;
$ 
\item[(iii)]  moreover, $\|\nabla \mathcal{L}_{\cJ}(\theta^*)\|_{H_{\cJ}^{-1}}^2
    \le
    c_1d\log(dT)+c_2|\cJ|V_{\cJ},$
where $\theta^*$ can be any $\theta_s$ with $s\in\cJ$, 
or the pseudo-true value $ \theta^*_{\cJ}= \argmin_{\theta\in\Theta}\mathbb E \mathcal{L}_{\cJ}(\theta)$.
\end{itemize}
Here, $H_{\cJ}=\sum_{t\in\cJ}x_tx_t^{\top}$ is the design matrix of $\mathcal{L}_{\cJ}(\theta)$.
\end{thmbis}

\noindent\textit{Proof}
\noindent(i).  First, \Cref{lem:mcpdp-tech-hessian}(ii), and
\Cref{lem:greedy-centered-information} give, with probability at least
$1-(dT^4)^{-1}$,
\begin{equation}
\frac{\underline\kappa}{2}|\cJ|\widetilde\Sigma_z
\preceq H_{\cJ}\preceq
\frac{3\overline\kappa}{2}|\cJ|\widetilde\Sigma_z.
\label{eq:mcpdp-design-comparison}
\end{equation} For an interval $\cJ$, define the GLM Hessian matrix
\[
    \widetilde H_{\cJ}(\theta)
    =
    \sum_{t\in\cJ}\psi''(x_t^\top\theta)x_tx_t^\top .
\]
Let $\Delta=\widehat\theta_{\cJ}-\theta^*$.  By a second-order
Taylor expansion,
\begin{equation}
     \mathcal L_{\cJ}(\widehat\theta_{\cJ})
     =
     \mathcal L_{\cJ}(\theta^*)
     +
     \nabla\mathcal L_{\cJ}(\theta^*)^\top\Delta
     +
     \frac12
     \Delta^\top
     \widetilde H_{\cJ}(\widetilde\theta_{\cJ})
     \Delta ,
     \label{eq:mcpdp-glm-taylor}
\end{equation}
where $\widetilde\theta_{\cJ}$ lies on the line segment between
$\theta^*$ and $\widehat\theta_{\cJ}$.  By \Cref{ass_glm}, we have that 
$m_{\psi2}H_{\cJ}\preceq
\widetilde H_{\cJ}(\theta)\preceq
M_{\psi2}H_{\cJ}$ uniformly over $\theta\in\Theta$. Then, the basic
inequality $ab-a^2/2\le b^2/2$ gives
\[
\mathcal L_{\cJ}(\theta^*)
-
\mathcal L_{\cJ}(\widehat\theta_{\cJ})
\le
\frac{1}{2}
\nabla\mathcal L_{\cJ}(\theta^*)^\top
\widetilde H_{\cJ}^{-1}(\widetilde\theta_{\cJ})
\nabla\mathcal L_{\cJ}(\theta^*)
\le
\frac{1}{2m_{\psi2}}
\|\nabla\mathcal L_{\cJ}(\theta^*)\|_{H_{\cJ}^{-1}}^2.
\]
On the other hand, using the first-order condition for
$\widehat\theta_{\cJ}$,
\[
    \nabla \mathcal L_{\cJ}(\widehat\theta_{\cJ})^\top
    (\theta^*-\widehat\theta_{\cJ})\ge0 ,
\]
another Taylor expansion around $\widehat\theta_{\cJ}$ gives
\[
    \mathcal L_{\cJ}(\theta^*)
    -
    \mathcal L_{\cJ}(\widehat\theta_{\cJ})
    \ge
    \frac12
    \Delta^\top
    \widetilde H_{\cJ}(\widetilde\theta_{\cJ}')
    \Delta
    \ge
    \frac{m_{\psi2}}{2}\|\Delta\|_{H_{\cJ}}^2 .
\]
This proves (i).

\noindent(ii). By the mean value theorem,
\[
\nabla\mathcal L_{\cJ}(\theta_1^*)
-
\nabla\mathcal L_{\cJ}(\theta_2^*)
=
\sum_{t\in\cJ}
\psi''(\zeta_t)x_tx_t^\top(\theta_1^*-\theta_2^*)
=:
\widetilde H_{\cJ}(\theta_1^*,\theta_2^*)
(\theta_1^*-\theta_2^*),
\]
where $\zeta_t$ lies between $x_t^\top\theta_1^*$ and
$x_t^\top\theta_2^*$.  Since
$\widetilde H_{\cJ}(\theta_1^*,\theta_2^*)\preceq
M_{\psi2}H_{\cJ}$,
\[
\begin{split}
&\|\nabla\mathcal L_{\cJ}(\theta_1^*)
-
\nabla\mathcal L_{\cJ}(\theta_2^*)\|_{H_{\cJ}^{-1}}^2=
(\theta_1^*-\theta_2^*)^\top
\widetilde H_{\cJ}(\theta_1^*,\theta_2^*)
H_{\cJ}^{-1}
\widetilde H_{\cJ}(\theta_1^*,\theta_2^*)
(\theta_1^*-\theta_2^*) \le
M_{\psi2}^2
\|\theta_1^*-\theta_2^*\|_{H_{\cJ}}^2 .
\end{split}
\]

 \noindent(iii).  Recall that
\[
    \epsilon_t=y_t-\psi'(x_t^\top\theta_t).
\]
For $\Delta_t=\theta_t-\theta^*$, write
\[
\begin{split}
\nabla\mathcal L_{\cJ}(\theta^*)
&=
\sum_{t\in\cJ}
\{\psi'(x_t^\top\theta^*)-\psi'(x_t^\top\theta_t)\}x_t
-
\sum_{t\in\cJ}\epsilon_tx_t=: X_{\cJ}^\top b_{\cJ}
-
X_{\cJ}^\top\mathcal E_{\cJ}.
\end{split}
\]
Therefore,
\[
\|\nabla\mathcal L_{\cJ}(\theta^*)\|_{H_{\cJ}^{-1}}^2
\le
2b_{\cJ}^\top X_{\cJ}H_{\cJ}^{-1}
X_{\cJ}^\top b_{\cJ}
+
2\mathcal E_{\cJ}^\top X_{\cJ}H_{\cJ}^{-1}
X_{\cJ}^\top\mathcal E_{\cJ}
=:T_1+T_2 .
\]

\noindent\textbf{[Bound for $T_1$]}
By the mean value theorem,
\[
\psi'(x_t^\top\theta^*)-\psi'(x_t^\top\theta_t)
=
\psi''(\xi_t)x_t^\top(\theta^*-\theta_t),
\]
where $\xi_t$ lies between $x_t^\top\theta^*$ and
$x_t^\top\theta_t$.  Thus
\[
    T_1
    \le
    2\|b_{\cJ}\|^2
    \le
    2M_{\psi2}^2
    \sum_{t\in\cJ}(x_t^\top\Delta_t)^2 .
\]
For $W_t=(x_t^\top\Delta_t)^2$, since $|p_t|\le u$,
\[
W_t
=
\{z_t^\top(\Delta_{\alpha,t}-p_t\Delta_{\beta,t})\}^2\leq z_t^{\top}\Sigma_z^{-1}z_t\lambda_{\max}\begin{pmatrix}
1 & -p_t\\
-p_t & p_t^2\end{pmatrix}\|\Delta_t\|^2_{\widetilde{\Sigma}_z}
\le
z_t^\top\Sigma_z^{-1}z_t(1+u^2)
\|\Delta_t\|_{\widetilde\Sigma_z}^2 .
\]
Hence
\[
\mathbb EW_t\le (1+u^2)\|\Delta_t\|_{\widetilde\Sigma_z}^2,
\qquad
0\le W_t\le
\lambda_z^{-1}c_z^2(1+u^2)
\|\Delta_t\|_{\widetilde\Sigma_z}^2 .
\]
By Bernstein's inequality, with probability at least $1-(dT^4)^{-1}$,
\begin{align}
\sum_{t\in\cJ}W_t
&\le
\sum_{t\in\cJ}(1+u^2)\|\Delta_t\|_{\widetilde\Sigma_z}^2
+
(1+u^2)c_z
\max_{t\in\cJ}\|\Delta_t\|_{\widetilde\Sigma_z}^2
\left\{
\sqrt{2\lambda_z^{-1}\log(dT^4)|\cJ|}
+\lambda_z^{-1}c_z\log(dT^4)/3
\right\}.
\label{eq:mcpdp-bernstein}
\end{align}
If $\theta^*=\theta_s$ for some $s\in\cJ$, then
$\max_{t\in\cJ}\|\theta_t-\theta_s\|_{\widetilde\Sigma_z}^2
\le V_{\cJ}$ by definition of the variation budget.  Thus
\[
    \sum_{t\in\cJ}W_t
    \le
    2(1+u^2)|\cJ|V_{\cJ},
\]
provided $|\cJ|\ge 32\lambda_z^{-1}c_z^2\log(dT)$.

It remains to consider $\theta^*=\theta_{\cJ}^*$, the
pseudo-true value.  Fix any $s\in\cJ$.  Since $\theta_{\cJ}^*$ minimizes
$\mathbb E\mathcal L_{\cJ}(\theta)$ over the closed convex set $\Theta$,
the first-order optimality condition gives
\[
\nabla\mathbb E\mathcal L_{\cJ}(\theta_{\cJ}^*)^\top
(\theta_s-\theta_{\cJ}^*)\geq 0.
\]
Note the mean-value theorem gives
\[
\mathbb E\nabla\mathcal L_{\cJ}(\theta_s)
-
\mathbb E\nabla\mathcal L_{\cJ}(\theta_{\cJ}^*)
=
\mathbb E\widetilde H_{\cJ}(\widetilde\theta)
(\theta_s-\theta_{\cJ}^*)
\]
for some $\widetilde\theta$ between $\theta_s$ and
$\theta_{\cJ}^*$.  By 
\Cref{lem:greedy-centered-information}, there exist constants
$0<\underline\kappa\le\overline\kappa<\infty$ such that
\[
    \underline\kappa\widetilde\Sigma_z
    \preceq
    \mathbb E(x_tx_t^\top)
    \preceq
    \overline\kappa\widetilde\Sigma_z
\]
for every price-experiment interval. Consequently,
\begin{align}
(\theta_s-\theta_{\cJ}^*)^\top
\mathbb E\nabla\mathcal L_{\cJ}(\theta_s)
&=
(\theta_s-\theta_{\cJ}^*)^\top
\left\{
\nabla\mathbb E\mathcal L_{\cJ}(\theta_s)
-
\nabla\mathbb E\mathcal L_{\cJ}(\theta_{\cJ}^*)
\right\}
+
(\theta_s-\theta_{\cJ}^*)^\top
\nabla\mathbb E\mathcal L_{\cJ}(\theta_{\cJ}^*)
\nonumber\\
&\geq
m_{\psi2}\underline\kappa|\cJ|
\|(\theta_s-\theta_{\cJ}^*)\|_{\widetilde\Sigma_z}^2.
\label{eq:mcpdp-pseudotrue-lb}
\end{align}
On the other hand,
\[
\mathbb E\nabla\mathcal L_{\cJ}(\theta_s)
=
\sum_{t\in\cJ}
\mathbb E[
\psi''(\xi_t)x_tx_t^\top(\theta_s-\theta_t)
],
\]
and therefore
\begin{equation}
  \|\mathbb E\nabla\mathcal L_{\cJ}(\theta_s)\|_{\widetilde\Sigma_z^{-1}}
  \le
  M_{\psi2}\overline\kappa
  \sum_{t\in\cJ}
  \|\theta_s-\theta_t\|_{\widetilde\Sigma_z}.
  \label{eq:mcpdp-pseudotrue-ub}
\end{equation}
Combining \eqref{eq:mcpdp-pseudotrue-lb} and
\eqref{eq:mcpdp-pseudotrue-ub}, we obtain
\[
\|\theta_s-\theta_{\cJ}^*\|_{\widetilde\Sigma_z}^2
\le
\frac{M_{\psi2}^2\overline\kappa^2}
{m_{\psi2}^2\underline\kappa^2}
V_{\cJ}.
\]
Plugging this into \eqref{eq:mcpdp-bernstein} gives
\[
\sum_{t\in\cJ}(x_t^\top\Delta_t)^2
\le
2(1+u^2)
\frac{M_{\psi2}^2\overline\kappa^2}
{m_{\psi2}^2\underline\kappa^2}
|\cJ|V_{\cJ},
\]
under the same minimum length condition.  

\noindent\textbf{[Bound for $T_2$]}
By the sub-Gaussian concentration inequality for quadratic forms in
\cite{hsu2012tail}, with probability at least $1-(dT^4)^{-1}$,
\[
    T_2
    \le
    2\sigma^2
    \{2d+2\sqrt{2d\log(dT^4)}+2\log(dT^4)\}
    \le
    14\sigma^2d\log(dT^4).
\]
Combining the bounds for $T_1$ and $T_2$ finishes the proof with $$c_1=56\sigma^2,\quad c_2=4(1+u^2)
\frac{M_{\psi2}^2\overline\kappa^2}
{m_{\psi2}^2\underline\kappa^2}.$$
\qed

Here, we note that by the matrix Chernoff bound, for any price experiment sets $\cJ$ with pseudo Hessian matrix $H_{\cJ}$, we have with probability larger than $1-(dT^4)^{-1}$, $\lambda_{\min}(H_{\cJ})\geq \lambda_{z}\underline{\kappa}|\cJ|/2$, if $|\cJ|\geq 32(\lambda_z\underline{\kappa})^{-1}c_z^2(1+u^2)\log(dT)$. Note such minimum length condition can be ensured for sufficiently large $c_L$ as in \eqref{eq:mcpdp-minimum-lengthJ}. Then by a union bound argument, we can ensure that all the pseudo-Hessian matrices of the price experiment sets  are invertible.

\subsection{Proof of Corollary \ref{prop:mcpdp-lrt}}
For the upper bound, by Cauchy-Schwarz inequality, we have that 
\begin{align*}
\|\nabla \mathcal L_{\cJ}(\widehat{\theta}_{i,j-1})\|_{H_{\cJ}^{-1}}^2\leq&  2\|\nabla \mathcal{L}_{\cJ}(\theta^*)\|_{H_{\cJ}^{-1}}^2+2\|\nabla \mathcal{L}_{\cJ}(\widehat{\theta}_{i,j-1})-\nabla \mathcal{L}_{\cJ}(\theta^*)\|_{H_{\cJ}^{-1}}^2\\\lesssim& d\log(dT)+|\cJ|V_{\cJ}+\|\theta^*-\widehat{\theta}_{i,j-1}\|_{H_{\cJ}}^2,  
\end{align*}
where the last inequality holds by \Cref{prop:prediction2}(iii) for the first term and (ii) for the second term.
Note that $\Lambda_{\cJ}(\widehat{\theta}_{i,j-1})\leq (2m_{\psi2})^{-1}\|\nabla \mathcal L_{\cJ}(\widehat{\theta}_{i,j-1})\|_{H_{\cJ}^{-1}}^2$, together  with \eqref{eq:mcpdp-design-comparison}, this proves \eqref{eq:lrt-upper-calibration}.

For the lower bound, applying \Cref{prop:prediction2}(i) with
$\theta^*=\widehat{\theta}_{i,j-1}$ gives
\[
    \Lambda_{\cJ}(\widehat{\theta}_{i,j-1})
    \gtrsim
    \|\widehat{\theta}_{i,j-1} -\widehat\theta_{\cJ}\|_{H_{\cJ}}^2 .
\]
Using elementary inequality, $    \|a-b\|^2\ge \frac12\|a\|^2-\|b\|^2$ with
$a=\widehat{\theta}_{i,j-1} -\theta^*$ and
$b=\widehat\theta_{\cJ}-\theta^*$, we have
\[
    \|\widehat{\theta}_{i,j-1} -\widehat\theta_{\cJ}\|_{H_{\cJ}}^2
    \ge
    \frac12
    \|\widehat{\theta}_{i,j-1} -\theta^*\|_{H_{\cJ}}^2
    -
    \|\widehat\theta_{\cJ}-\theta^*\|_{H_{\cJ}}^2 .
\]
It remains to bound the second term.  By \Cref{prop:prediction2}(i),
$
    \|\widehat\theta_{\cJ}-\theta^*\|_{H_{\cJ}}^2
    \lesssim
    \|\nabla \mathcal{L}_{\cJ}(\theta^*)\|_{H_{\cJ}^{-1}}^2 .
$
Then \Cref{prop:prediction2}(iii) yields
\[
    \|\widehat\theta_{\cJ}-\theta^*\|_{H_{\cJ}}^2
    \lesssim\{d\log(dT)+|\cJ|V_{\cJ}\}.
\]
Combining the last three displays and \eqref{eq:mcpdp-design-comparison}
proves \eqref{eq:lrt-lower-calibration}. \qed

\section{Proof of Theorem \ref{thm:mcpdp-upper}.}\label{sec:proof2}

  



\subsection{Auxiliary lemmas}
\textbf{{Proof of Lemma \ref{lem:greedy-centered-information}.}}
\begin{proof}
Condition on the history and on the current context $z_t$, we write $\overline{p}_t=\mathrm{clip}_{[l+\rho,u-\rho]}p^*(z_t;\theta).$
Since $\xi$ is an independent Rademacher
random variable,
\[
    \mathbb E[p_t\mid z_t,\theta]=\overline{p}_t,
    \qquad
    \mathbb E[p_t^2\mid z,\theta]=\overline{p}_t^2+\rho^2 .
\]
Therefore,
\[
    \mathbb E[x_tx_t^\top\mid z_t,\theta]
    =
    M(\overline{p}_t)\otimes z_tz_t^\top,
    \qquad
    M(\overline{p}_t)=
    \begin{pmatrix}
        1 & -\overline{p}_t\\
        -\overline{p}_t & \overline{p}_t^2+\rho^2
    \end{pmatrix}.
\]
Because $\overline{p}_t\in[l+\rho,u-\rho]$, we have $|\overline{p}_t|\le u$.  Moreover,
\[
    \det M(\overline{p}_t)=\rho^2,
    \qquad
    \operatorname{tr} [M(\overline{p}_t)]\le 1+u^2+\rho^2 .
\]
Hence
\[
    \lambda_{\min}(M(\overline{p}_t))
    \ge
    \frac{\rho^2}{1+u^2+\rho^2},
    \qquad
    \lambda_{\max}(M(\overline{p}_t))
    \le
    1+u^2+\rho^2 .
\]
Taking expectation over $z_t$ gives the desired result. 
   
\end{proof}

The following lemma controls the Hessian matrices in the exploration data.
\begin{lemma}
\label{lem:mcpdp-tech-hessian}
Let $\{x_t^{(1)}\}_{t=1}^{J_1}$ and
$\{x_t^{(2)}\}_{t=1}^{J_2}$ be two independent sequences of i.i.d. random
vectors in $\mathbb R^{2d}$.  Assume that for some constants
$0<\underline\kappa\le\overline\kappa<\infty$,
\[
\underline\kappa\widetilde\Sigma_z
\preceq
\mathbb E[x_t^{(k)}(x_t^{(k)})^\top]
\preceq
\overline\kappa\widetilde\Sigma_z,
\qquad k=1,2,
\]
and assume further that $\|x_t^{(k)}\|_2\le u_x$ almost surely.  Define
\[
    H_1=\sum_{t=1}^{J_1}x_t^{(1)}(x_t^{(1)})^\top,
    \qquad
    H_2=\sum_{t=1}^{J_2}x_t^{(2)}(x_t^{(2)})^\top.
\]
If
$
    J_1,J_2\ge
    96{u_x^2}\log(dT)/({\underline\kappa\lambda_z}),
$
then with probability at least $1-(dT^4)^{-1}$, we have 
\begin{itemize}
    \item[(i)] $
    \lambda_{\max}(H_1^{-1/2}H_2H_1^{-1/2})
    \le
    3J_2\overline\kappa/(J_1\underline\kappa),
$
\item[(ii)]
$
    \lambda_{\min}(\widetilde\Sigma_z^{-1/2}H_1\widetilde\Sigma_z^{-1/2})
    \ge
    J_1\underline\kappa/2,
$ $
    \lambda_{\max}(\widetilde\Sigma_z^{-1/2}H_1\widetilde\Sigma_z^{-1/2})
    \le
    3J_1\overline\kappa/2,$ $
    \lambda_{\max}(H_1^{-1/2}\widetilde\Sigma_zH_1^{-1/2})
    \le
    2/(J_1\underline\kappa).
$
\end{itemize}

\end{lemma}

\begin{proof}
Define the whitened random vectors
\[
    \widetilde x_t^{(k)}=\widetilde\Sigma_z^{-1/2}x_t^{(k)},\qquad k=1,2,
\]
and the corresponding whitened sample covariance matrices
\[
\widetilde H_1
=
\widetilde\Sigma_z^{-1/2}H_1\widetilde\Sigma_z^{-1/2}
=
\sum_{t=1}^{J_1}\widetilde x_t^{(1)}(\widetilde x_t^{(1)})^\top,
\]
\[
\widetilde H_2
=
\widetilde\Sigma_z^{-1/2}H_2\widetilde\Sigma_z^{-1/2}
=
\sum_{t=1}^{J_2}\widetilde x_t^{(2)}(\widetilde x_t^{(2)})^\top .
\]
Since $\lambda_{\min}(\widetilde\Sigma_z)=\lambda_z$,
\[
\|\widetilde x_t^{(k)}\|_2^2
=
(x_t^{(k)})^\top\widetilde\Sigma_z^{-1}x_t^{(k)}
\le
\lambda_z^{-1}\|x_t^{(k)}\|_2^2
\le
\frac{u_x^2}{\lambda_z}
=:\widetilde R .
\]
Moreover,
\[
\underline\kappa J_1I_{2d}\preceq\mathbb E\widetilde H_1
\preceq\overline\kappa J_1I_{2d},
\qquad
\underline\kappa J_2I_{2d}\preceq\mathbb E\widetilde H_2
\preceq\overline\kappa J_2I_{2d}.
\]
By the matrix Chernoff bound (see e.g. Theorem 5.1 in \cite{tropp2015introduction}),
\[
\mathbb P\!\left(
\lambda_{\min}(\widetilde H_1)\le\frac12\underline\kappa J_1
\right)
\le
2d\exp\!\left(-\frac{\underline\kappa J_1}{8\widetilde R}\right),
\]
and similarly,
\[
\mathbb P\!\left(
\lambda_{\max}(\widetilde H_1)\ge\frac32\overline\kappa J_1
\right)
\le
2d\exp\!\left(-\frac{\underline\kappa J_1}{12\widetilde R}\right),
\]
\[
\mathbb P\!\left(
\lambda_{\max}(\widetilde H_2)\ge\frac32\overline\kappa J_2
\right)
\le
2d\exp\!\left(-\frac{\underline\kappa J_2}{12\widetilde R}\right).
\]
Then, for $T\geq 2,$ if $
    J_1,J_2\ge 96\widetilde{R}\log(dT)/\underline{\kappa}
    \geq 12u_x^2\log(6d^2T^4)/({\underline\kappa\lambda_z})
$, a union bound argument implies that with probability $1-(dT^4)^{-1}$,
\[
\lambda_{\min}(\widetilde H_1)\ge\frac12\underline\kappa J_1,
\quad
\lambda_{\max}(\widetilde H_1)\le\frac32\overline\kappa J_1,
\quad
\lambda_{\max}(\widetilde H_2)\le\frac32\overline\kappa J_2.
\]
Then
\[
\lambda_{\max}(H_1^{-1/2}\widetilde\Sigma_zH_1^{-1/2})
=
\lambda_{\max}(\widetilde H_1^{-1})
\le
\frac{2}{\underline\kappa J_1}.
\]
Finally, note that $H_1^{-1}H_2$ is similar to
$\widetilde H_1^{-1}\widetilde H_2$, and hence
\[
\lambda_{\max}(H_1^{-1/2}H_2H_1^{-1/2})
\le
\frac{\lambda_{\max}(\widetilde H_2)}
{\lambda_{\min}(\widetilde H_1)}
\le
3\frac{\overline\kappa}{\underline\kappa}\frac{J_2}{J_1}.
\]
\end{proof}

\subsection{Regret Decomposition and the Good Event}\label{subsec:regret_decompose_supplement}
\label{app:regret-decomposition}
First, by the Berge's maximum theorem \citep{berge1957two} and Proposition S.5.2 in \cite{zhao2026contextual}, the uniqueness of optimal price as in \Cref{ass_boundedness} ensures that there exists a bivariate Lipschitz continuous function $\varphi:[c_l,c_u]^2\mapsto(l,u)$
such that for every $z\in\mathcal Z$ and
$\theta=(\alpha^\top,\beta^\top)^\top\in\Theta$, the optimal price can be written as $p^*(z,\theta)
    =
    \varphi(z^\top\alpha,z^\top\beta).$ Moreover, there exists an absolute constant $C_{\varphi}>0$,
\begin{equation}\label{eq:lip}
    |\varphi(a,b)-\varphi(a',b')|
    \le
    C_{\varphi}\{|a-a'|+|b-b'|\}.
\end{equation}
Hence using \eqref{eq:lip}, for any
$\theta=(\alpha^\top,\beta^\top)^\top\in\Theta$,
\begin{equation}
    |p^*(z_t;\theta)-p^*(z_t,\theta_t)|^2
    \le
    C_\psi^2
    (\theta_t-\theta)^\top
    (I_2\otimes z_tz_t^\top)
    (\theta_t-\theta).
    \label{eq:mcpdp-price-lipschitz}
\end{equation}
for a constant $C_\psi>0$.  Moreover, since the optimal price is an interior maximizer by \Cref{ass_boundedness} and the second derivative $\partial^2_{p}r(p,z,\theta)$ is uniformly bounded on the compact domain $p\in[l,u]$, $z^\top\alpha,z^\top\beta\in[c_l,c_u]$, a standard Taylor expansion gives
\[
    r(p_t^*,z_t;\theta_t)-r(p^*(z_t;\theta),z_t;\theta_t)
    \le
    C_{r}C_{\psi}^2
    (\theta_t-\theta)^\top
    (I_2\otimes z_tz_t^\top)
    (\theta_t-\theta)
\]
for an absolute constant $C_r>0$.  

Note that \Cref{ass_context} and \ref{ass_boundedness} imply that the instantaneous reward is finite, we can write for some constant $c_0$, and $C_{\rm reg}=C_rC_\psi^2$,
\begin{equation}
    r(p_t^*,z_t;\theta_t)-r(p^*(z_t,\theta),z_t;\theta_t)
    \le
    \left\{
    C_{\rm reg}
    (\theta_t-\theta)^\top
    (I_2\otimes z_tz_t^\top)
    (\theta_t-\theta)
    \right\}\wedge c_0 .
    \label{eq:mcpdp-regret-to-prediction}
\end{equation}

For each realized block $\mathcal{B}_{i,j}$, we decompose it into $\mathcal{B}_{i,j}=\cJ_{i,j}\cup\mathcal{G}_{i,j}\cup \mathcal{E}_{i,j}$, where $\cJ_{i,j}=\bigcup_{(s,m)\in\mathcal{S}_{i,j}}\cJ_{i,j,(s,m)}$ is the set of price exploration triggered in multiscale price explorations,  $\mathcal{G}_{i,j}$ is the set of price exploitation, and recall that $\mathcal{E}_{i,j}$ is the mandatory block-end
exploration set.  Denote the total epoch number as $E$, and block number as $B_i$ in the $i$th epoch, then regret in \eqref{eq:regret_decompose_context_LRT} can be rewritten as 
\begin{align}
    R_T:=&\left(\sum_{i=1}^E\sum_{t\in \mathcal{B}_{i,0}}+\sum_{i=1}^E\sum_{j=1}^{B_i}\sum_{t\in \cJ_{i,j}\cup\mathcal{E}_{i,j}}+\sum_{i=1}^E\sum_{j=1}^{B_i}\sum_{t\in \mathcal{G}_{i,j}}\right)\left[r(p^*_t, z_t, \theta_t)-r(\widehat p_t,z_t,  \theta_t)\right]\\
    &\le
    c_0EL
    +\sum_{i=1}^E\sum_{j=1}^{B_i}
    \sum_{t\in\cJ_{i,j}\cup\mathcal E_{i,j}} c_0
    \notag\\
    &\quad+
    \sum_{i=1}^E\sum_{j=1}^{B_i}
    \sum_{t\in\mathcal B_{i,j}\cap[\tau_i+1,\tau_{i+1}]}
    \left\{
    C_{\rm reg}
    (\theta_t-\widehat\theta_{i,j-1})^\top
    (I_2\otimes z_tz_t^\top)
    (\theta_t-\widehat\theta_{i,j-1})
    \right\}\wedge c_0
    \notag\\
    &:=R_{T,1}+R_{T,2}+R_{T,3}.
    \label{eq:mcpdp-regret-decompose}
\end{align}

Define 
$$
\mathcal{A}=\{\text{\Cref{prop:prediction2} and \Cref{lem:mcpdp-tech-hessian} hold uniformly over all
price-experiment intervals}\}.
$$
By a union
bound over all possible intervals considered by the algorithm, the event
$\mathcal A$ holds with probability at least $1-6/T$. Note that for this event to hold, we require that the length of the shortest multiscale exploration interval 
\[
    \ell_0=\sqrt{dL2^0\log(dT)}
    =
    \sqrt{c_L}\,d\log(dT),
\]
to be sufficiently large, where recall we set the base block length $L=c_Ld\log(dT)$ in \Cref{alg:mcpdp}. In particular, using
$u_x\le c_z\sqrt{1+u^2}$, it is easy to see that for \Cref{prop:prediction2} and \Cref{lem:mcpdp-tech-hessian} to hold on an interval of length $\ell_0$, we require
\begin{equation}
  \sqrt{c_L}\,d
  \ge
  [32c_z^2\lambda_z^{-1}]
  \vee
  [96(1+u^2)c_z^2(\underline\kappa\lambda_z)^{-1}].
  \label{eq:mcpdp-minimum-lengthJ}
\end{equation}
For $\lambda_z\asymp 1/d$, \eqref{eq:mcpdp-minimum-lengthJ} holds for all sufficiently large absolute constant $c_L$, provided that $T\geq c_Ld\log(dT).$

If $\lambda_z>0$ but $\lambda_z\ll1/d$, we can set the base block length $L=c_Ld(\log(dT))^{1+2\alpha}$, which implies $\ell_0=\sqrt{c_L}d\{\log(dT)\}^{1+\alpha}$. Therefore, for \Cref{prop:prediction2} and \Cref{lem:mcpdp-tech-hessian} to hold on an interval of length $\ell_0$, we require
\begin{align}\label{eq:mcpdp-minimum-lengthJ2}
  \sqrt{c_L}\,d (\log(dT))^{\alpha}
  \ge
  [32c_z^2\lambda_z^{-1}]
  \vee
  [96(1+u^2)c_z^2(\underline\kappa\lambda_z)^{-1}].
\end{align}
It is easy to see that, for any absolute constant $c_L$, \eqref{eq:mcpdp-minimum-lengthJ2} holds for all sufficiently large $T$ such that $T\geq \exp\{c/(d\lambda_z)^{1/\alpha}\}$, where $c$ is an absolute constant. Moreover, it can be seen easily that the modified block length $L=c_Ld(\log(dT))^{1+2\alpha}$ will only introduce an additional logarithmic factor (i.e.\ $(\log(dT))^{2\alpha}$) into the regret upper bound. Therefore, we stick with $L=c_Ld\log(dT)$ for the rest of the proof.

\subsection{Bounding the Number of Restarts}
\label{app:number-of-restarts}

\begin{lemma}
\label{lem:mcpdp-restart-trigger}
Assume event $\mathcal A$ holds.  Then for all $t$ in epoch $i$ such that
\[
    V_{[\tau_i+1,t]}
    \le
    \sqrt{\frac{d}{t-\tau_i}},
\]
no restart is triggered at time $t$.
\end{lemma}

\begin{proof}
We first show that the the likelihood ratio test at the end of the block does not trigger restart.

Let $t=\tau_i+2^jL$ and consider
\[
    \mathcal B_{i,j}
    =
    [\tau_i+2^{j-1}L+1,\tau_i+2^jL].
\]
Define $s_{i,j}=\tau_i+2^{j-1}L+1$ and $s_{i,0}=\tau_i+1$.  Let
$H_{i,j}$ and $H_{i,j-1}$ denote the Hessian matrices corresponding to
$\mathcal E_{i,j}$ and $\mathcal E_{i,j-1}$.

Under event $\mathcal A$, 
\begin{align}\label{ineq:thetaij}
\begin{split}
    m_{\psi2}^2
\|\widehat\theta_{i,j}-\theta_{s_{i,j}}\|_{H_{i,j}}^2
&\stackrel{(a)}{\leq}
\|\nabla\mathcal L_{\mathcal E_{i,j}}(\theta_{s_{i,j}})
\|_{H_{i,j}^{-1}}^2\\
&\stackrel{(b)}{\leq}
c_1d\log(dT)+c_2|\mathcal E_{i,j}|V_{\mathcal B_{i,j}}\\
&\stackrel{(c)}{\leq}
c_1d\log(dT)
+
c_2\sqrt{2d|\mathcal B_{i,j}|\log(dT)}
V_{\mathcal B_{i,j}}\\
&\stackrel{(d)}{\leq}
(c_1+c_2)d\log(dT),
\end{split}
\end{align}
where (a) holds by \Cref{prop:prediction2} (i), (b) holds by \Cref{prop:prediction2} (iii) and $V_{\mathcal E_{i,j}}\leq V_{\mathcal B_{i,j}}$, (c) holds by noting $|\mathcal{B}_{i,j}|=2^{j-1}L$ and $|\mathcal{E}_{i,j}|\leq \sqrt{d\log(dT)2^jL}$, and (d) by the fact that
\[
    V_{\mathcal B_{i,j}}
    \le
    V_{[\tau_i+1,t]}
    \le
    \sqrt{\frac{d}{t-\tau_i}}
    \le
    \sqrt{\frac{d}{2|\mathcal B_{i,j}|}}.
\]
By a similar argument, we have 
\begin{equation}\label{ineq:thetaij1}
  m_{\psi2}^2
\|\widehat\theta_{i,j-1}-\theta_{s_{i,j-1}}\|_{H_{i,j-1}}^2
\le
(c_1+c_2)d\log(dT).  
\end{equation}
Since $\widehat\theta_{i,j}$ minimizes
$\mathcal L_{\mathcal E_{i,j}}(\cdot)$, \Cref{prop:prediction2} (i) gives
\[
2m_{\psi2}
\{\mathcal L_{\mathcal E_{i,j}}(\widehat\theta_{i,j-1})
-
\mathcal L_{\mathcal E_{i,j}}(\widehat\theta_{i,j})\}
\le
\|\nabla\mathcal L_{\mathcal E_{i,j}}(\widehat\theta_{i,j-1})
\|_{H_{i,j}^{-1}}^2 .
\]
Then, using the inequality $(x+y+z)^2\le3x^2+3y^2+3z^2$, we have that for \begin{equation}\label{bound_c}
    c\geq M_{\psi2}^2\overline{\kappa}\vee M_{\psi2}^2(c_1+c_2)\overline{\kappa}/(m_{\psi2}^2\underline{\kappa})
\end{equation}
\begin{align*}
&2m_{\psi2}
\{\mathcal L_{\mathcal E_{i,j}}(\widehat\theta_{i,j-1})
-
\mathcal L_{\mathcal E_{i,j}}(\widehat\theta_{i,j})\}\leq\|\nabla \mathcal L_{\mathcal{E}_{i,j}}(\widehat\theta_{i,j-1}) \|^2_{H_{i,j}^{-1}}\\
\leq& 3\|\nabla \mathcal  L_{\mathcal{E}_{i,j}}( \theta_{s_{i,j}}) \|^2_{H_{i,j}^{-1}}
+3\|\nabla \mathcal  L_{\mathcal{E}_{i,j}}(\theta_{s_{i,j-1}})-\nabla \mathcal L_{\mathcal{E}_{i,j}}(\theta_{s_{i,j}}) \|^2_{H_{i,j}^{-1}}\\
&\quad+3\|\nabla \mathcal L_{\mathcal{E}_{i,j}}(\widehat\theta_{i,j-1})-\nabla \mathcal L_{\mathcal{E}_{i,j}}(\theta_{s_{i,j-1}}) \|^2_{H_{i,j}^{-1}}\\
\stackrel{(a)}{\leq}&
3\|\nabla \mathcal L_{\mathcal E_{i,j}}(\theta_{s_{i,j}})
\|_{H_{i,j}^{-1}}^2+
3M_{\psi2}^2
\|\theta_{s_{i,j}}-\theta_{s_{i,j-1}}\|_{H_{i,j}}^2
+
3M_{\psi2}^2
\|\widehat\theta_{i,j-1}-\theta_{s_{i,j-1}}\|_{H_{i,j}}^2\\
\stackrel{(b)}{\leq}&
3cd\log(dT)
+
3M_{\psi2}^2
\|\theta_{s_{i,j}}-\theta_{s_{i,j-1}}\|_{\widetilde\Sigma_z}^2
\lambda_{\max}(\widetilde\Sigma_z^{-1/2}H_{i,j}\widetilde\Sigma_z^{-1/2})
\\
&\quad+
3M_{\psi2}^2
\|\widehat\theta_{i,j-1}-\theta_{s_{i,j-1}}\|_{H_{i,j-1}}^2
\lambda_{\max}(H_{i,j-1}^{-1/2}H_{i,j}H_{i,j-1}^{-1/2})\\
\stackrel{(c)}{\leq}&
3cd\log(dT)
+
9M_{\psi2}^2\overline\kappa|\mathcal E_{i,j}|
(V_{\mathcal B_{i,j-1}}+V_{\mathcal B_{i,j}})
+
3(c_1+c_2)
\frac{M_{\psi2}^2}{m_{\psi2}^2}
d\log(dT)
\lambda_{\max}(H_{i,j-1}^{-1/2}H_{i,j}H_{i,j-1}^{-1/2})\\
\stackrel{(d)}{\leq}&
21cd\log(dT),
\end{align*}
where (a) follows from 
\Cref{prop:prediction2} (ii), (b) from \Cref{prop:prediction2} (iii), (c) from \Cref{lem:mcpdp-tech-hessian} (ii)
and \eqref{ineq:thetaij1}, and (d) holds by  \Cref{lem:mcpdp-tech-hessian} (i) and that 
$|\mathcal E_{i,j}|/|\mathcal E_{i,j-1}|=2$. This implies that if $c_{\gamma}\geq 21 c/(2m_{\psi 2})$, the mandatory block-end
exploration set passes the test. 

We then work on the testing on multiscale exploration interval. Consider a complete multiscale exploration that happens within a block $\mathcal{B}_{i,j}$ that starts at time $s+1$ with length $\ell_m$ for some  $m\leq j-1$. Thus we have $t=s+\ell_m-1\geq 2^{j-1}L.$  Let  the Hessian matrix on ${\cJ_{i,j,(s,m)}}$ be $H_{i,j,(s,m)}$. Using the same arguments as above,  we have that
\begin{align*}
    &2m_{\psi2}
\{\mathcal L_{\cJ_{i,j,(s,m)}}(\widehat\theta_{i,j-1})
-
\mathcal L_{\cJ_{i,j,(s,m)}}(\widehat\theta_{i,j,(s,m)})\} \\\leq &
3\|\nabla \mathcal  L_{\cJ_{i,j,(s,m)}}( \theta_{s}) \|^2_{H^{-1}_{i,j,(s,m)}}+3\|\nabla \mathcal  L_{\cJ_{i,j,(s,m)}}(\theta_{s_{i,j-1}})-\nabla \mathcal L_{\cJ_{i,j,(s,m)}}(\theta_{s}) \|^2_{H^{-1}_{i,j,(s,m)}}\\&+3\|\nabla \mathcal L_{\cJ_{i,j,(s,m)}}(\widehat\theta_{i,j-1})-\nabla \mathcal L_{\cJ_{i,j,(s,m)}}(\theta_{s_{i,j-1}}) \|^2_{H^{-1}_{i,j,(s,m)}} 
\\\leq & 3cd\log(dT)+  
3M_{\psi2}^2
\|\theta_{s}-\theta_{s_{i,j-1}}\|_{\widetilde\Sigma_z}^2
\lambda_{\max}(\widetilde\Sigma_z^{-1/2}H_{i,j,(s,m)}\widetilde\Sigma_z^{-1/2})\\
&\quad+
3M_{\psi2}^2
\|\widehat\theta_{i,j-1}-\theta_{s_{i,j-1}}\|_{H_{i,j-1}}^2
\lambda_{\max}(H_{i,j-1}^{-1/2}H_{i,j,(s,m)}H_{i,j-1}^{-1/2})\\
\leq & 3cd\log(dT)+
9M_{\psi2}^2\overline\kappa|\cJ_{i,j,(s,m)}|
(V_{\mathcal B_{i,j-1}}+V_{\cJ_{i,j,(s,m)}})\\
&\quad+
3(c_1+c_2)
\frac{M_{\psi2}^2}{m_{\psi2}^2}
d\log(dT)
\lambda_{\max}(H_{i,j-1}^{-1/2}H_{i,j,(s,m)}H_{i,j-1}^{-1/2})\\
\leq&
3cd\log(dT)+ 9M_{\psi2}^2\sqrt{d2^mL\log(dT)}\sqrt{d/2^{j-1}L}+6(c_1+c_2)
\frac{M_{\psi2}^2}{m_{\psi2}^2}
d\log(dT)\leq 21cd\log(dT),
\end{align*}
where the penultimate inequality holds because $V_{\mathcal B_{i,j-1}}+V_{\cJ_{i,j,(s,m)}}\leq V_{[\tau_i+1,t]}\leq \sqrt{d/(t-\tau_i)}\leq \sqrt{d/(2^{j-1}L)}$, and that $|\cJ_{i,j,(s,m)}|=\ell_m\leq \sqrt{d2^mL\log(dT)}\leq 2^{j-1}L$. 
\end{proof}
\begin{rem}
Note that Algorithm~\ref{alg:mcpdp} sets
$\widehat\theta_{i,j}=\widehat\theta_{\mathcal E_{i,j}\cup \mathcal{J}_{i,j}}$ using all price-experiment observations collected in block $\mathcal B_{i,j}$, namely the mandatory block-end exploration interval $\mathcal E_{i,j}$ together with the randomized multiscale exploration intervals $\mathcal{J}_{i,j}$ that are actually sampled in that block. For clarity, our proof analyzes the
estimator based only on $\mathcal E_{i,j}$.  This simplification does not affect the regret rate: with high probability, the total amount of exploration data in a block is of the same order as $|\mathcal E_{i,j}|$, up to logarithmic factors, and its design matrix satisfies the same concentration bounds.  Consequently, pooling all exploration observations yields the same order of prediction error as the estimator analyzed in the proof, while potentially improving its
finite-sample accuracy.  The mandatory block-end interval is used in the analysis because it guarantees a sufficiently large estimation sample even when no randomized multiscale window is activated.
\end{rem}

We next use \Cref{lem:mcpdp-restart-trigger} to bound the epoch number. 
\begin{corollary}
\label{cor:mcpdp-number-epoch}
Assume that event $\mathcal A$ holds.  The number of epochs satisfies
\[
    E
    \le
    s_T
    \wedge
    \{1+d^{-1/3}T^{1/3}V_T^{2/3}\}.
\]
\end{corollary}

\begin{proof}
On one hand, by \Cref{lem:mcpdp-restart-trigger}, if
$V_{[\tau_i+1,t]}=0$, that is, no change point occurs, then no restart is
triggered.  Hence the number of restarts is bounded above by the number of
change-points, so $E\le s_T$.

On the other hand, denote the epochs by
$\mathcal{I}_1,\ldots,\mathcal{I}_E$.  By \Cref{lem:mcpdp-restart-trigger}, for every completed
epoch that ends with a restart, we must have
\[
    V_{\mathcal{I}_i}
    >
    \sqrt{\frac{d}{|\mathcal{I}_i|}},
\]
because otherwise no restart can be triggered.  Therefore, by H\"older's
inequality,
\begin{align*}
E-1
&=
\sum_{i=1}^{E-1}(|\mathcal{I}_i|/d)^{1/3}(|\mathcal{I}_i|/d)^{-1/3}\\
&\le
d^{-1/3}
\left(\sum_{i=1}^{E-1}|\mathcal{I}_i|\right)^{1/3}
\left(\sum_{i=1}^{E-1}\sqrt{\frac{d}{|\mathcal{I}_i|}}\right)^{2/3}\\
&\le
d^{-1/3}T^{1/3}
\left(\sum_{i=1}^{E-1}V_{\mathcal{I}_i}\right)^{2/3}
\le
d^{-1/3}T^{1/3}V_T^{2/3}.
\end{align*}
This proves the result.
\end{proof}

\subsection{Regret from Exploration Sets}
\label{app:exploration-set-regret}

\subsubsection{Initial exploration after restarts}

Since each epoch has length at least $L$, we also have $E\le T/L$.  Under
event $\mathcal A$,
\begin{equation}
    R_{T,1}
    \le
    c_0EL
    \le
    c_0\sqrt{ETL}
    \le
   c_0\sqrt{c_L\log(dT)}\!\left(
    \sqrt{ds_TT}
    \wedge
    \{\sqrt{dT}+d^{1/3}T^{2/3}V_T^{1/3}\}
    \right),
    \label{eq:mcpdp-bound-RT1}
\end{equation}
where the last inequality follows from \Cref{cor:mcpdp-number-epoch} and
$L=c_Ld\log(dT)$.

\subsubsection{Scheduled multiscale exploration and block-end exploration}

By the design of \Cref{alg:mcpdp}, for each time point $t$, with probability at
most $2^{-(j+m)/2}$, the algorithm enters a multiscale price-experiment
and detection interval with length
$\sqrt{d2^mL\log(dT)}$ for some $m\in\{0,\ldots,j-1\}$. Let
\[
    \mathbb I_{t,m}
    =
    \mathbb I\{\text{the algorithm starts a detection interval at time }t
    \text{ with index }m\}.
\]
Inside $\cI=\mathcal{B}_{i,j}$ with $|\cI|=2^{j-1}L$, by design, the starting time can only happen on an $L-$spaced grid, i.e. $t\in \mathcal{T}_{i,j}=\{\tau_i+2^{j-1}L+1+kL,k=0,1,\cdots,2^{j-1}-1\}$. Therefore, the number of price
experiments $N(\cI)$ is bounded by
\[
    N(\cI)
    \le
    \sum_{m=0}^{j-1}
    \sum_{t\in\mathcal{T}_{i,j}}
    \mathbb I_{t,m}
    \sqrt{d2^mL\log(dT)}.
\]
Conditional on the history before block $\cI$, the indicators
$\{\mathbb I_{t,m}\}$ are independent Bernoulli variables for each fixed
$m$.  By Bernstein's inequality and by a union bound argument for all possible $i$, $j$, and $m$ (with at most $T^3$ combinations), for any $\delta>0$,  we have with probability at least $1-\delta$ such that
$$
\sum_{t\in\mathcal{T}_{i,j}}\mathbb{I}_{t,m}\leq \sum_{t\in\mathcal{T}_{i,j}}\mathbb{E}[\mathbb{I}_{t,m}]+\sum_{t\in\mathcal{T}_{i,j}}\mathbb{E}[\mathbb{I}_{t,m}^2]+\log (T^3/\delta)\leq 2^{-(m+j)/2}\times 2^{j-1}\times 2+\log (T^3/\delta).
$$
Hence, choosing $\delta=T^{-1}$, and noting that $j\leq \log_2T$, we have $$
N(\cI)\leq \log_2T[4\log T+\sqrt{d2^jL\log(dT)}].
$$
The mandatory block-end exploration interval has the same order of length,
namely $\sqrt{d2^jL\log(dT)}$.  Let
$a_0:=\sqrt2/(\sqrt2-1)$.  Since
\[
    \sum_{j=1}^{B_i}\sqrt{2^jL}
    \le
    a_0\sqrt{2^{B_i}L}
    \le
    a_0\sqrt{2(\tau_{i+1}-\tau_i)},
\]
under event $\mathcal A$, with probability at least $1-T^{-1}$,
\begin{align}
R_{T,2}
&\le
6c_0\log_2T
\sum_{i=1}^E\sum_{j=1}^{B_i}
\sqrt{d2^jL\log(dT)}
\notag\\
&\le
6a_0c_0\log_2 T\sqrt{2d\log(dT)}
\sum_{i=1}^E
\sqrt{\tau_{i+1}-\tau_i}
\notag\\
&\le
6a_0c_0\log_2 T\sqrt{2dTE\log(dT)}
\notag\\
&\leq 6\sqrt2a_0c_0\log_2T\sqrt{\log(dT)}
\left(
\sqrt{ds_TT}
\wedge
\{\sqrt{dT}+d^{1/3}V_T^{1/3}T^{2/3}\}
\right),
\label{eq:mcpdp-bound-RT2}
\end{align}
where the third inequality follows from Cauchy--Schwarz and
$\sum_{i=1}^E(\tau_{i+1}-\tau_i)\leq T$, and the last inequality follows from
\Cref{cor:mcpdp-number-epoch}.

\subsection{Exploitation Regret \texorpdfstring{$R_{T,3}$}{R(T,3)}}
\label{app:exploitation-regret}

\subsubsection{Decomposition of \texorpdfstring{$R_{T,3}$}{R(T,3)}}

To analyze $R_{T,3}$, we partition each realized block into subintervals with
limited variation.

\begin{lemma}
\label{lem:mcpdp-partition}
There is a way to partition any interval $\cI$ into
$\cI_1\cup\cdots\cup\cI_\Gamma$ such that
\[
    V_{\cI_k}
    \le
    \alpha_{\cI_k}
    :=
    \sqrt{\frac{d}{|\cI_k|}},
    \qquad k\in[\Gamma],
\]
and
\[
    \Gamma
    \le
    s_{\cI}
    \wedge
    \left\{1+(2/d)^{1/3}|\cI|^{1/3}V_{\cI}^{2/3}\right\}.
\]
\end{lemma}

\begin{proof}
The proof follows the partitioning argument of Lemma 5 in \cite{chen2019new}.
Consider the following procedure.

\begin{algorithm}[H]
\begin{algorithmic}
    \State \textbf{Input}: an interval $\cI=[s,e]$.
    \State \textbf{Initialize}: let $k=1$, $s_1=s$, and $t=s$.
    \While{$t\le e$}
        \If{$V_{[s_k,t]}\le \sqrt{d/(t-s_k+1)}$ and
        $V_{[s_k,t+1]}>\sqrt{d/(t-s_k+2)}$}
            \State Let $e_k\leftarrow t$, $\cI_k\leftarrow[s_k,e_k]$,
            $k\leftarrow k+1$, and $s_k\leftarrow t+1$.
        \EndIf
        \State $t\leftarrow t+1$.
    \EndWhile
    \If{$s_k\le e$}
        \State Let $e_k\leftarrow e$ and $\cI_k\leftarrow[s_k,e_k]$.
    \EndIf
    \caption{Partitioning an interval}
    \label{alg:mcpdp-partition}
\end{algorithmic}
\end{algorithm}

It is clear that this procedure ensures that $V_{\cI_k}\leq \alpha_{\cI_k}$ for all $k.$ It remains to bound $\Gamma$. If $\Gamma>1$, by the procedure and the decomposability (i.e.\ superadditivity) of variation we have that (note that $e_k+1=s_{k+1}$ for $k=1,2,\dots,\Gamma-2$)
    \begin{align*}
        V_{[s,e]}\geq V_{[s_1,e_1+1]}+V_{[s_2,e_2+1]}+\cdots + V_{[s_{\Gamma-1},e_{\Gamma-1}+1]}\geq \sum_{k=1}^{\Gamma-1}\sqrt{\frac{d}{e_k-s_k+2}}=\sum_{k=1}^{\Gamma-1}\sqrt{\frac{d}{|\cI_k|+1}}.
    \end{align*}
    On the other hand, by H\"older's inequality, we have
    \begin{align*}
        \left(\sum_{k=1}^{\Gamma-1}\frac{\sqrt{d}}{\sqrt{|\cI_k|+1}}\right)^{2/3} \left(\sum_{k=1}^{\Gamma-1}\frac{(|\cI_k|+1)}{d}\right)^{1/3}\geq (\Gamma-1).
    \end{align*}
    Combining the two inequalities above, we have
    \begin{align*}
        \Gamma-1\leq \left(\sum_{k=1}^{\Gamma-1}\frac{(|\cI_k|+1)}{d}\right)^{1/3}V_{[s,e]}^{2/3}\leq (2/d)^{1/3}|\cI|^{1/3}V_{[s,e]}^{2/3}.
    \end{align*}
    Moreover, it is easy to see from the condition $V_{[s_k,t+1]}>\sqrt{d/(t-s_k+2) }$ that the procedure creates a new interval only when the parameter changes. Therefore, we have $\Gamma\leq s_{\cI}$, which completes the proof.
\end{proof}

We then proceed to bound $R_{T,3}$.  Denote
\[
    m_t
    =
    \left\{
    C_{\rm reg}
    (\theta_t-\widehat\theta_{i,j-1})^\top
    (I_2\otimes z_tz_t^\top)
    (\theta_t-\widehat\theta_{i,j-1})
    \right\}
    \wedge c_0 .
\]
Given $\widehat\theta_{i,j-1}$, the random variables
$m_t-\mathbb E(m_t)$ are independent, mean-zero, and bounded by $c_0$.
Hoeffding's inequality gives, with probability at least $1-T^{-4}$,
\[
    \sum_{t\in\cI^{*}}\{m_t-\mathbb E(m_t)\}
    \le
    4c_0\sqrt{|\cI^{*}|\log T}.
\]
Thus it suffices to bound
$
    \sum_{t\in\cI^{*}}
    \|\theta_t-\widehat\theta_{i,j-1}\|_{\widetilde\Sigma_z}^2 .
$

\begin{lemma}
\label{lem:mcpdp-bound-on-I}
Let
$\cI^{*}=\mathcal B_{i,j}\cap[\tau_i+1,\tau_{i+1}]\neq\varnothing$ for
some $i,j>0$.  For any $\cI=[s,e]\subseteq\cI^{*}$, define
\[
    \epsilon_{\cI}
    =
    \|\theta_s-\widehat\theta_{i,j-1}\|_{\widetilde\Sigma_z}^2,
    \qquad
    \alpha_{\cI}
    =
    \sqrt{{d}/{|\cI|}}.
\]
Then, for any constant $D_0>0$,
\[
\sum_{t\in\cI}
\|\theta_t-\widehat\theta_{i,j-1}\|_{\widetilde\Sigma_z}^2
\le
2|\cI|V_{\cI}
+
2|\cI|D_0\alpha_{\cI}
+
2|\cI|\epsilon_{\cI}
\mathbb I\{\epsilon_{\cI}>D_0\alpha_{\cI}\}.
\]
\end{lemma}
\begin{proof}
We note that $\widehat\theta_{i,j-1}$ is well defined over the block $\cI^{*}$ as  otherwise the algorithm would have triggered a new epoch instead of entering   $\mathcal{B}_{i,j}.$ Then, 
\begin{align*}
    \sum_{t\in \cI} \|\theta_t-\widehat\theta_{i,j-1}\|^2_{\widetilde{\Sigma}_z}&\leq 2\sum_{t\in \cI} \|\theta_t-\theta_s\|^2_{\widetilde{\Sigma}_z} + 2\sum_{t\in \cI} \|\theta_s-\widehat\theta_{i,j-1}\|^2_{\widetilde{\Sigma}_z} \leq 2|\cI|V_\cI + 2 |\cI| \epsilon_\cI \\
    &\leq 2|\cI|V_\cI + 2|\cI|D_0\alpha_\cI  + 2|\cI|\epsilon_\cI \mathbb I\{\epsilon_\cI>D_0\alpha_\cI\},
\end{align*}
where  the second inequality holds by the definition of variation $V_{\cI}$ on $\cI$.
\end{proof}

Based on \Cref{lem:mcpdp-partition},  we can partition $\cI^{*}$ into $\Gamma$ number of sub-intervals $\bigcup_{k=1}^\Gamma \cI_k$ such that each one has $V_{\cI_k}\leq \alpha_{\cI_k}$. Note that conditional on all history before block $j$, $\cI^{*}$ and $\Gamma$ are random as the algorithm may restart. Same as \cite{chen2019new}, we introduce a fictitious block $\mathcal{B}_{i,j}'$ to handle the difficulty due to this randomness. In particular, given $\tau_i$, define
\begin{align*}
    \mathcal{B}_{i,j}':=[\tau_i+2^{j-1}L+1,\tau_i+2^jL].
\end{align*}

Note that $\mathcal{B}_{i,j}'$ is the \textit{planned} block $j$ while  $\cI^{*}=\mathcal{B}_{i,j}\cap[\tau_i+1,\tau_{i+1}]$ is the realized block $j.$ Based on \Cref{lem:mcpdp-partition}, we can partition $\mathcal{B}_{i,j}'$ into $S$ number of sub-intervals $\bigcup_{k=1}^S \cI_k'$ such that each one has $V_{\cI_k'}\leq \alpha_{\cI_k'}=\sqrt{d/|{\cI_k'}|}$. Importantly, note that conditional on all history before block $j$, the fictitious block $\mathcal{B}_{i,j}'$ and $S$ are determined. Moreover, it is easy to see that (i).\ $\Gamma\leq S$; (ii).\ $\cI_k=\cI_k'=[s_{\cI_k}, e_{\cI_k}]$ for all $k\leq \Gamma-1$, and (iii).\ $s_{\cI_\Gamma}=s_{\cI_\Gamma'}$, $e_{\cI_\Gamma}\leq e_{\cI_\Gamma'}$, where $\cI_\Gamma=[s_{\cI_\Gamma}, e_{\cI_\Gamma}]$ and $\cI_\Gamma'=[s_{\cI_\Gamma'}, e_{\cI_\Gamma'}]$.

Let   $\epsilon_{\cI_\Gamma}:=\|\widehat\theta_{i,j-1}-\theta_{s_{\cI_\Gamma}}\|^2_{\widetilde{\Sigma}_z}$, then by the proof of \Cref{lem:mcpdp-bound-on-I}, for any constant $D_0>0,$   we have that
\begin{align}\label{eq:bound_on_I2_context_LRT}
    \sum_{t\in \cI_\Gamma}\|\theta_t-\widehat\theta_{i,j-1}\|^2_{\widetilde{\Sigma}_z}    \leq 2|\cI_\Gamma|V_{\cI_\Gamma'} + 2|\cI_\Gamma|D_0\alpha_{\cI'_\Gamma}  + 2|\cI_\Gamma|\epsilon_{\cI_\Gamma} \mathbb I\{\epsilon_{\cI_\Gamma}>D_0\alpha_{\cI_\Gamma'}\}.
\end{align}
Recall the partition $\cI^{*}=\bigcup_{k=1}^\Gamma \cI_k$, and denote
$\epsilon_{\cI_k}=\|\widehat{\theta}_{i,j-1}-\theta_{s_{\cI_k}}\|^2_{\widetilde{\Sigma}_z}$ and
$\alpha_{\cI_k}=\sqrt{d/|{\cI_k}|}$ for $k<\Gamma$.  Similarly, write
$\epsilon_{\cI_\Gamma'}=\|\widehat{\theta}_{i,j-1}-\theta_{s_{\cI_\Gamma'}}\|^2_{\widetilde{\Sigma}_z}$.
Then
\begin{align}\label{RT3_decomp_context_LRT}
\begin{split}
&\sum_{t\in\cI^{*}}\|\theta_t-\widehat\theta_{i,j-1}\|^2_{\widetilde{\Sigma}_z}=\sum_{k=1}^\Gamma \sum_{t\in \cI_k}\|\theta_t-\widehat\theta_{i,j-1}\|^2_{\widetilde{\Sigma}_z}\\
\leq&\sum_{k=1}^{\Gamma-1} 2|\cI_k|V_{\cI_k} + 2|\cI_k|D_0\alpha_{\cI_k}  + 2|\cI_k|\epsilon_{\cI_k} \mathbb I\{\epsilon_{\cI_k}>D_0\alpha_{\cI_k}\} \\&+2|\cI_\Gamma|V_{\cI_\Gamma'} + 2|\cI_\Gamma|D_0\alpha_{\cI'_\Gamma} +
 2|\cI_\Gamma|\epsilon_{\cI_\Gamma'} \mathbb I\{\epsilon_{\cI_\Gamma'}>D_0\alpha_{\cI_\Gamma'}\}\\
\leq & \sum_{k=1}^{\Gamma-1} 2|\cI_k|(D_0+1)\alpha_{\cI_k} + 2|\cI_{\Gamma}|(D_0+1)\alpha_{\cI_{\Gamma}'}\\
&+ \sum_{k=1}^{\Gamma-1}2|\cI_k|\epsilon_{\cI_k} \mathbb I\{\epsilon_{\cI_k}>D_0\alpha_{\cI_k}\} + 2|\cI_\Gamma|\epsilon_{\cI_\Gamma'} \mathbb I\{\epsilon_{\cI_\Gamma'}>D_0\alpha_{\cI_\Gamma'}\}\\:=&T_1+T_2,    
\end{split}
\end{align}
where the first inequality follows from applying \Cref{lem:mcpdp-bound-on-I} on each $\cI_k$ and \eqref{eq:bound_on_I2_context_LRT}, and the second inequality follows from \Cref{lem:mcpdp-partition}.

For $T_1$, set $b_j:=2^{j-1}L=|\mathcal B_{i,j}'|$.  We first have
\begin{align}
T_1
&\le
\sum_{k=1}^{\Gamma}
2|\cI_k|(D_0+1)\alpha_{\cI_k}
=
\sum_{k=1}^{\Gamma}
2(D_0+1)\sqrt{d|\cI_k|}
\notag\\
&\le
2(D_0+1)\sqrt{\Gamma d|\cI^{*}|}
\label{eq:mcpdp-bound-T1-basic}
\end{align}
and hence, using $\Gamma\le s_{\cI^{*}}$ and $|\cI^{*}|\le b_j$,
\begin{equation}
T_1
\le
2(D_0+1)\sqrt{d s_{\cI^{*}}b_j}.
\label{eq:mcpdp-bound-T1-switch}
\end{equation}
On the other hand, using
$\Gamma\le 1+(2/d)^{1/3}|\cI^{*}|^{1/3}V_{\cI^{*}}^{2/3}$,
$\sqrt{1+x}\le 1+\sqrt{x}$, and $|\cI^{*}|\le b_j$, we also have
\begin{equation}
T_1
\le
2(D_0+1)\sqrt{db_j}
+
2^{7/6}(D_0+1)d^{1/3}b_j^{2/3}V_{\cI^{*}}^{1/3}.
\label{eq:mcpdp-bound-T1}
\end{equation}
Here the first inequality in \eqref{eq:mcpdp-bound-T1-basic} follows from
$\alpha_{\cI_\Gamma}\geq \alpha_{\cI_\Gamma'}$, and the second follows from
Cauchy--Schwarz and $\sum_{k=1}^\Gamma|\cI_k|=|\cI^{*}|$.

\subsubsection{Charge large prediction error to failure of multiscale detection}
The term $T_2$ measures regret due to a large discrepancy between the previous stale
estimator $\widehat\theta_{i,j-1}$ and the current parameter on the
subintervals of the fictitious block.  We next show that such a discrepancy can be
detected by some multiscale likelihood-contrast test, and the failure of such detection is of low probability.

Recall that $c_{\gamma}$ is the test threshold, $c$ is defined in \eqref{bound_c}, and $c_1,c_2$ are given in the proof of \Cref{prop:prediction2}.

\begin{lemma}\label{lem:detection_interval_context_LRT}
Define $D_0:=c_D\sqrt{\log(dT)},$ for a sufficiently large constant \[
c_D>
\left\{
\frac{4c_1}{m_{\psi2}^2\underline\kappa}
+
\frac{4\sqrt{c_L}\,c_2c_u^2}
     {m_{\psi2}^2\underline\kappa}
+
16c_u^2\sqrt{c_L}
\right\}
\vee
\left\{
\frac{32}{3\underline\kappa}
\left(
6\overline\kappa+3c+\frac{c_\gamma}{m_{\psi2}}
\right)
\right\}.
\] Assume the event $\mathcal{A}$ holds. Let $\cI=[s,e]$ be an interval in the fictitious block $\mathcal{B}_{i,j}'$ with $V_{\cI}\leq \alpha_{\cI}:=\sqrt{d/|\cI|}$. If    $\epsilon_{\cI}:=\|\widehat\theta_{i,j-1}-\theta_s\|^2_{\widetilde{\Sigma}_z}\geq D_0\sqrt{d/|\cI|}=D_0\alpha_\cI,$ then there exists an index $m\in\{0,1,\cdots,j-2\}$ (that depends on $\cI$) such that (a) $D_0\sqrt{d}/\sqrt{2^{m+1}L}\leq \epsilon_\cI\leq D_0\sqrt{d}/\sqrt{2^{m}L}$; (b) $|\cI|\geq 2^mL$; (c) if an active exploration phase with index $m$ starts within the range of $[s,e-2^mL+1]$, then the algorithm restarts when the multiscale exploration phase finishes.
\end{lemma}

\begin{proof}
     For (a), we have with high probability at least $1-1/(dT^4)$ it holds that
    \begin{align*}
        \epsilon_\cI&=\|\widehat\theta_{i,j-1}-\theta_s\|^2_{\widetilde{\Sigma}_z}\leq 2\|\widehat\theta_{i,j-1}-\theta_{s_{i,j-1}}\|^2_{H_{i,j-1}}\lambda_{\max}(H_{i,j-1}^{-1/2}\widetilde{\Sigma}_z H_{i,j-1}^{-1/2}) +2 \|\theta_{s_{i,j-1}}-\theta_s\|^2_{\widetilde{\Sigma}_z}\\
        &\leq 4[c_1d\log(dT) + c_2|\mathcal{E}_{i,j-1}|c_u^2]/(m_{\psi2}^2\underline{\kappa}|\mathcal{E}_{i,j-1}|)+2\|\theta_{s_{i,j-1}}-\theta_s\|^2_{\widetilde{\Sigma}_z}\leq D_0\sqrt{d}/\sqrt{2^0L},
    \end{align*}
    where the second inequality holds by \Cref{prop:prediction2}(i) and \Cref{lem:mcpdp-tech-hessian}(ii), and the  last inequality holds by the fact that $\|\theta_t\|_{\widetilde{\Sigma}_z}^2\leq 2c_u^2$ for all $t$ by \Cref{ass_boundedness}, and that $c_D\geq
4\{c_1/(m_{\psi2}^2\underline\kappa)
+\sqrt{c_L}c_2c_u^2/(m_{\psi2}^2\underline\kappa)
+4c_u^2\sqrt{c_L}\}$. Furthermore, we have that $\epsilon_\cI \geq D_0\sqrt{d}/\sqrt{|\cI|}\geq D_0\sqrt{d}/\sqrt{2^{j-1}L}$ as we have that $|\cI|\leq |\mathcal{B}_{i,j}'|=2^{j-1}L.$

    For (b), by result of (a) and the condition $\epsilon_{\cI} \geq D_0\sqrt{d/|\cI|}$, we have that $D_0\sqrt{d/|\cI|} \leq \epsilon_\cI \leq D_0\sqrt{d}/\sqrt{2^{m}L}$, which implies that $|\cI|\geq 2^mL$.

    For (c), denote $\widetilde s\in [s,e-2^mL+1]$ as the starting index of the exploration phase. We have that
    \begin{align*}
       \mathcal L_{\cJ_{i,j,(\widetilde{s},m)}}(\widehat\theta_{i,j-1})- \mathcal L_{\cJ_{i,j,(\widetilde{s},m)}}(\widehat\theta_{i,j,(\widetilde s,m)})\geq &m_{\psi2}(\widehat\theta_{i,j-1}-\widehat\theta_{i,j,(\widetilde s,m)})^\top H_{i,j,(\widetilde s,m)}(\widehat\theta_{i,j-1}-\widehat\theta_{i,j,(\widetilde s,m)})/2\\
        \geq &3m_{\psi2}/8\|\widehat\theta_{i,j-1}-\theta_s\|^2_{H_{i,j,(\widetilde s,m)}} - 3m_{\psi2}\|\theta_s-\theta_{\widetilde s}\|^2_{H_{i,j,(\widetilde s,m)}} \\&-3m_{\psi2}\|\theta_{\widetilde s}-\widehat\theta_{i,j,(\widetilde s,m)}\|^2_{H_{i,j,(\widetilde s,m)}}.
    \end{align*}
    
    For the first term, we have
    \begin{align*}
        \|\widehat\theta_{i,j-1}-\theta_s\|^2_{H_{i,j,(\widetilde s,m)}}\geq & \|\widehat\theta_{i,j-1}-\theta_s\|^2_{\widetilde{\Sigma}_z}\lambda_{\min}(\widetilde{\Sigma}_z^{-1/2}H_{i,j,(\widetilde s,m)}\widetilde{\Sigma}_z^{-1/2})\\
        \geq & D_0\sqrt{d}/\sqrt{2^{m+1}L} \cdot \underline{\kappa}\sqrt{d2^mL\log(dT)}/2\geq D_0d\sqrt{\log(dT)}\underline{\kappa}/4,
    \end{align*}
    where the second inequality follows from (a) and \Cref{lem:mcpdp-tech-hessian}(ii).

    For the second term, we have $$\|\theta_s-\theta_{\widetilde s}\|^2_{H_{i,j,(\widetilde s,m)}}\leq \|\theta_s-\theta_{\widetilde s}\|^2_{\widetilde{\Sigma}_z}\lambda_{\max}(\widetilde{\Sigma}_z^{-1/2}H_{i,j,(\widetilde s,m)}\widetilde{\Sigma}_z^{-1/2})\leq 2\overline{\kappa}\sqrt{d2^mL\log(dT)}V_{\cI}\leq 2\overline{\kappa}d\sqrt{\log(dT)}.$$ 
    
    For the third term, by \Cref{prop:prediction2} (iii), we have
    \begin{align*}
         \|\theta_{\widetilde s}-\widehat\theta_{i,j,(\widetilde s,m)}\|^2_{H_{i,j,(\widetilde s,m)}} \leq m_{\psi2}^{-2}[c_1d\log(dT) + c_2|\cJ_{i,j,(\widetilde s,m)}| V_\cI]\leq cd\log(dT).
    \end{align*}

Therefore, the algorithm will restart if $3D_0\underline{\kappa}/16-12\overline{\kappa}-6c\geq 2c_{\gamma}/m_{\psi2}.$  This finishes the proof.
\end{proof}

We now bound $T_2$. We condition on the history before block $j.$ Note that under this condition, the partition $\cI_1',\cdots,\cI_S'$, as well as $\epsilon_{\cI'_k}$ and $\alpha_{\cI'_k}$ for $k\in[S]$ are all determined. Define $\mathcal{M}=\{k\in[S]: \epsilon_{\cI_k'}>D_0\alpha_{\cI_k'}\}$. Therefore we can rewrite $T_2$ as
\begin{align*}
    T_2=\sum_{k=1}^{\Gamma}2|\cI_k|\epsilon_{\cI_k'} \mathbb I\{\epsilon_{\cI_k'}>D_0\alpha_{\cI_k'}\}=\sum_{k\in [\Gamma] \cap \mathcal{M}}2|\cI_k|\epsilon_{\cI_k'}.
\end{align*}
For each $k\in\mathcal M$, let
\[
\mathfrak M_k
:=
\left\{m\in\{0,\ldots,j-2\}:
D_0\sqrt{\frac d{2^{m+1}L}}
\leq\epsilon_{\cI_k'}
\leq
D_0\sqrt{\frac d{2^mL}}
\right\}.
\]
Set $m_k=\min\mathfrak M_k$ when this set is nonempty and set $m_k=0$
otherwise.  This is an $\mathcal F_{i,j}$-measurable, everywhere-defined
selection.  On $\mathcal A$, \Cref{lem:detection_interval_context_LRT}
ensures that $\mathfrak M_k$ is nonempty and that its selected $m_k$
satisfies parts (a)--(c). We further have
\begin{align}\label{T2_decomp_LRT}
\begin{split}
    &\sum_{k\in [\Gamma] \cap \mathcal{M}}2|\cI_k|\epsilon_{\cI_k'}\\\leq& 2\sum_{k\in [\Gamma] \cap \mathcal{M}}|\cI_k|D_0\sqrt{d/(2^{m_k}L)}\\
    =&2\sum_{k\in [\Gamma] \cap \mathcal{M}}2^{m_k}L\cdot D_0\sqrt{d/(2^{m_k}L)}+2\sum_{k\in [\Gamma] \cap \mathcal{M}}\{|\cI_k|-2^{m_k}L\}\cdot D_0\sqrt{d/(2^{m_k}L)}\\
    \leq & 2\sum_{k\in [\Gamma] \cap \mathcal{M}} D_0\sqrt{d|\cI_k'|}+2\sum_{k\in [\Gamma] \cap \mathcal{M}}\{|\cI_k|-2^{m_k}L\}\cdot D_0\sqrt{d/(2^{m_k}L)}:=T_3+T_4,
\end{split}
\end{align}
where the first inequality follows from \Cref{lem:detection_interval_context_LRT}(a) and the second inequality follows from \Cref{lem:detection_interval_context_LRT}(b).

For $T_3$, by the Cauchy-Schwartz inequality, we have that \begin{equation}\label{bound_T3_LRT}
    T_3=2\sum_{k\in [\Gamma] \cap \mathcal{M}} D_0\sqrt{d|\cI_k'|}\leq 2D_0\sqrt{d\Gamma 2^{j-1}L},
\end{equation}  as $\sum_{k\in [\Gamma] \cap \mathcal{M}} |\cI_k'|\leq 2^{j-1}L$, which is the length of the fictitious block $B_{i,j}'.$

We now control $T_4$.  Split its summands according to whether
$|\cI_k|-2^{m_k}L\leq2^{m_k+1}L$.  The contribution of these short
remainders is at most $2T_3$.  Let $\mathcal K$ collect the planned pieces
for which a larger remainder can occur:
\[
    \mathcal K=\left\{k\in[S]\cap\mathcal M:
    |\cI_k'|-2^{m_k}L>2^{m_k+1}L\right\}.
\]
Every summand satisfying the preceding length condition is indexed by
$\mathcal K$.  Let $R^{\rm miss}_{i,j}$ denote their contribution to $T_4$,
and let
$\mathcal T_k'$ be the $L$-spaced grid on $\cI_k'=[s_k',e_k']$.
Since each such remainder is longer than $2^{m_k+1}L$,
\begin{align}
R^{\rm miss}_{i,j}
&\leq
4L\sum_{k\in\mathcal K}
\sum_{t\in\mathcal T_k'\cap[s_k'+2^{m_k}L,e_k']}
D_0\sqrt{\frac{d}{2^{m_k}L}}\,
\mathbb I\{t\leq e_\Gamma\}
=:\Phi(e_\Gamma).
\label{eq:mcpdp-missed-charge-envelope}
\end{align}
Here each $t$ is a checkpoint for a detection opportunity that starts at
$t-2^{m_k}L$ and, by construction, finishes no later than $t$.
\begin{lemma}
\label{lem:mcpdp-bernoulli-miss}
Suppose event $\mathcal{A}$ holds, for the missed-detection contribution $R^{\rm miss}_{i,j}$ defined in
\eqref{eq:mcpdp-missed-charge-envelope}, every $x>4D_0\sqrt{dL}$ satisfies
\begin{equation}
\mathbb P\!\left(\mathcal{A}\cap \{
 R^{\rm miss}_{i,j}>x\}\mid\mathcal{F}_{i,j}
\right)
\leq
\exp\!\left\{
-\frac{x-4D_0\sqrt{dL}}
{4D_0\sqrt{d2^jL}}
\right\}.
\label{eq:mcpdp-bernoulli-miss-tail}
\end{equation}
Consequently, for every $\delta\in(0,1)$,
\[
\mathbb P\!\left(\mathcal{A}\cap\{
R^{\rm miss}_{i,j}>
4D_0\sqrt{d2^jL}\bigl(1+\log(1/\delta)\bigr)\}\mid\mathcal{F}_{i,j}
\right)
\leq\delta.
\]
Here, $\mathcal F_{i,j}$ denotes the history available immediately before
MSS is called for block $\mathcal B_{i,j}$, excluding the MSS randomization
within the current block.
\end{lemma}

\begin{proof}
Condition on $\mathcal F_{i,j}$.  The set $\mathcal K$ and its associated
scales $m_k$ are then fixed.
For any $\tau$ in the planned block, extend the definition in
\eqref{eq:mcpdp-missed-charge-envelope} by setting
\[
\Phi(\tau)
=4L\sum_{k\in\mathcal K}
\sum_{t\in\mathcal T_k'\cap[s_k'+2^{m_k}L,e_k']}
D_0\sqrt{\frac{d}{2^{m_k}L}}\,
\mathbb I\{t\leq\tau\}.
\]
By \Cref{lem:detection_interval_context_LRT}(c), on $\mathcal A$, sampling
the scale-$m_k$ window starting at $t-2^{m_k}L$ would restart the algorithm
by time $t$.  Hence, if the realized block continues beyond $\tau$, every
opportunity with $t\leq\tau$ must have been missed.  These starting points
are distinct, and the required scale is sampled at each of them with
probability $2^{-(j+m_k)/2}$.  Their Bernoulli draws are independent, so
\begin{align*}
\mathbb P\!\left(e_\Gamma>\tau
\mid\mathcal F_{i,j}\right)
&\leq
\exp\!\left\{-\sum_{k\in\mathcal K}
\sum_{\substack{t\in\mathcal T_k'\cap[s_k'+2^{m_k}L,e_k']\\t\leq\tau}}
2^{-(j+m_k)/2}\right\}\\
&=
\exp\!\left\{-\frac{\Phi(\tau)}{4D_0\sqrt{d2^jL}}\right\},
\end{align*}
where the equality follows from
\[
\frac{4LD_0\sqrt{d/(2^mL)}}{2^{-(j+m)/2}}
=4D_0\sqrt{d2^jL}.
\]
Because the intervals $\cI_k'$ are disjoint, at most one term is
added to $\Phi$ at any time, and each term is at most
$4D_0\sqrt{dL}$.  For $x>4D_0\sqrt{dL}$, choose $\tau$ such that
$\Phi(\tau)\leq x<\Phi(\tau+1)$; if no such $\tau$ exists, the probability
in \eqref{eq:mcpdp-bernoulli-miss-tail} is zero.  Then
$\Phi(\tau)>x-4D_0\sqrt{dL}$, while
$R^{\rm miss}_{i,j}\leq\Phi(e_\Gamma)$ implies that
$R^{\rm miss}_{i,j}>x$ only if $e_\Gamma>\tau$.  This proves
\eqref{eq:mcpdp-bernoulli-miss-tail} after averaging over
$\mathcal F_{i,j}$.  The final claim follows because
$4D_0\sqrt{dL}\leq4D_0\sqrt{d2^jL}$.
\end{proof}

Taking $\delta=T^{-3}$ in \Cref{lem:mcpdp-bernoulli-miss} gives
\[
\mathbb P\!\left(\mathcal{A}\cap
\left\{R^{\rm miss}_{i,j}>
20D_0\log(T)\sqrt{d2^jL}\mid\mathcal{F}_{i,j}\right\}
\right)
\leq T^{-3}.
\]  Hence, after a union bound
over at most $T$ realized blocks, the probability that $\mathcal A$ holds and
the following bound fails for at least one block is at most $T^{-2}$:
\begin{equation}
    T_4
    \leq
    2T_3+20D_0\log(T)\sqrt{d2^jL}.
    \label{bound_T4_LRT}
\end{equation}

Combining  \eqref{T2_decomp_LRT}, \eqref{bound_T3_LRT} and \eqref{bound_T4_LRT}, by \Cref{lem:mcpdp-partition}, we have that with probability at least $1-T^{-2}$,
\begin{align}\label{bound_T2_context_LRT}
    \begin{split}
       T_2&=T_3+T_4\leq 6D_0\sqrt{d\Gamma b_j}+20\log(T)D_0\sqrt{2db_j}.
    \end{split}
\end{align}
Consequently,
\begin{align}
T_2
&\le
6D_0\sqrt{d s_{\cI^{*}}b_j}
+20\log(T)D_0\sqrt{2db_j},
\label{eq:mcpdp-bound-T2-switch}\\
T_2
&\le
6D_0\sqrt{db_j}
+3\cdot 2^{7/6}D_0d^{1/3}b_j^{2/3}V_{\cI^{*}}^{1/3}
+20\log(T)D_0\sqrt{2db_j}.
\label{eq:mcpdp-bound-T2-var}
\end{align}

Let
\[
    A_T:=\sqrt{ds_TT},
    \qquad
    B_T:=\sqrt{dT}+d^{1/3}V_T^{1/3}T^{2/3},
    \qquad
    K_T:=1+\lceil\log_2T\rceil .
\]
Combining \eqref{eq:mcpdp-bound-T1-switch}, \eqref{eq:mcpdp-bound-T1},
\eqref{eq:mcpdp-bound-T2-switch}, and \eqref{eq:mcpdp-bound-T2-var}, we obtain,
on the event $\mathcal A$ and on the high-probability event used to control
$T_4$, that for each realized part of the block
$\mathcal R_{i,j}:=\mathcal B_{i,j}\cap[\tau_i+1,\tau_{i+1}]$ with
$b_j=2^{j-1}L$,
\begin{align}
\sum_{t\in\mathcal R_{i,j}}
\|\theta_t-\widehat\theta_{i,j-1}\|_{\widetilde\Sigma_z}^2
&\le
\Bigl\{
(8D_0+2)\sqrt{d s_{\mathcal R_{i,j}}b_j}
\Bigr\}
\notag\\
&\quad\wedge
\Bigl\{
(8D_0+2)\sqrt{db_j}
+2^{7/6}(4D_0+1)d^{1/3}b_j^{2/3}V_{\mathcal R_{i,j}}^{1/3}
\Bigr\}
\notag\\
&\quad+
20D_0\log(T)\sqrt{2db_j}.
\label{eq:mcpdp-local-RT3-final}
\end{align}
Together with the concentration step preceding \Cref{lem:mcpdp-bound-on-I} and
\eqref{eq:mcpdp-regret-to-prediction}, this implies
\begin{align}
R_{T,3}
&\le
C_{\rm reg}\sum_{i=1}^E\sum_{j=1}^{B_i}
\Bigl[
\Bigl\{(8D_0+2)\sqrt{d s_{\mathcal R_{i,j}}b_j}\Bigr\}
\notag\\
&\qquad\qquad\qquad\wedge
\Bigl\{(8D_0+2)\sqrt{db_j}
+2^{7/6}(4D_0+1)d^{1/3}b_j^{2/3}V_{\mathcal R_{i,j}}^{1/3}\Bigr\}
\Bigr]
\notag\\
&\quad+
20\sqrt2\,C_{\rm reg}D_0\log(T)
\sum_{i=1}^E\sum_{j=1}^{B_i}\sqrt{db_j}
+4c_0\sqrt{\log T}
\sum_{i=1}^E\sum_{j=1}^{B_i}\sqrt{|\mathcal R_{i,j}|}.
\label{eq:mcpdp-RT3-before-sum}
\end{align}
We now sum the terms in \eqref{eq:mcpdp-RT3-before-sum}.  Since each epoch
contains at most $K_T$ realized dyadic blocks, $E\le s_T$, and the realized
blocks are disjoint, the stationary-segment count satisfies the first
bound below.  For the second bound, if epoch $i$ enters block $B_i\geq1$,
then $\tau_{i+1}-\tau_i\geq2^{B_i-1}L$, whereas
\[
\sum_{j=1}^{B_i}b_j
\leq\sum_{j=1}^{B_i}2^{j-1}L
<2^{B_i}L
\leq2(\tau_{i+1}-\tau_i).
\]
 Summing this dyadic-entry bound over
epochs gives
\[
\sum_{i=1}^E\sum_{j=1}^{B_i}s_{\mathcal R_{i,j}}
\le
2s_TK_T,
\qquad
\sum_{i=1}^E\sum_{j=1}^{B_i}b_j
\le
2T.
\]
Thus the Cauchy--Schwarz inequality gives
\[
\sum_{i=1}^E\sum_{j=1}^{B_i}
\sqrt{d s_{\mathcal R_{i,j}}b_j}
\le2
\sqrt{K_T}\,A_T.
\]
For the variation side, the dyadic structure gives
\[
\sum_{j=1}^{B_i}\sqrt{b_j}
\le
a_0\sqrt{2(\tau_{i+1}-\tau_i)},
\]
and hence, by Cauchy--Schwarz and \Cref{cor:mcpdp-number-epoch},
\[
\sum_{i=1}^E\sum_{j=1}^{B_i}\sqrt{db_j}
\le
a_0\sqrt{2dTE}
\le
a_0\sqrt2\,(A_T\wedge B_T).
\]
Moreover, by H\"older's inequality and the additivity of the ordered variation
over disjoint intervals,
\[
\sum_{i=1}^E\sum_{j=1}^{B_i}
d^{1/3}b_j^{2/3}V_{\mathcal R_{i,j}}^{1/3}
\le
d^{1/3}
\left(\sum_{i,j}b_j\right)^{2/3}
\left(\sum_{i,j}V_{\mathcal R_{i,j}}\right)^{1/3}
\le
2^{2/3}d^{1/3}T^{2/3}V_T^{1/3}.
\]
Finally, since $|\mathcal R_{i,j}|\le b_j$,
\[
\sum_{i=1}^E\sum_{j=1}^{B_i}\sqrt{|\mathcal R_{i,j}|}
\le
a_0\sqrt{2TE}
\le
a_0\sqrt2\,(A_T\wedge B_T),
\]
where we used $d\ge1$ in the last step.  Substituting the above inequalities
into \eqref{eq:mcpdp-RT3-before-sum} yields
\begin{align}
R_{T,3}
&\le
\Bigl[
C_{\rm reg}\{(16D_0+4)\sqrt{K_T}
+(8D_0+2)a_0\sqrt2
+2^{11/6}(4D_0+1)\}
\notag\\
&\qquad\qquad
+40a_0C_{\rm reg}D_0\log(T)
+4a_0c_0\sqrt{2\log T}
\Bigr]
(A_T\wedge B_T).
\label{eq:mcpdp-bound-RT3-final}
\end{align}

\subsection{Combining the Three Regret Terms}
\label{app:combine-three-regret-terms}

Finally, we combine \eqref{eq:mcpdp-bound-RT1}, \eqref{eq:mcpdp-bound-RT2}, and
\eqref{eq:mcpdp-bound-RT3-final}.  Since $D_0\asymp\sqrt{\log dT}$, all displayed coefficients above are bounded by a
constant depending only on
$c_0,c_L,C_{\rm reg},a_0,c_D,c,c_\gamma,\underline\kappa,\overline\kappa$
times $\log^{3/2}(dT)$.  This implies that for a sufficiently large $C_{\pi}$,
\[
    R_T^{\pi}
    \le
    C_{\pi}\log^{3/2}(dT)
        \left(\sqrt{ds_TT}
        \wedge
        \left\{
        \sqrt{dT}
        +
        d^{1/3}V_T^{1/3}T^{2/3}
        \right\}
    \right),
\]
which proves \Cref{thm:mcpdp-upper}.


\end{document}